\documentclass[final]{alt2026} 

\title[Sample Complexity for Linear CMDP]{Sample Complexity Bounds for Linear Constrained \\ MDPs with a Generative Model}
\usepackage{times}

\usepackage{comment}
\makeatletter
\def\ps@jmlrtps{\let\@mkboth\@gobbletwo
 \def\@oddhead{}%
 \let\@evenhead\@empty
 \def\@oddfoot{}%
 \let\@evenfoot\@oddfoot}
\makeatother
\editors{}
\usepackage{titletoc} 
\usepackage{bbm}
\PassOptionsToPackage{ruled,vlined,linesnumbered}{algorithm2e}

\usepackage{tikz}
\usetikzlibrary{automata, positioning, arrows.meta}
\usepackage{caption}

\newcommand{\svktb}[1]{{V}^{\pi^{\prime}_{#1}}_{\Box}}

\newcommand{\vbkr}{{V}^{\bar{\pi}}_{r}}
\newcommand{\vbkc}{{V}^{\bar{\pi}}_{c}}

\newcommand{\vssr}{{V}^{\pi^{*+}}_{r}}
\newcommand{\vssc}{{V}^{\pi^{*+}}_{c}}

\newcommand{\vsr}{{V}^{\pi^*}_{r}}
\newcommand{\vsc}{{V}^{\pi^*}_{c}}

\newcommand{\vkb}{{V}^{\pi^*_k}_{\Box}}

\newcommand{\vkr}{{V}^{\pi^*_k}_{r}}
\newcommand{\vkc}{{V}^{\pi^*_k}_{c}}

\newcommand{\vktb}{{V}^{\pi_{T}}_{\Box}}

\newcommand{\vktr}{{V}^{\pi_{T}}_{r}}
\newcommand{\vktc}{{V}^{\pi_{T}}_{c}}

\newcommand{\hvkb}{\hat{V}^{t}_{\Box}}
\newcommand{\hvkd}{\hat{V}^{t}_{\diamond}}

\newcommand{\hvkib}{\hat{V}^{i}_{\Box}}
\newcommand{\hvkid}{\hat{V}^{i}_{\diamond}}

\newcommand{\hvkttb}[1]{\hat{V}^{#1}_{\Box}}

\newcommand{\bvkb}{\bar{V}^{T}_{\Box}}
\newcommand{\bvkd}{\bar{V}^{T}_{\diamond}}
\newcommand{\bvkr}{\bar{V}^{T}_{r}}
\newcommand{\bvkc}{\bar{V}^{T}_{c}}

\newcommand{\bvktb}{\bar{V}^{t}_{\Box}}
\newcommand{\bvktd}{\bar{V}^{t}_{\diamond}}

\newcommand{\bvkttb}[1]{\bar{V}^{#1}_{\Box}}

\newcommand{\hqkib}{\hat{Q}^{i}_{\Box}}
\newcommand{\hqkid}{\hat{Q}^{i}_{\diamond}}

\newcommand{\hqktb}[1]{\hat{Q}^{#1}_{\Box}}
\newcommand{\hqktd}[1]{\hat{Q}^{#1}_{\diamond}}

\newcommand{\bmvd}{\bar{\mathcal{V}}^{\bar{\pi}}_{\diamond}}
\newcommand{\bmvr}{\bar{\mathcal{V}}^{\bar{\pi}}_{r}}
\newcommand{\bmvc}{\bar{\mathcal{V}}^{\bar{\pi}}_{c}}

\newcommand{\hmvkd}{\hat{\mathcal{V}}^{t}_{\diamond}}

\newcommand{\hmvkid}{\hat{\mathcal{V}}^{i}_{\diamond}}

\newcommand{\bmvkd}{\bar{\mathcal{V}}^{T}_{\diamond}}

\newcommand{\bmvktd}{\bar{\mathcal{V}}^{t}_{\diamond}}

\newcommand{\hmqkid}{\hat{\mathcal{Q}}^{i}_{\diamond}}

\newcommand{\hmvkttd}[1]{\hat{\mathcal{V}}^{#1}_{\diamond}}

\newcommand{\bmvkttd}[1]{\bar{\mathcal{V}}^{#1}_{\diamond}}

\newcommand{\hmqktd}[1]{\hat{\mathcal{Q}}^{#1}_{\diamond}}

\newcommand{\norm}[1]{\left\|#1 \right\|}

\newcommand{\norminf}[1]{\left\|#1\right\|_{\infty}}

\newcommand{\cA}{\mathcal{A}}
\newcommand{\cB}{\mathcal{B}}

\newcommand{\cS}{\mathcal{S}}

\renewcommand{\epsilon}{\varepsilon}

\newcommand{\const}[1]{V_c^{#1}(\rho)}

 \newcommand{\reward}[1]{V_r^{#1}(\rho)}

 \newcommand{\rewardq}[1]{Q_{r}^{#1}}

\newcommand{\constq}[1]{Q_{c}^{#1}}

\newcommand{\piopt}{\pi^*}

\newcommand{\PP}{\mathbb{P}}

\newcommand{\cP}{\mathcal{P}}

\usepackage[utf8]{inputenc}          
\usepackage[T1]{fontenc}             
\usepackage{amsmath,amssymb,amsfonts} 
\usepackage{nicefrac}               
\usepackage{microtype}              
\usepackage{xcolor}                 
\usepackage{ifthen}
\usepackage{bm}
\usepackage{enumitem}
\usepackage{graphicx}
\usepackage{xspace}
\usepackage{verbatim}
\usepackage{algorithm}
\usepackage{algorithmicx}
\usepackage{algpseudocode}
\usepackage{thm-restate}
\usepackage{latexsym}
\usepackage{array, makecell}
\usepackage[colorinlistoftodos,textsize=scriptsize]{todonotes}
\usepackage{csquotes}
\usepackage[capitalize]{cleveref}
\usepackage{thmtools}
\usepackage{lipsum}
\crefname{assumption}{Assumption}{Assumptions}
\usepackage{mathtools}
\usepackage{setspace}
\allowdisplaybreaks

\declaretheorem[name=Assumption]{assumption}

\Crefname{equation}{Eq.}{Eqs.}
\Crefname{section}{Sec.}{Secs.}
\Crefname{appendix}{App.}{Apps.}
\Crefname{algorithm}{Alg.}{Algs.}
\Crefname{assumption}{Assn.}{Assns.}
\Crefname{theorem}{Thm.}{Thms.}

\makeatletter
\def\thm@space@setup{\thm@preskip=1ex
\thm@postskip=1ex}
\makeatother

\DeclareMathOperator*{\argmin}{arg\ min}
\DeclareMathOperator*{\argmax}{arg\ max}

\altauthor{%
 \Name{Xingtu Liu} \Email{rltheory@outlook.com}\\
 \addr{Simon Fraser University}
 \AND
 \Name{Lin F. Yang} \Email{linyang@ee.ucla.edu}\\
 \addr{University of California, Los Angeles}
 \AND
 \Name{Sharan Vaswani} \Email{vaswani.sharan@gmail.com}\\
 \addr{Simon Fraser University}
}

\begin{document}

\maketitle

\begin{abstract}
We consider infinite-horizon $\gamma$-discounted (linear) constrained Markov decision processes (CMDPs) where the objective is to find a policy that maximizes the expected cumulative reward subject to expected cumulative constraints. Given access to a generative model, we propose to solve CMDPs with a primal-dual framework that can leverage any black-box unconstrained MDP solver. For linear CMDPs with feature dimension $d$, we instantiate the framework by using mirror descent value iteration (\texttt{MDVI})~\citep{kitamura2023regularization} an example MDP solver. We provide sample complexity bounds for the resulting CMDP algorithm in two cases: (i) relaxed feasibility, where small constraint violations are allowed, and (ii) strict feasibility, where the output policy is required to exactly satisfy the constraint. For (i), we prove that the algorithm can return an $\epsilon$-optimal policy with high probability by using $\tilde{O}\left(\frac{d^2}{(1-\gamma)^4\epsilon^2}\right)$ samples. For (ii), we show that the algorithm requires $\tilde{O}\left(\frac{d^2}{(1-\gamma)^6\epsilon^2\zeta^2}\right)$ samples, where $\zeta$ is the problem-dependent Slater constant that characterizes the size of the feasible region. Furthermore, we prove a lower-bound of $\Omega\left(\frac{d^2}{(1-\gamma)^5\epsilon^2\zeta^2}\right)$ for the strict feasibility setting. We note that our upper bounds under both settings exhibit a near-optimal dependence on $d$, $\epsilon$, and $\zeta$. Finally, we instantiate our framework for tabular CMDPs and show that it can be used to recover near-optimal sample complexities in this setting.
\end{abstract}


\section{Introduction}
\label{sec:introduction}

Reinforcement learning (RL)~\citep{sutton1998reinforcement} is a machine learning paradigm aimed at building learning agents capable of making sequential decisions in an (unknown) environment. RL algorithms have found applications in games such as Atari~\citep{mnih2015human} or Go~\citep{silver2016mastering}, robot manipulation tasks~\citep{tan2018sim,zeng2020tossingbot}, clinical trials~\citep{schaefer2005modeling} and more recently, aligning large language models to human preferences~\citep{shao2024deepseekmath,ouyang2022training}. Typical RL algorithms only focus on optimizing an unconstrained objective, although in many real-world applications, agents are often required to not only maximize cumulative rewards but also to satisfy constraints imposed by safety, fairness, or resource usage. RL with such side-constraints is typically formulated within the framework of constrained Markov decision processes (CMDPs)~\citep{altman1999constrained}, where the goal is to optimize an expected reward function while ensuring that the expected cumulative cost (or utility) satisfies a given threshold. For example, in wireless sensor networks~\citep{buratti2009overview, julian2002qos}, the agent aims to deploy a policy that maximizes the bitrate with a constraint on its average power consumption.

Given the practical importance of constrained RL, there is a vast literature~\citep{efroni2020exploration,zheng2020constrained,qiu2020upper,brantley2020constrained,kalagarla2020sample,yu2021provably,ding2021provably, gattami2021reinforcement,miryoosefi2022simple,mondal2024sample} that aims to obtain a near-optimal policy in unknown tabular CMDPs with finite states and actions. These works simultaneously tackle the exploration, estimation and planning problems and aim to minimize the regret and constraint violation in the online setting. On the other hand, recent works~\citep{hasan2021model, wei2021provably,bai2021achieving, vaswani2022near} consider an easier, but even more fundamental problem of obtaining a near-optimal policy with access to a simulator or \emph{generative model}~\citep{kearns1999finite,kakade2003sample,agarwal2020model, sidford2018near, yang2019sample}. In particular, these works assume that the agent has access to a sampling oracle (the generative model) that returns a sample of the next state when given any state-action pair as input. Depending on the application of interest, such a generative model is often available either directly for the task at hand (for example, in Atari games where the aim is to win the game) or as an proxy to the task (for example, the CARLA simulator~\citep{dosovitskiy2017carla} for training autonomous vehicles). Moreover, from a theoretical perspective, since the generative model setting removes the need for exploration it has been used to characterize the statistical complexity of obtaining near-optimal policies for (C)MDPs~\citep{azar2013minimax,agarwal2020model,li2020breaking,vaswani2022near}. In particular, for CMDPs,~\citet{vaswani2022near} established near-optimal upper and lower-bounds on the sample complexity in two settings: (i) relaxed feasibility, where small constraint violations are allowed, and (ii) strict feasibility, where the output policy is required to exactly satisfy the constraint. For tabular CMDPs, the proposed algorithms and resulting bounds depend on the cardinality of the state-action space, and hence do not apply to modern applications involving large or infinite state spaces. Consequently, it is essential to develop provably efficient algorithms that can incorporate function approximation and go beyond the tabular case.

For unconstrained MDPs, the linear MDP assumption~(e.g., \citep{yang2019sample, jin2020provably}) is a common formalization to analyze algorithms that have access to state-action features and can incorporate linear function approximation. The assumption implies that both the rewards and transition probabilities (approximately) lie in the span of the given $d$-dimensional feature representation, and can be used to obtain sample complexity bounds independent of the size of the state-action space. Unconstrained linear MDPs have been extensively studied in the context of both finite-horizon regret minimization~\citep{jin2020provably,hu2022nearly,weisz2022confident, sherman2023rate,liu2023optimistic} and with access to a generative model~\citep{kitamura2023regularization,taupin2023best}. Following the linear MDP literature, recent works consider CMDPs with linear function approximation~\citep{jain2022towards, ding2021provably, miryoosefi2022simple, ghosh2022provably, ghosh2024towards, liu2022policy, tian2024confident} and assume that (in addition to the rewards and transition probabilities), the costs or utilities can also be expressed using the given features. However, all previous work on linear CMDPs considers the online regret minimization setting and the statistical complexity of the problem remains unclear. Motivated by~\citet{vaswani2022near}, we aim to \textit{study the sample complexity of solving linear CMDPs with access to a generative model}. In particular, we make the following contributions.

\textbf{(1) Generic primal-dual algorithm framework}: In~\cref{sec_Methodology}, we provide a generic primal-dual algorithmic framework (\cref{alg_CMDPL}) that can be used to achieve both the \textit{relaxed} and \textit{strict} feasibility objectives, for both \textit{tabular} and \textit{linear} CMDPs. As model-based approaches~\citep{vaswani2022near} are not applicable in the linear CMDP setting, \cref{alg_CMDPL} is designed to be model-free and relies on three black-box subroutines: a \texttt{DataCollection} procedure, a black-box \texttt{MDP-Solver} and a \texttt{PolicyEvaluation} oracle. We prove a meta-theorem (\cref{thm_main2}) to quantify the sample complexity of \cref{alg_CMDPL} in terms of that of the \texttt{MDP-Solver} and \texttt{PolicyEvaluation} oracle.  

\textbf{(2) Instantiating the framework for linear CMDPs:} In~\cref{sec:instant-mdp}, we instantiate the linear \texttt{MDP-Solver} with a variant of the mirror-descent value iteration (\texttt{MDVI}) algorithm~\citep{kozuno2022kl, kitamura2023regularization}. In contrast to the existing \texttt{MDVI} variants, the proposed~\cref{alg_MDVICLMDP} does not use entropy regularization and outputs a stationary policy, thus simplifying the algorithm design. We develop a new theoretical analysis for~\cref{alg_MDVICLMDP} and characterize its sample complexity for solving unconstrained linear MDPs. In~\cref{sec:instant-eval}, we instantiate the \texttt{PolicyEvaluation} oracle with least-squares policy evaluation (\cref{alg_CVIL}) and analyze the sample complexity required to evaluate the performance of a (data-dependent) policy. 

\textbf{(3) Upper-bound on sample complexity for linear CMDPs:} In~\cref{sec:all-linear}, we leverage our meta-theorem and analyze the sample complexity for the resulting CMDP algorithm that uses~\cref{alg_MDVICLMDP,alg_CVIL}. In particular, if $d$ is the dimension of the feature mapping, we prove that the proposed algorithm requires no more than $\tilde{O}\left(\frac{d^2}{(1-\gamma)^4\epsilon^2}\right)$ samples to obtain an $\epsilon$-optimal policy in the relaxed feasibility setting. Since the lower-bound on the sample complexity for solving unconstrained linear MDP is ${\Omega}\left(\frac{d^2}{(1-\gamma)^3\epsilon^2}\right)$~\citep{weisz2022confident}, our sample complexity achieves the near-optimal dependence on $d$ and $\epsilon$, and is away from the lower bound by atmost a multiplicative factor of $\tilde O\left(\nicefrac{1}{1-\gamma}\right)$. Under strict feasibility, our algorithm requires no more than $\tilde{O}\left(\frac{d^2}{(1-\gamma)^6\epsilon^2\zeta^2}\right)$ samples, where $\zeta$ is the problem-dependent Slater constant that characterizes the size of the feasible region and dictates the difficulty of the problem. Given the lower-bounds for tabular CMDPs in~\citet{vaswani2022near}, we conjecture that the dependence on $d$, $\epsilon$, and $\zeta$ in our bounds is tight, with suboptimality arising only in the multiplicative dependence on $O(\nicefrac{1}{1-\gamma})$. To the best of our knowledge, \textit{these are the first such sample complexity bounds with the near-optimal dependence on both $d$ and $\epsilon$ }. In~\cref{app:gss}, we alternatively instantiate the linear \texttt{MDP-Solver} to be the G-Sampling-and-Stop (\texttt{GSS}) algorithm~\citep{taupin2023best} and analyze the sample complexity of the resulting CMDP algorithm, thus demonstrating the flexibility of our framework.  

\textbf{(4) Lower-bound on sample complexity for linear CMDPs:} In~\cref{sec:all-linear}, we prove a problem-dependent ${\Omega}\left(\frac{d^2}{(1-\gamma)^5\epsilon^2\zeta^2}\right)$ lower-bound on the sample-complexity in the strict feasibility setting. Our results thus demonstrate that the proposed algorithm is near-optimal in terms of $d$, $\epsilon$, and $\zeta$, with a suboptimality only in the multiplicative dependence on $H$. The lower bound also indicates that under strict feasibility solving linear CMDPs is inherently more difficult than solving unconstrained linear MDPs, and the problem difficulty increases as the size of the feasible region (measured in terms of $\zeta$) decreases. To the best of our knowledge, it is the first result characterizing the difficulty of solving linear CMDPs with access to a generative model. 

\textbf{(5) Sample complexity bounds for tabular CMDPs:} Finally, in~\cref{sec_tabularCMDP}, we utilize our framework for tabular CMDPs. In particular, we instantiate~\cref{alg_CMDPL} with tabular variants of~\cref{alg_MDVICLMDP,alg_CVIL} (obtained by setting $d = SA$ and considering one-hot features) and analyze the resulting CMDP algorithm. Under the relaxed and strict feasibility settings, the resulting algorithm attains sample complexity bounds of $\tilde{O}\left(\frac{|\mathcal{S}||\mathcal{A}|}{(1-\gamma)^3 \epsilon^2}\right)$ and $\tilde{O}\left(\frac{|\mathcal{S}||\mathcal{A}|}{(1-\gamma)^5 \epsilon^2 \zeta^2}\right)$, respectively. These results match the near-optimal bounds attained by the model-based algorithm in~\citet{vaswani2022near}, and improve upon the sample-complexity of the model-free approach proposed in~\citet{bai2021achieving}.

\section{Problem Formulation}
\label{sec:problem}

An infinite-horizon discounted constrained tabular Markov decision process (CMDP)~\citep{altman1999constrained} is denoted by $\mathcal{M}$, and is defined by the tuple $\langle \cS, \cA, \cP, r, c, b, \rho, \gamma \rangle$ where $\cS$ is the set of states, $\cA$ is the action set, $\cP : \cS \times \cA \rightarrow \Delta_\cS$ is the transition probability function, $\rho \in \Delta_{\cS}$ is the initial distribution of states and $\gamma \in [0, 1)$ is the discount factor. The primary reward to be maximized is denoted by $r: \cS \times \cA \rightarrow [0,1]$, whereas the constraint reward is denoted by $c: \cS \times \cA \rightarrow [0,1]$\footnote{These ranges for $r$ and $c$ are chosen for simplicity. Our results can be easily extended to handle other ranges.}. If $\Delta_{\cA}$ denotes the simplex over the action space, the expected discounted return or \emph{reward value function} of a stationary, stochastic policy\footnote{The performance of an optimal policy in a CMDP can always be achieved by a stationary, stochastic policy~\citep{altman1999constrained}. On the other hand, for an MDP, it suffices to only consider stationary, deterministic policies~\citep{puterman2014markov}.} $\pi: \cS \rightarrow \Delta_{\cA}$ is defined as $\reward{\pi} = \mathbb{E}_{s_0, a_0, \ldots} \Big[\sum_{t=0}^\infty \gamma^t r(s_t, a_t)\Big]$, where $s_0 \sim \rho, a_t \sim \pi( \cdot | s_t),$ and $s_{t+1} \sim \cP( \cdot | s_t, a_t)$. For each state-action pair $(s,a)$ and policy $\pi$, the reward action-value function is defined as $\rewardq{\pi}: \cS \times \cA \rightarrow \mathbb{R}$, and satisfies the relation: $V_{r}^{\pi}(s) = \langle \pi(\cdot | s), \rewardq{\pi}(s,\cdot) \rangle$, where $V_{r}^{\pi}(s)$ is the reward value function when the starting state is equal to $s$. Analogously, the \emph{constraint value function} and constraint action-value function of policy $\pi$ is denoted by $\const{\pi}$ and $\constq{\pi}$ respectively. Throughout, it will be convenient to present our results in terms of the effective horizon $H := \nicefrac{1}{(1 - \gamma)}$. 

In addition to the tabular CMDPs with a finite state-action space, we also consider linear~\citep{jin2020provably} CMDPs where the state space can be large or possibly infinite. In this case, we assume access to a feature representation $\phi$ such that $r, c$ and the transition probabilities $\cP$ (approximately) lie in the span of the given $d$-dimensional feature representation. 
\begin{assumption}[Linear Constrained MDP]
\label{asm_32}
For the CMDP \( \mathcal{M} \) with the state-action space \( \mathcal{S} \times \mathcal{A} \), we have access to a known feature map \( \phi : \mathcal{S} \times \mathcal{A} \to \mathbb{R}^d \) that satisfies the following condition: there exist vectors \( \psi_r, \psi_c \in \mathbb{R}^d \) and signed measures \( \mu := (\mu_1, \dots, \mu_d) \) on \( \mathcal{S} \) such that \( P(\cdot | s, a) = \langle \phi(s, a), \mu \rangle \) for any \( (s, a) \in \mathcal{S} \times \mathcal{A} \), \( r =  \langle \phi, \psi_r \rangle \), and \( c = \langle \phi,  \psi_c \rangle\). Let \( \Phi := \{ \phi(s, a) : (s, a) \in \mathcal{S} \times \mathcal{A} \} \subset \mathbb{R}^d \) be the set of all feature vectors. We assume that \( \Phi \) is compact and spans \( \mathbb{R}^d \).
\end{assumption}
The objective is to return a policy that maximizes $\reward{\pi}$, while ensuring that $\const{\pi} \geq b$. Formally, 
\begin{align}
\max_{\pi} \reward{\pi} \quad \text{s.t.} \quad  \const{\pi} \geq b.
\label{eq:true-CMDP}
\end{align}
The optimal stochastic policy for the above CMDP is denoted by $\piopt$ and the corresponding reward value function is denoted by $\reward{*}$. We also define $\zeta := \max_{\pi} \const{\pi} - b > 0$ as the problem-dependent quantity referred to as the Slater constant~\citep{ding2021provably,bai2021achieving}. The Slater constant is a measure of the size of the feasible region and determines the difficulty of solving~\cref{eq:true-CMDP}. 

For simplicity of exposition, we assume that the rewards $r$ and constraint rewards $c$ are known, but the transition probabilities $\cP$ are unknown. We note that assuming the knowledge of the rewards does not affect the leading terms of the sample complexity since learning these is an easier problem~\citep{azar2013minimax, sidford2018near}. Following~\citet{azar2013minimax,vaswani2022near}, we assume access to a \emph{generative model} or simulator that allows the agent to obtain samples from the $\cP(\cdot|s,a)$ distribution for any $(s,a)$. 

\begin{definition}[Generative Model]
\label{def_genm}
A {generative model} $\mathsf{Gen}$ for an MDP is an oracle that, given any state-action pair $(s, a)$, returns an independent sample of the next state $s' \sim P(\cdot \mid s, a)$. 
\end{definition}

Assuming access to such a generative model, we aim to characterize the sample complexity (number of times $\mathsf{Gen}$ is queried) required to return a near-optimal policy $\bar{\pi}$. Specifically, given a target error $\epsilon > 0$, we consider two different definitions of optimality. \\
\textbf{Relaxed feasibility}: We require $\bar{\pi}$ to achieve an approximately optimal reward value, while allowing it to have a small constraint violation. Formally, we aim to find a $\bar{\pi}$ such that,
\begin{align}
V^{\bar{\pi}}_r(\rho) \geq V^{\pi^*}_r(\rho) - \epsilon \quad \text{and} \quad V^{\bar{\pi}}_c(\rho) \geq b - \epsilon. \label{fml_RF}
\end{align}
\textbf{Strict feasibility}: We require $\bar{\pi}$ to achieve an approximately optimal reward value, while simultaneously demanding zero constraint violation. Formally,we aim to find a $\bar{\pi}$ such that,
\begin{align}
V^{\bar{\pi}}_r(\rho) \geq V^{\pi^*}_r(\rho) - \epsilon \quad \text{and} \quad V^{\bar{\pi}}_c(\rho) \geq b. \label{fml_SF}
\end{align}
In the next section, we design a generic algorithmic framework to achieves these objectives.

\section{A Generic Framework for Solving CMDPs}
\label{sec_Methodology}

We first present a generic primal-dual algorithmic framework for solving CMDPs, and subsequently present a meta-theorem that quantifies its sample-complexity in the relaxed and strict feasibility settings. For this, we frame the CMDP problem in~\cref{eq:true-CMDP} as an equivalent saddle-point problem,
\begin{align}
    \max_{\pi}\min_{\lambda\geq 0} \left[ V_r^{\pi}(\rho) + \lambda \left( V_c^{\pi}(\rho) - b \right)  \right] \,,
\label{eq:saddle}
\end{align}
where, $\lambda$ is the Lagrange multiplier. The solution to~\cref{eq:saddle} is $(\pi^*, \lambda^*)$ where $\pi^*$ is the optimal policy to the CMDP and $\lambda^*$ is the optimal Lagrange multiplier. We solve~\cref{eq:saddle} iteratively, by alternatively updating the policy (primal variable) and the Lagrange multiplier (dual variable)~\citep{ding2021provably,vaswani2022near}.

\begin{algorithm}[H]
\SetAlgoLined
\DontPrintSemicolon
\caption{Primal-dual CMDP framework with a generative model}
\label{alg_CMDPL}

\KwIn{$r$ (rewards), $c$ (constraint rewards), $b'$ (constraint RHS), $U$ (projection upper bound), $K$ (number of iterations), $\eta$ (step-size), $\lambda_0 = 0$ (initialization), $\mathsf{Gen}$ (generative model), $\mathcal{C}$ (subset of $\mathcal{S}\times\mathcal{A}$), $N$ (sample size for each $(s,a)$ pair in $\mathcal{C}$), $\phi$ (feature map).}
\KwOut{Mixture policy $\bar{\pi} = \frac{1}{K} \sum_{k=0}^{K-1}\pi_{k}$.}

\BlankLine

\SetKwProg{Proc}{procedure}{}{end}
\Proc{\texttt{CMDPF}($r,c,b^{\prime},U,K,\eta, \mathsf{Gen},\mathcal{C},N,\phi$)}{
    Data collection procedure to populate buffer: $\mathcal{B} = \texttt{DataCollection}(\mathsf{Gen}, \mathcal{C},N).$ \\
    \For{$k = 0, \dots,K-1$}{ 
        \ \ Updating the primal variable: Let $\pi_{k}=$ \texttt{MDP-Solver}$(r+\lambda_kc, \mathcal{B}, \phi)$.  \\
        Policy Evaluation: Let $\hat{V}^{k}_c=$ \texttt{PolicyEvaluation}$(\pi_{k}, c, \mathcal{B}, \phi)$. \\
        Updating the dual variable: $\lambda_{k+1} = \PP_{[0,U]} \left[\lambda_k - \eta \, (\hat{V}^{k}_c(\rho) - b')\right]$.
    }
}
\end{algorithm}

\begin{remark}
The model-based framework for solving tabular CMDPs in~\citet{vaswani2022near} does not directly extend to the linear MDP setting for the following reasons. In the PAC-RL framework (with access to a simulator), model-based methods build a model for the transition matrix by sampling next states for each state-action pair. In the linear (C)MDP setting, the transition kernel and reward functions are assumed to be approximated by a linear function. Specifically, the $SA\times S$ transition matrix can be factorized into a known $SA \times d$ matrix of features and an unknown $d \times S$ matrix to be estimated. Hence, naively estimating the transition matrix will require $O(S)$ samples, resulting in high sample complexity.
\end{remark}
The primal and dual updates in~\cref{alg_CMDPL} rely on three oracles, which we instantiate subsequently. 

\noindent \textbf{Data Collection Oracle}: We first describe the mechanism of the \texttt{DataCollection} oracle in~\cref{alg_CMDPL}. This oracle takes as input a generative model $\mathsf{Gen}$, a subset of state-action pairs $\mathcal{C} \subseteq \mathcal{S} \times \mathcal{A}$, and a sample size $N$. For each $(s,a) \in \mathcal{C}$, it queries the generative model $\mathsf{Gen}$ to obtain $N$ independent next-state samples $(s^{\prime}_i)_{i=1}^N$ from the distribution $\mathsf{Gen}(\cdot \mid s,a)$. It then stores the resulting triplets $(s, a, s'_i)_{i=1}^N$ in a buffer $\mathcal{B}$. After all state-action pairs in $\mathcal{C}$ are processed, the buffer $\mathcal{B}$ contains $N$ samples for each pair and is returned as the output.

\noindent \textbf{MDP-Solver}: The primal update at iteration $k$ uses the \texttt{MDP-Solver}, which takes as input a buffer $\mathcal{B}$ of samples and returns a policy $\pi_k$ satisfying the following assumption.
\begin{assumption}
\label{asmp_b1}
We have access to a black-box algorithm \texttt{MDP-Solver}$(\Box,\mathcal{B},\phi)$ for which the inputs are the feature map $\phi$ and an arbitrary but bounded reward function $\Box {\in [0,R]}$, and the output is a policy $\tilde{\pi}$ satisfying the following condition with probability $1-\delta$,
\begin{align*}
\max_\pi V^{\pi}_{\Box}(\rho) - V^{\tilde{\pi}}_{\Box}(\rho) \leq {R} \, f_{\mathrm{mdp}}(\mathcal{B})  \,, 
\end{align*} 
where, $f_{\mathrm{mdp}}(\mathcal{B})$ denotes an upper bound on the sub-optimality when given access to buffer $\mathcal{B}$. 
\end{assumption}

\noindent \textbf{Policy Evaluation Oracle}: The dual update at iteration $k$ in~\cref{alg_CMDPL} is given as:
$$\lambda_{k+1} = \PP_{[0,U]} \left[\lambda_k - \eta \, (\hat{V}^{k}_c(\rho) - b')\right]\,,$$
where $\PP_{[0,U]}$ denotes the projection onto the interval $[0,U]$, and $b'$ is a relaxed constraint parameter that depends on $b$, $f_{\mathrm{mdp}}$ and the problem setting (relaxed or strict). The term $\hat{V}^{k}_c$ is an estimate of $V^{\pi_k}_c$, computed via the \texttt{PolicyEvaluation} oracle which satisfies following assumption.
\begin{assumption}
\label{asmp_b2}
We have access to a black-box algorithm \texttt{PolicyEvaluation}$(\pi,\diamond,\mathcal{B},\phi)$ for which the inputs are a possibly data-dependent (one that depends on the buffer $\cB$) policy $\pi$, the feature map $\phi$, a reward function $\diamond \in [0,1]$, and the output is a value function $\hat{V}_{\diamond}$ satisfying the following condition with probability $1-\delta$,
\begin{align*}
|\hat{V}_{\diamond}(\rho) - V^{\pi}_{\diamond}(\rho)| \leq f_{\mathrm{eva}}(\mathcal{B}) \,,
\end{align*}
where, $f_{\mathrm{eva}}(\mathcal{B})$ denotes an upper bound on the sub-optimality when given access to buffer $\mathcal{B}$. 
\end{assumption}

\begin{remark}
We note that while there are existing methods that update the dual variable using empirical utility value functions obtained from the MDP solver, this approach does not result in near-optimal statistical guarantees. This is because the MDP solver can only imply a concentration guarantee on the value function for the combined reward function $r + \lambda_kc$ and not for the individual value functions $V_r$ and $V_c$. Proving near-optimal statistical guarantees necessitates the development of the proposed policy evaluation subroutine. Alternative ways to handle such concentration include using a uniform concentration bound~\citep{ghosh2022provably} (i.e., building an $\epsilon$-cover over the function class and using a union bound). However, such an analysis results in a sub-optimal $O(d^3)$ dependence on the dimension. Furthermore, in order to guarantee concentration, \citet{ghosh2022provably} needs to guarantee Lipschitzness in the policy update. Our technique does not require this, and consequently, we can use greedy policy updates that are not Lipschitz.
\end{remark}

After $K$ iterations of primal and dual updates,~\cref{alg_CMDPL} returns a mixture policy $\bar{\pi}$ which is a policy drawn uniformly at random from the set $\{\pi_0,\dots,\pi_{K-1} \}$. Given access to these oracles, we state a meta-theorem (proved in~\cref{appendix_mainthm}) to characterize the sub-optimality of the algorithm. 
\begin{restatable}{theorem}{MainresultL}
\label{thm_main2}
Suppose~\cref{asmp_b1,asmp_b2} hold and let $f(\mathcal{B}) := \max\{ f_{\mathrm{mdp}}(\mathcal{B}), f_{\mathrm{eva}}(\mathcal{B}) \}$. For $\delta \in (0,1)$,~\cref{alg_CMDPL}  with $U=\frac{2}{\zeta(1-\gamma)}$, $\eta= \frac{U(1-\gamma)}{\sqrt{K}}$, $K=\frac{U^2}{[f(\mathcal{B}]^2(1-\gamma)^2}$ and $b^{\prime}=b-2f(\mathcal{B})$, returns a mixture policy $\bar{\pi}$ satisfying the following condition with probability $1-\delta$,
\begin{align*}
\vbkr(\rho) \geq \vsr(\rho) - 4f(\mathcal{B}) \quad \text{,} \quad \vbkc(\rho) \geq b - 6f(\mathcal{B}) . \quad (\textbf{Relaxed Feasibility Setting})
\end{align*} 
With the same algorithm parameters, but with $b^{\prime} = b+4f(\mathcal{B})$ for $f(\mathcal{B})\leq \frac{\zeta}{6}$,~\cref{alg_CMDPL} returns a mixture policy $\bar{\pi}$ satisfying the following condition with probability $1-\delta$,
\begin{align*}
\vbkr(\rho) \geq \vsr(\rho) - \frac{16f(\mathcal{B})}{\zeta (1-\gamma)} \quad \text{,} \quad \vbkc(\rho) \geq b. \quad (\textbf{Strict Feasibility Setting})
\end{align*} 
\end{restatable}

The above theorem implies that, provided we can adequately control the terms $f_{\mathrm{mdp}}(\mathcal{B})$ and $f_{\mathrm{eva}}(\mathcal{B})$ via the three oracle procedures, both the relaxed feasibility condition~\eqref{fml_RF} and the strict feasibility conditions~\eqref{fml_SF} can be satisfied. Furthermore, we note that similar to~\citep{vaswani2022near}, the error for the strict feasibility setting is inflated by an $O\left(\frac{1}{\zeta \, (1-\gamma)}\right)$ factor. 

Hence, in the next section, we instantiate the subroutines \texttt{DataCollection}, \texttt{MDP-Solver} and \texttt{PolicyEvaluation} such that the quantities $f_{\mathrm{mdp}}(\mathcal{B})$ and $f_{\mathrm{eva}}(\mathcal{B})$ are sufficiently small.

\section{Instantiating the Framework for Linear Constrained MDPs}
\label{sec_linearCMDP}

We first describe the construction of the coreset $\mathcal{C}$, which serves as input to the \texttt{DataCollection} procedure. We then introduce a model-free algorithm, \texttt{LS-MDVI}, as an instantiation of the \texttt{MDP-Solver}. Finally, we present \texttt{LS-PE}, which serves as the instantiation of the \texttt{PolicyEvaluation} subroutine.

\subsection{Data Collection via Core Set Construction}
\label{sec:instant-data}

Recall that the \texttt{DataCollection} procedure requires as input a subset of $\mathcal{S} \times \mathcal{A}$. In the linear setting, we provide a coreset $\mathcal{C}$ as this input. We now describe the construction of the coreset \citep{lattimore2020learning, kitamura2023regularization}. The key properties of the coreset are that it has few elements (independent of the cardinality of $\cS$ and $\cA$), while the features corresponding to the $(x,b) \in \mathcal{C}$ provide a good coverage of the feature space. For a distribution $\tilde \rho$ over $\mathcal{S} \times \mathcal{A}$, let $G \in \mathbb{R}^{d \times d}$ and $g(\tilde\rho) \in \mathbb{R}$ be defined as:
\begin{align*}
G := \sum_{(x,b) \in \mathcal{C}} \tilde{\rho}(x, b) \, \phi(x, b) \phi(x, b)^\top 
\qquad \text{and} \qquad
g(\tilde\rho) := \max_{(s,a) \in \mathcal{S} \times \mathcal{A}} \langle \phi(s,a), G^{-1} \phi(s,a) \rangle .
\end{align*}
We refer to $\tilde \rho$ as the design, $G$ as the corresponding design matrix, and define the coreset of $\tilde\rho$ as its support, $\mathcal{C} := \text{Supp}(\tilde\rho)$. The task of identifying a design that minimizes $g$ is known as the \textit{$G$-optimal design problem}. We assume that we can construct near-optimal experimental design. 
\begin{assumption}[Optimal Design]
We have access to an oracle called \texttt{ComputeOptimalDesign} which returns $\tilde\rho$, $\mathcal{C}$ and $G$ such that
$g(\tilde\rho) \leq 2d$ and the coreset of $\tilde\rho$ has size at most $\tilde{O}(d)$.
\end{assumption}
Such a design can be obtained using the Frank-Wolfe algorithm~\citep{todd2016minimum} described in~\cref{appendix_A}. Accordingly, we first use the \texttt{ComputeOptimalDesign} procedure to construct $\tilde{\rho}$, $\mathcal{C}$, and the associated design matrix $G$, and then utilize the resulting coreset $\mathcal{C}$ to collect data. For each state-action pair in $\mathcal{C}$, we collect $N$ independent samples and store them in the buffer $\mathcal{B}$. Hence, the total sample complexity is $N \, |\mathcal{C}|$. In the subsequent section, it is convenient to consider $\mathcal{B}$ as a union of $T$ disjoint subsets $B_0 \cup \cdots \cup B_{T-1}$, where each $B_i$ consists of $M$ independent samples for every state-action pair in $\mathcal{C}$. Consequently, we have $N = T \, M$.

\subsection{Instantiating the \texttt{MDP-Solver}: Least-Squares Mirror Descent Value Iteration}
\label{sec:instant-mdp}

We now introduce a model-free algorithm referred to as least-squares mirror descent value iteration (\texttt{LS-MDVI}) which serves as an instantiation of the \texttt{MDP-Solver}. 

\texttt{LS-MDVI} is a generalization of \texttt{MDVI}~\citep{geist2019theory, vieillard2020leverage, kozuno2022kl} to the linear function approximation setting and is related to the algorithm proposed in~\citet{kitamura2023regularization}. In particular, \texttt{LS-MDVI} corresponds to a limiting case of policy mirror descent~\citep{lan2023policy} when the KL regularization tends to zero (or equivalently, the step-size tends to infinity). This results in a value iteration method which we describe below. 

Define $\mathcal{H}(\pi(\cdot|s))$ as the entropy of the policy $\pi$ in state $s$ and $\mathrm{KL}(\pi(\cdot|s)||\pi'(\cdot|s))$ as the KL divergence between policies $\pi(\cdot|s)$ and $\pi'(\cdot|s)$ in state $s$. With a slight abuse of notation, we consider $\pi$ to be an operator such that $(\pi Q)(s):= \sum_{a\in \mathcal{A}}\pi(a|s)Q(s,a)$. At iteration $t \in [T]$, \texttt{LS-MDVI} requires the corresponding action-value function to update the policy. Specifically, if $\tau$ is the strength of the KL regularization and $\kappa$ is the entropy regularization coefficient s.t. $\alpha=\frac{\tau}{\tau+\kappa}$,  $\beta=\frac{1}{\tau+\kappa}$, given $Q^{t+1}$ for some reward function, the entropic mirror descent and \texttt{LS-MDVI} updates can be written as: 

\begin{align*}
& \textbf{Entropic Mirror Descent}: \pi_{t+1}(a|s) \propto [\pi_{t}(a|s)]^{\alpha} \exp\left(\beta Q^{t+1}(s, a)\right), \\
&  \hspace{25ex} V^{t+1}(s) = (\pi_{t+1} Q^{t+1})(s) - \tau \mathrm{KL}(\pi_{t+1}(\cdot|s)\| \pi_{t}(\cdot|s)) + \kappa \mathcal{H} (\pi_{t+1}(\cdot|s)). \\
& \textbf{LS-MDVI}: \pi_{t+1}(\cdot|s) = \argmax_{a} \sum_{i=0}^{t+1} Q^i(s,a)  \; \text{;} \; V^{t+1}(s) = \Big(\pi_{t+1} \, \sum_{i=0}^{t+1} Q^{i} \Big)(s) - \Big(\pi_{t} \, \sum_{i=0}^{t} Q^{i}\Big)(s). 
\end{align*}
Starting from entropic mirror descent, for $\kappa = 0$ and as $\tau \to 0$, implying $\alpha = 1$, we recover the \texttt{LS-MDVI} update (see~\citep[App. B]{kozuno2022kl} for the derivation). 

\begin{remark}
In contrast,~\citet{kitamura2023regularization} consider both $\kappa \to 0$, $\tau \to 0$ while keeping $\alpha$ fixed and effectively consider an entropy-regularized update. This proposed change simplifies the algorithm design for \texttt{LS-MDVI}. Furthermore, while the algorithm in~\citet{kitamura2023regularization} produces non-stationary policies, \texttt{LS-MDVI} outputs a stationary policy.     
\end{remark}

Next, we present~\cref{alg_MDVICLMDP} which implements the above \texttt{LS-MDVI} update, but uses the linear CMDP structure and the data collected in the buffer $\cB$ to estimate $Q^{t+1}$. Specifically, $\theta^{t+1}$ estimation in~\cref{alg_MDVICLMDP} corresponds to the $Q^{t+1}$ estimation using linear regression and the last line in~\cref{alg_MDVICLMDP} corresponds to the above update. Similar to approximate value iteration, the $\hat{Q}^{t+1}$ update depends on $\hat{V}^{t}$ via the Bellman equation, however, $\pi_{t+1}$ depends on $\tilde{Q}^{t+1}$, the ``soft'' Q function formed by using the estimates up to iteration $t+1$.

\begin{algorithm}[H]
\SetAlgoLined
\DontPrintSemicolon
\caption{Least-Squares Mirror Descent Value Iteration (\texttt{LS-MDVI})}
\label{alg_MDVICLMDP}

\KwIn{$T$ (number of iterations), $M$ (number of next-state samples obtained per state-action pair in each iteration), $\Box$ (rewards in MDP), $\mathcal{B} = \mathcal{B}_0 \cup \cdots \cup \mathcal{B}_{T-1}$ (Buffer), $\tilde{\rho}$ (design), $\mathcal{C}$ (coreset), $\phi$ (feature map).}
\KwOut{$\pi_{T}$ where $\forall s \in \mathcal{S}$, $\pi_{T}(\cdot|s) \in \argmax_{a} \tilde{Q}_{\Box}^T(s,a).$}
\textbf{Initialize:} $\hat{V}^0_{\Box} = \mathbf{0}$, $\tilde{\theta}^0_{\Box} = \mathbf{0}.$

\BlankLine

\SetKwProg{Proc}{procedure}{}{end}
\Proc{\texttt{LS-MDVI}($T, M, \Box, \mathcal{B}, \mathcal{C}, \phi$)}{
    \For{$t = 0, \dots, T-1$}{ 
        \ \ $\forall (s,a) \in \mathcal{C}:$ Access $ (s,a,s^{\prime}_{m})_{m=1}^{M}$ from the buffer $\mathcal{B}_t$.\\
        Define regression target $\hat{Q}^{t+1}_{\Box} (s, a) := \Box (s, a) + \gamma  {\textstyle \frac{1}{M} \sum\nolimits_{m=1}^M} \hat{V}^t_{\Box}(s^{\prime}_m)$.\\
        $\theta_{\Box}^{t+1} = \underset{\theta \in \mathbb{R}^d}{\argmin} {\textstyle\sum\nolimits_{(x,b)\in \mathcal{C}}} {\tilde{\rho}(x,b)} ( \langle\phi(x,b),\theta \rangle - \hat{Q}^{t+1}_{\Box}(x,b) )^2$.\\
        Define $\tilde{Q}^{t+1}_{\Box} := \langle \phi, {\textstyle\sum\nolimits_{i=0}^{t+1}}\theta^i_{\Box} \rangle;\hat{V}^{t+1}_{\Box} (s) := {\underset{{a}}{\max}} \left\{ \tilde{Q}^{t+1}_{\Box} (s,a) \right\} - {\underset{{a}}{\max}} \left\{  \tilde{Q}^{t}_{\Box} (s,a)\right\}$.
    }
}
\end{algorithm}

In each iteration $t \in [T]$, \cref{alg_MDVICLMDP} uses the buffer $B_t$ consisting of $M$ samples per state-action pair in $\mathcal{C}$. However, since $\hat{V}^{t+1}$ and $\tilde{Q}^{t+1}$ depend on all the past $\theta^{i}$ vectors and hence, on the data collected in the previous iterations, the algorithm can effectively leverage all the data in $\cB$. Furthermore, using the difference between the consecutive $\tilde{Q}$ functions can be viewed as a form of variance reduction. This enables us to prove an $O(1/\sqrt{N})$ concentration result for $\tilde{Q}^{T}$. Moreover, since the \texttt{DataCollection} procedure constructs a coreset which ensures good coverage across the feature space, the resulting sample complexity is independent of the size of the state-action space. Formally, in~\cref{sec_LP1}, we prove the following sub-optimality bound for $\Box = r + \lambda_k \, c$ at iteration $k$ of~\cref{alg_CMDPL}. 
\begin{restatable}{lemma}{OPGuaranteeaL}
\label{lem_optGL}
For a fixed $\epsilon \in (0,1]$, $\delta \in (0,1)$ , and any $k \in [K]$, when using~\cref{alg_MDVICLMDP} at iteration $k$ of~\cref{alg_CMDPL} with $\Box = r+\lambda_k c$, $M= \tilde{O}\left( \frac{dH^2}{\epsilon} \right)$ and $T = {O}\left( \frac{H^2}{\epsilon} \right)$, the output policy $\pi_{T}$ satisfies the following condition with probability $1-\delta$,
\begin{align*}
\max_\pi V^{\pi}_{r+\lambda_k c}(\rho) -  V^{\pi_{T}}_{r+\lambda_k c}(\rho) \leq {O}((1+\lambda_k)\epsilon) 
\end{align*}
\end{restatable}

Hence, with a buffer $\mathcal{B}$ of size $T \, M |\mathcal{C}| = \tilde{O}\left( \frac{d^2 H^4}{\epsilon^2} \right)$,~\cref{alg_MDVICLMDP} guarantees an optimality gap of $f_{\mathrm{mdp}}(\mathcal{B}) = O(\epsilon)$, thereby satisfying~\cref{asmp_b1}. We note that the entropy-regularized variant of the above linear \texttt{MDVI} algorithm ~\citep{kitamura2023regularization} also attains a similar guarantee but for a non-stationary policy output by the corresponding algorithm. Furthermore, in contrast to~\cref{lem_optGL}, the guarantee in~\citet{kitamura2023regularization} only holds for a more restricted range of $\epsilon \in (0,1/H]$. In the next section, we instantiate the \texttt{PolicyEvaluation} oracle. 

\subsection{Instantiating the \texttt{PolicyEvaluation} oracle: Least-Squares Policy Evaluation}
\label{sec:instant-eval}

To understand the need for an explicit \texttt{PolicyEvaluation} oracle, note that in each iteration $k$, we can prove that~\cref{alg_MDVICLMDP} ensures a concentration guarantee for the value function corresponding to $r + \lambda_k c$. However, this does not directly imply a concentration guarantee on the individual value functions corresponding to the reward and constraint rewards. This is in contrast to model-based approaches~\citep{vaswani2022near}  for tabular CMDPs that guarantee concentration for the empirical transition probabilities, and use that to ensure concentration for both the reward and constraint reward value functions. However, since such model-based approaches cannot be used for linear MDPs, we require an additional algorithm that can compute the empirical value functions satisfying~\cref{asmp_b2}. To that end, we present \cref{alg_CVIL} that can be used as an instantiation of the \texttt{PolicyEvaluation} oracle in ~\cref{alg_CMDPL}. The algorithm is also based on least-squares and uses the same coreset constructed in~\cref{sec:instant-data}. Furthermore, we note that~\cref{alg_CVIL} can be viewed as a special case of~\cref{alg_MDVICLMDP} for a fixed policy.

\begin{algorithm}[H]
\SetAlgoLined
\DontPrintSemicolon
\caption{Least-Squares Policy Evaluation (\texttt{LS-PE})}
\label{alg_CVIL}

\KwIn{$T$ (number of iterations), $M$ (number of next-state samples obtained per state-action pair in each iteration), $\diamond$ (either r or c), $\mathcal{B} = \mathcal{B}_0 \cup \cdots \cup \mathcal{B}_{T-1}$ (Buffer), $\pi$ (policy to be evaluated), \\ $\tilde{\rho}$ (design), $\mathcal{C}$ (coreset), $\phi$ (feature map).}
\KwOut{$\bar{\mathcal{V}}^T_{\diamond}(\rho) = \frac{1}{T}\sum_{i=1}^T \hat{\mathcal{V}}^i_{\diamond} (\rho)$.}
\textbf{Initialize:} Define $\hat{\mathcal{V}}^0_{\diamond} = \mathbf{0}$.

\BlankLine

\SetKwProg{Proc}{procedure}{}{end}
\Proc{\texttt{LS-PE}($T$, $M$, $\diamond$, $\mathcal{B}$, $\pi$, $\tilde{\rho}, \mathcal{C}, \phi$)}{
    \For{$t = 0,1,2 \dots,T-1$}{ 
        \ \ $\forall (s,a) \in \mathcal{C}:$ Access $(s,a,s^{\prime}_m)_{m=1}^{M}$ from the buffer $\mathcal{B}_t$. \\
        Define regression target $\hat{\mathcal{Q}}^{t+1}_{\diamond} (s, a) := \diamond (s, a) + \gamma \frac{1}{M} \sum_{m=1}^M \hat{\mathcal{V}}^t_{\diamond}(s^{\prime}_m)$. \\
        $\omega_{\diamond}^{t+1} = \argmin_{\omega \in \mathbb{R}^d} \sum_{(x,b)\in \mathcal{C}} {\tilde{\rho}(x,b)}(\langle \phi(x,b), \omega \rangle- \hat{\mathcal{Q}}^{t+1}_{\diamond}(x,b))^2$.\\
        Define $\hat{\mathcal{V}}_{\diamond}^{t+1}(s) := ({\pi} \, \langle \phi, {\omega}_{\diamond}^{t+1} \rangle)(s)$. 
    }
}
\end{algorithm}

\begin{remark}
There exists prior CMDP literature that uses PE as a subroutine~\citep{efroni2020exploration,ding2021provably}, but our PE subroutine differs in several aspects. First, ~\citet{efroni2020exploration,ding2021provably} use value functions estimated via model-based policy evaluation to update both the primal and dual variables; we utilize the PE subroutine solely for updating the dual variable. This decoupling enables a more reductionist algorithm and modular proof. Moreover, the resulting algorithmsin~\citet{efroni2020exploration,ding2021provably} do not achieve the near-optimal regret even for tabular CMDPs (in terms of their dependence on both $S$ and $H$), and consequently sample-complexity in the generative model setting. In contrast, our method leverages averaging within PE and returns an averaged value function, which plays a crucial role in attaining the near-optimal (in $d$ and $\epsilon$ for the linear setting and in all parameters for the tabular setting) sample complexity. An additional technical challenge lies in integrating the PE-induced error into the primal-dual analysis (i.e., Lemma \ref{lem_pd}) without losing statistical complexity.
\end{remark}

Note that \texttt{LS-PE} uses a fixed dataset (the buffer $\cB$) to evaluate a fixed policy, and is similar to the policy evaluation algorithms in offline reinforcement learning~\citep{duan2020minimax}. The theoretical guarantees for such offline algorithms depend on the quality of the dataset, measured in terms of metrics such as coverage or concentrability. However, in our case, we curate the dataset and choose the buffer $\cB$ such that it has good coverage properties that allow for fine-grained control on the algorithm's sub-optimality. In particular, we prove the following result in~\cref{sec_LP2}.

\begin{restatable}{lemma}{OPGuaranteebL}
\label{lem_optG2L}
For a fixed $\epsilon \in (0,{1}]$, $\delta \in (0,1)$, \cref{alg_CVIL} with $M= \tilde{O}\left( \frac{dH^2}{\epsilon} \right)$ and $T = {O}\left( \frac{H^2}{\epsilon} \right)$, the output $\bar{\mathcal{V}}^T_{\diamond}$ satisfies the following condition with probability $1-\delta$,
\begin{align*}
|\bar{\mathcal{V}}^T_{\diamond}(\rho) - V^{\pi}_{\diamond}(\rho) | \leq {O}\left( \epsilon \right).
\end{align*}
\end{restatable}

Hence, with a buffer $\mathcal{B}$ of size $T \, M |\mathcal{C}| = \tilde{O}\left( \frac{d^2 H^4}{\epsilon^2} \right)$, ~\cref{alg_CVIL} guarantees an optimality gap of $f_{\mathrm{eva}}(\mathcal{B}) = O(\epsilon)$, thereby satisfying~\cref{asmp_b2}. 

\subsection{Putting everything together}
\label{sec:all-linear}

We have seen that~\cref{alg_MDVICLMDP,alg_CVIL} use the buffer $\mathcal{B}$ constructed by the \texttt{DataCollection} procedure to provide control over the terms $f_{\mathrm{mdp}}(\mathcal{B})$ and $f_{\mathrm{eva}}(\mathcal{B})$ appearing in~\cref{thm_main2}. Combining these results, we prove the following corollary in~\cref{appendix_maincor}.
\begin{restatable}{corollary}{MainresultLC}
\label{thm_main2C}
Using \texttt{LS-MDVI} (~\cref{alg_MDVICLMDP}) and \texttt{LS-PE} (~\cref{alg_CVIL}) as instantiations of the \texttt{MDP-Solver} and \texttt{PolicyEvaluation} in \cref{alg_CMDPL} and using the \texttt{DataCollection} oracle described in~\cref{sec:instant-data} has the following guarantee: for a fixed $\epsilon \in (0,{1}]$, $\delta \in (0,1)$, \cref{alg_CMDPL} with $\tilde{O}\left( \frac{d^2 H^4}{\epsilon^2} \right)$ samples, $U=O\left(\frac{1}{\zeta(1-\gamma)}\right)$, $\eta= \frac{U(1-\gamma)}{\sqrt{K}}$, $K=O\left(\frac{1}{\epsilon^2 \, (1-\gamma)^2} \right)$, and $b^{\prime}=b-O(\epsilon)$, returns a mixture policy $\bar{\pi}$ satisfying the following condition with probability $1-\delta$,
\begin{align*}
\vbkr(\rho) \geq \vsr(\rho) - O(\epsilon), \quad \text{and} \quad \vbkc(\rho) \geq b - O(\epsilon).
\end{align*} 
With the same algorithm parameters, but with $b^{\prime} = b + O(\epsilon)$ and $\tilde{O}\left( \frac{d^2 H^6}{\zeta^2\epsilon^2} \right)$ samples, \cref{alg_CMDPL} returns a mixture policy $\bar{\pi}$ satisfying the following condition with probability $1-\delta$,
\begin{align*}
\vbkr(\rho) \geq \vsr(\rho) - O({\epsilon}), \quad \text{and} \quad \vbkc(\rho) \geq b.
\end{align*} 
\end{restatable}

Hence, the total sample complexity required to achieve the relaxed feasibility objective in~\cref{fml_RF} and the strict feasibility objective in~\cref{fml_SF} is $\tilde{O}\left( \frac{d^2 H^4}{\epsilon^2} \right)$ and $\tilde{O}\left( \frac{d^2 H^6}{\epsilon^2 \zeta^2} \right)$ respectively. Since the lower bound for unconstrained linear MDPs is $\Omega\left( \frac{d^2 H^3}{\epsilon^2} \right)$ \citep{weisz2022confident}, our sample complexity achieves the optimal dependence on $d$ and $\epsilon$ in the relaxed setting. 

\paragraph{Flexibility of \cref{alg_CMDPL}.} Note that instead of \texttt{LS-MDVI}, we can use other unconstrained linear MDP solvers. For example, the G-Sampling and Stop (\texttt{GSS}) algorithm from~\citet{taupin2023best} uses a different \texttt{DataCollection} procedure and algorithm to return an $\epsilon$-optimal policy. It requires $\tilde{O}\left( \frac{d^2 H^4}{\epsilon^2} \right)$ samples to do so, thus matching the sample complexity of \texttt{LS-MDVI}. We describe this algorithm in detail and formally instantiate~\cref{alg_CMDPL} in~\cref{app:gss}. 

\subsection{Lower Bound}
\label{sec:all-linear}
As shown by \citet{vaswani2022near}, under the relaxed feasibility setting the statistical complexity of solving CMDPs matches that of unconstrained MDPs. In contrast, solving CMDPs under the strict setting is more challenging. In the tabular case, \citet{vaswani2022near} established a lower bound using a linear-programming argument. In this section, we establish a lower bound for solving linear CMDPs under strict feasibility. The proof technique of the following result is simpler and fundamentally different from that of \citet{vaswani2022near}. The proof is presented in Section~\ref{sec_LB}.

\begin{restatable}{theorem}{LB}
Let $\delta\in(0,0.08]$, $\gamma\in[7/12,1)$, $H=1/(1-\gamma)$, $\epsilon\in(0,0.002)$, $\zeta\in (0,49/2280)$, $b\in[H/2,H]$, and $d\geq 6$. There exists a class of linear constrained MDPs such that any $(\epsilon,\delta)$-sound algorithm requires $\Omega\left(d^2 H^5/\epsilon^2\zeta^2\right)$ samples from the generative model in the worst case.
\end{restatable}
Thus, the dependence on $d$, $\epsilon$, $\zeta$ in our bounds for the strict feasibility setting is also tight, with a suboptimality arising only in the multiplicative dependence on $H$.\\

\begin{figure}[t!]
\begin{minipage}[b]{0.73\textwidth}
    \centering
    \begin{tikzpicture}[-latex, auto, node distance = 3.5 cm and 3.5 cm, on grid, semithick, state/.style={circle, draw, minimum width = 1.25cm}]
      \node[state] (A) {$s_0$};
      \node[state] (B) [right = of A] {$s_1$};
      \node[state] (C) [below = of A] {$o$};
      \node[state] (D) [below = of B] {$z$};
      \path (A) edge [bend right = 15] node[align=center, above = 0.05 cm] {$r=1$ \\ $c=u$} (B);
      \draw [->] (A) to[in=140, out=230, looseness=3.5]  node[align=center, left = 0.05 cm] {$r=1$ \\ $c=u$} (A);
      \draw [->] (B) to[in=310, out=50, looseness=4.0] node[align=center, right = 0.05 cm] {$r=0$ \\ $c=0$} (B);
      \path (C) edge [bend right = 15] node[align=center, left = 0.05 cm] {$r=0$ \\ $c=0$} (A);
      \path (C) edge [bend right = 15] node[align=center, above = 0.1 cm] {$r=0$ \\ $c=0$} (D);
      \draw [->] (D) to[in=310, out=50, looseness=4.0] node[align=left, right = 0.05 cm] {$r=0$ \\ $c=(b+\zeta)(1+\gamma)/\gamma$} (D);
    \end{tikzpicture}
        \end{minipage}\hfill
    \begin{minipage}[b]{0.27\textwidth} 
        \captionof{figure}{The lower bound instance consists of CMDPs with four states. $o$ is the fixed starting state. At state $o$, taking an action will either transition to state $s_0$ or to the ``safe" state $z$. At state $s_0$, taking an action will either transition to state $s_1$ or stay in $s_0$. States $z$ and $s_1$ are absorbing.}
        \label{fig:parameter-update}
    \end{minipage}
\end{figure}


\subsection{Discussion}
\paragraph{Why variance-weighted least squares fails in our setting.} For unconstrained linear MDPs,~\citet{kitamura2023regularization} provide an alternative entropy-regularized algorithm that constructs coresets that depend on the estimated empirical variance in the value function. The resulting algorithm uses variance-weighted least squares and is able to attain the near-optimal $O\left( \frac{d^2 H^3}{\epsilon^2} \right)$ sample complexity for unconstrained linear MDPs. To the best of our knowledge, this is the only algorithm that can achieve such an optimal bound. Unfortunately, using such an idea for linear CMDPs fails. This is because in the linear CMDP setting, since the MDP reward function $r + \lambda_k \, c$ (and hence the MDP value function) change in every iteration $k$ of~\cref{alg_CMDPL}, using variance-aware coresets implies that we need to construct a distinct coreset in every such iteration. This prevents the resulting algorithm from reusing data similar to~\cref{alg_CMDPL}, and actually increases the corresponding sample complexity. Resolving this issue and attaining the optimal dependence on $H$ is an important direction for future work.

\paragraph{Comparing with online methods.} In order to further contextualize our results, we use the state-of-the-art regret guarantees for the finite-horizon online setting~\citep{ghosh2022provably} and use the reduction in~\citet{bai2021achieving} to our problem setting. The reduction implies that the algorithm in~\citep{ghosh2022provably} (designed and analyzed for the more difficult online regret minimization) results in an $\tilde{O}\left( \frac{d^3 H^4}{\epsilon^2} \right)$ and $\tilde{O}\left( \frac{d^3 H^6}{\epsilon^2 \zeta^2} \right)$ sample complexity for the relaxed and strict settings respectively. Hence, our results have a better dimension dependence. Interestingly, the analysis in~\citep{ghosh2022provably} has a worse dependence on $d$ because it uses a uniform concentration argument to get a handle on the concentration for the individual value functions corresponding to the (constraint) rewards. Recall that in~\cref{sec:instant-eval}, we encountered a similar issue and resolved it by using policy evaluation. We believe that our technique might be useful even for online regret minimization.

\paragraph{Comparing with offline methods.} Since we store all the samples at the beginning (see~\cref{alg_CMDPL}), an alternative approach is to treat these stored samples as an offline dataset, and use a standard offline RL algorithm~\citep{neu2024offline,hong2024primal} for the CMDP setting. It is worth noting that the sample complexity for such a reduction from the offline constrained RL setting can be substantially worse or even vacuous. In particular, the result on linear CMDP (e.g., Theorem 9) in~\citet{hong2024primal} depends on the concentrability coefficient $C^*$ (see Assumption 2 therein), which can become unbounded if we do not collect data that covers the optimal action in every state. We note that constructing the $O(d^2)$ size coreset as in our paper does not directly ensure such coverage for each optimal action. When using the result from~\citet{hong2024primal}, this can lead to unbounded values of $C^*$ and a vacuous bound. For the sake of argument, even if we assume that $C^*$ is bounded by an absolute constant (independent of $d$ and $H$), the bounds for the offline CMDP setting in~\citet{hong2024primal} are worse than ours. Specifically,~\citet[Theorem 9]{hong2024primal} is established under the relaxed feasibility setting; however, their sample complexity bound exhibits a quadratic dependence on the Slater constant, while our bounds in this setting do not have such a dependence. It is known that the Slater constant is only relevant in the strict feasibility regime~\citep{vaswani2022near}. Additionally, the dependence on the feature dimension $d$ is $d^3$ in~\citet{hong2024primal}, whereas our approach achieves the near-optimal dependence of $d^2$.

\section{Conclusion}
\label{sec:discussion}

Given access to a generative model, we proposed a generic primal-dual framework for reducing the (linear) CMDP problem to the (linear) MDP problem. Using (linear) \texttt{MDVI} as the \texttt{MDP-Solver} enabled us to obtain sample complexity bounds for both tabular and linear CMDPs with either $O(\epsilon)$ or zero constraint violation. We obtained the first near-optimal (in $d$, $\epsilon$, and $\zeta$) guarantees for linear CMDPs, whereas for tabular CMDPs, we matched the existing near-optimal guarantees. We also provide a lower bound for solving linear CMDPs under strict feasibility. For linear CMDPs, improving the dependence of the sample complexity on the effective horizon $H$ is an important direction for future work.


\bibliography{ref}

\newpage
\appendix

\newcommand{\appendixTitle}{%
\vbox{
    \centering
	\hrule height 4pt
	\vskip 0.2in
	{\LARGE \bf Supplementary material}
	\vskip 0.2in
	\hrule height 1pt 
}}
\appendixTitle

\startcontents 
\printcontents{}{1}{} 

\allowdisplaybreaks

\section{An Instantiation of \texttt{ComputeOptimalDesign}}
\label{appendix_A}

In this section, we present an instantiation of the \texttt{ComputeOptimalDesign} oracle using the \texttt{Frank-Wolfe} algorithm~\citep{todd2016minimum}.

\begin{algorithm}[H]
\SetAlgoLined
\DontPrintSemicolon
\caption{\texttt{InitializeDesign}}
\label{alg_ID}

Choose an arbitrary nonzero $c_0 \in \mathbb{R}^d$.\\
\textbf{Output: }{$\tilde{\rho}$.}

\BlankLine

\SetKwProg{Proc}{procedure}{}{end}
\Proc{\texttt{InitializeDesign}}{
    \For{$j = 0,1,2 \dots,d-1$}{ 
        $(\bar{s}_j, \bar{a}_j) = \argmax_{(s, a) \in \mathcal{S} \times \mathcal{A}} c_j^{\top} \phi(s, a)$.\\
        $(s_j, a_j) = \argmin_{(s, a) \in \mathcal{S} \times \mathcal{A}} c_j^{\top} \phi(s, a)$.\\
        $x_j = \phi(\bar{s}_j, \bar{a}_j) - \phi(s_j, a_j)$.\\
        Choose an arbitrary nonzero $c_{j+1}$ orthogonal to $x_0, \dots, x_j$. \\
    }
    \ Let $\mathcal{Z} := \{(\bar{s}_j, \bar{a}_j), (s_j, a_j) \mid j = 0, \dots, d - 1\}$.\\
    Choose $\tilde{\rho}$ to put equal weight on each of the distinct points of $\mathcal{Z}$.
}
\end{algorithm}

Below we present the classical \texttt{Frank-Wolfe} algorithm for experimental design.

\begin{algorithm}[H]
\SetAlgoLined
\DontPrintSemicolon
\caption{\texttt{Frank-Wolfe}}
\label{alGW}

\KwIn{$\epsilon^{FW}$. (Tolerance for algorithm)}
\KwOut{$\tilde{\rho}, \mathcal{C}, G$. (Coreset, optimal design and covariance matrix)}
\textbf{Initialize:} Define $\hat{\mathcal{V}}^0_{\diamond} = \mathbf{0}$.

\BlankLine

\SetKwProg{Proc}{procedure}{}{end}
\Proc{\texttt{Frank-Wolfe}($\epsilon^{FW}$)}{
    \ \ $\tilde{\rho}=$ \texttt{InitializeDesign} by Algorithm \ref{alg_ID}. \\
    Define $\mathcal{U} : \tilde\rho \mapsto \text{diag}(\tilde\rho) \in \mathbb{R}^{|\mathcal{S}||\mathcal{A}| \times |\mathcal{S}||\mathcal{A}|}$, where $\text{diag}(\tilde{\rho})$ is a diagonal matrix with elements of $\tilde\rho$.\\
    For $(s, a) \in \mathcal{S} \times \mathcal{A}$, let $\Phi \in \mathbb{R}^{|\mathcal{S}||\mathcal{A}| \times d}$ be a matrix such that its $(s|\mathcal{A}| + a)$th row is $\phi(s, a)$.\\
    Define $\mathcal{I} : \tilde\rho \mapsto (\Phi^{\top} \, \mathcal{U}(\tilde\rho) \, \Phi)^{-1}$. (defines the inverse of the covariance matrix)\\
    Let $\nu : (s, a, \tilde\rho) \mapsto \phi(s, a)^{\top} \mathcal{I}(\tilde\rho) \phi(s, a)$. (measures the variance proxy for $(s,a)$)\\
    Let $\delta : \tilde\rho \mapsto \max_{(s, a) \in \mathcal{S} \times \mathcal{A}} (\nu(s, a, \tilde\rho) - d)/d$ (computes the relative difference between the worst-case variance and $d$)
        
    \While{$\delta(\tilde\rho) > \epsilon^{FW}$}{ 
        \ \ Let $(x, b) := \arg \max_{(s, a) \in \mathcal{S} \times \mathcal{A}} \nu(s, a, \tilde\rho)$.\\
        Let $\eta^* := (\nu(x, b, \tilde\rho) - d) / ((d - 1) \nu(x, b, \tilde\rho))$.\\
        $\tilde\rho(x, b) \gets \tilde\rho(x, b) + \eta*$.\\
        $\tilde\rho \gets \tilde\rho / (1 + \eta^*)$.
    }
    \ \ \ Let $\mathcal{C} := \left\{(s, a) \mid \nu(s, a, \tilde\rho) \geq d \left( 1 + \frac{\delta(\tilde\rho) d}{2} - \sqrt{\delta(\tilde\rho) (d - 1) + \frac{\delta(\tilde \rho)^2 d^2}{4}} \right) \right\}$. \\ \ \ \ (form the coreset containing state-action pairs with sufficiently high variance value)\\
    \vspace{1ex}
    Let $G := \sum_{(x, b) \in \mathcal{C}} {\tilde\rho(x, b)} \phi(x, b) \phi(x, b)^{\top}$. (calculate the corresponding covariance matrix)
}
\end{algorithm}

The subroutine \texttt{InitializeDesign} returns an initial design to be used in \texttt{Frank-Wolfe}. \texttt{InitializeDesign} is a deterministic procedure for constructing a core set of state-action pairs that provides good coverage of the feature space in linear MDPs. The algorithm sequentially identifies informative directions in the feature space by iteratively computing difference vectors between state-action pairs with maximal and minimal feature projections along a given search direction. The algorithm iteratively updates the search direction to be orthogonal to the span of the previously discovered directions. Specifically, the vector $c_j \in \mathbb{R}^d$ is an auxiliary direction vector used to sequentially identify maximally informative state-action pairs. The next vector $c_{j+1}$ is then chosen to be orthogonal to all previous $x_0,\dots,x_j$ ensuring that the design explores linearly independent directions in feature space.  The resulting set of state-action pairs is then used as the support for a design distribution in regression.

\section{Table of Notation}
\label{Appendix_Table-Linear}

\begin{table}[h]
\label{Table1}
\centering
\begin{tabular}{l|l}
\hline 
\thead{\textbf{Notation}} & \thead{\textbf{Meaning}} \\
\hline
$\mathcal{A}, \mathcal{S}$ & action space of size $|\mathcal{A}|$, state space of size $|\mathcal{S}|$ \\
$\gamma, H$ & discount factor in $[0, 1)$, $1/(1-\gamma)$ \\
$P$ & transition matrix  $P \in \mathbb{R}^{|S||A|\times |S|}$ \\
$P_{\pi}, \hat{P}_{\pi}^t$ & $\pi P \in \mathbb{R}^{|S|\times |S|}$, $\pi \hat{P}_t \in \mathbb{R}^{|S|\times |S|}$ \\
$r,c$ &  reward vector in [0,1] range, constraint reward vector in [0,1] range  \\
$\rho$ & initial distribution of states\\
$\diamond$ & $r$ or $c$ \\
$\Box$ &  $r+\lambda c$ where $\lambda \in \{\lambda_1,\cdots,\lambda_K\}$ \\
\hline
$b, \zeta$ & constraint value in $[0, 1/(1-\gamma))$, Slater constant \\
$\lambda, \lambda^*$& Lagrange multiplier, the optimal Lagrange multiplier \\
$U$ & projection upper bound \\
\hline
$\phi,d$ & feature map of a linear MDP and its dimension  \\
$\tilde{\rho}, \mathcal{C}$ & a design over $\mathcal{S} \times \mathcal{A}$, coreset \\
$G$ & design matrix with respect to $\phi$ and $\tilde{\rho}$. Equal to $\sum_{(x,b) \in \mathcal{C}} \tilde{\rho}(x,b) \phi(x,b) \phi(x,b)^{\top} $\\
$W(z)$ & $G^{-1} \sum_{(x,b) \in \mathcal{C}} {\tilde{\rho}(x,b) \phi(x,b) z(x,b)}$ \\ &(solution of a least-squares estimation with features $\phi(x,b)$, weights $\tilde{\rho}$ and targets $z(x,b)$)\\
\hline
$\epsilon, \delta$ & admissible suboptimality, admissible failure probability \\
$K, T$ & number of outer and inner iterations \\
\hline
$(\hat{P}_t \hat{V}_{\diamond}^t)(s,a)$ & $ \frac{1}{M} \sum_{m=1}^M\hat{V}_{\diamond}^t(s^{\prime}_m)$ where $s^{\prime}_m \in \mathcal{B}_t$ \\
$(P\hvkd)(s,a)$ & $\mathbb{E}[\hvkd(s^{\prime}) \, | \, s_0=s, a_0 = a]$\\
$\mathcal{F}_{t,m}$ & $\sigma$-algebra in the filtration for \cref{alg_MDVICMDP,alg_MDVICLMDP,alg_CVI,alg_CVIL} \\ 
\hline
$\mathcal{T}^{\pi}Q$ & Bellman operator $r + \gamma P(\pi Q)$\\
$Q^{\pi}$ & state-action value function for policy $\pi$ \\
\hline
$\hqktb{t}$ & estimated state-action value function in iteration $t$ in \cref{alg_MDVICMDP,alg_MDVICLMDP}  \\
$\hmqktd{t}$ & estimated state-action value function in iteration $t$ in \cref{alg_CVI,alg_CVIL}  \\
\hline
\end{tabular}
\end{table}

\begin{table}[h]
\label{Table1}
\centering
\begin{tabular}{l|l}
\hline 
$\hvkb(s)$ (Tabular) & ${\underset{{a}}{\max}} \left\{ \sum_{i=0}^{t} \hqkib(s,a) \right\}  - {\underset{{a}}{\max}}  \left\{ \sum_{i=0}^{t-1} \hqkib(s,a)\right\} $ in~\cref{alg_MDVICMDP}\\
$\hvkb(s)$ (Linear) & $ {\underset{{a}}{\max}} \left\{(\langle \phi, \sum_{i=0}^{t}\theta^i_{\Box} \rangle) (s,a) \right\} - {\underset{{a}}{\max}} \left\{  (\langle \phi, \sum_{i=0}^{t-1}\theta^i_{\Box} \rangle) (s,a)\right\}$ in~\cref{alg_MDVICLMDP} \\
$\tilde{Q}_{\Box}^{t}$ (Tabular) & $\sum_{i=0}^{t}\hat{Q}_{\Box}^i$ in~\cref{alg_MDVICMDP}\\
$\tilde{Q}_{\Box}^{t}$ (Linear) & $\langle \phi, \sum_{i=0}^t \theta^i_{\Box} \rangle$ in~\cref{alg_MDVICLMDP} \\
$\hmvkd(s)$ (Tabular) & $(\pi \hmqktd{t})(s)$ in~\cref{alg_CVI}\\
$\hmvkd(s)$ (Linear) & $({\pi} \, \langle \phi, {\omega}_{\diamond}^{t} \rangle)(s)$ in~\cref{alg_CVIL} \\
\hline
$\bvktb(s), \bmvktd(s)$ & $\frac{1}{t}\sum_{i=1}^{t} \hvkib(s), \frac{1}{t}\sum_{i=1}^{t} \hmvkid(s)$ \\
\hline
$\hat{V}^k_{\diamond},\bmvd$ & output of the \texttt{PolicyEvaluation} oracle in line 5 in Algorithm \ref{alg_CMDPL}, $\frac{1}{K} \sum_{k=0}^{K-1} \hat{V}^k_{\diamond}$\\
\hline
$\pi_{k}$ & output policy of \texttt{MDP-Solver}\\ 
$\bar{\pi}$ & mixture policy equal to $\frac{1}{K} \sum_{k=0}^{K-1}\pi_{k}$\\
$\pi^*$ & $ \mathrm{argmax}_{\pi}  V^{\pi}_r(\rho) \text{ s.t. } V^{\pi}_c(\rho) \geq b$\\
$\pi^{*+}$ & $ \mathrm{argmax}_{\pi}  V^{\pi}_r(\rho) \text{ s.t. } V^{\pi}_c(\rho) \geq b + 6f(\mathcal{B})$\\
$\pi^*_k$ & $ \mathrm{argmax}_{\pi} \{ V^{\pi}_{r+\lambda_k c} \}$\\
$\pi^{\prime}_{t}$ &  a non-stationary policy that follows policies $\pi_t, \pi_{t-1}, \dots$ \\
& upto timestep $t$ and follows $\pi_0$ thereafter \\
\hline
$(\pi Q)(s)$ & $\sum_{a\in \mathcal{A}}\pi(a|s) \, Q(s,a)$ \\
$(\pi r)(s)$ & $\sum_{a\in \mathcal{A}}\pi(a|s) \, r(s,a)$ \\
\hline
${\theta}_{\Box}^{t}$ & least-squares value estimate in \cref{alg_MDVICLMDP} \\
$\boldsymbol{\theta}_{\Box}^{t}$ & parameter that satisfies $\langle \phi, \boldsymbol{\theta}_{\Box}^{t} \rangle := \Box + \gamma P \hat{V}_{\Box}^{t-1}$ in the linear MDP\\
${\omega}_{\diamond}^{t}$ &least-squares value estimate in \cref{alg_CVIL} \\
$\boldsymbol{\omega}_{\diamond}^{t}$ & parameter that satisfies $\langle \phi, \boldsymbol{\omega}_{\diamond}^{t} \rangle := \diamond + \gamma P \hat{\mathcal{V}}_{\diamond}^{t-1}$ in the linear MDP\\
\hline
\end{tabular}
\end{table}

\section{Proof of Theorem \ref{thm_main2}}
\label{appendix_mainthm}

\MainresultL*

\begin{proof}
We denote $\bmvd = \frac{1}{K} \sum_{k=0}^{K-1} \hat{V}^k_{\diamond} $ where $\diamond=r$ or $c$.
We first prove the relaxed feasibility statement. By Lemma \ref{lem_pd}, we have $\bmvc(\rho) \geq b -  5f(\mathcal{B})$. Hence,
\begin{align*}
\vbkc(\rho) &= \vbkc(\rho) - \bmvc(\rho) + \bmvc(\rho)\\
&\geq b -  5f(\mathcal{B}) - |\vbkc(\rho) - \bmvc(\rho)|\\
&\geq b -  5f(\mathcal{B}) - f(\mathcal{B}) \tag{By Assumption \ref{asmp_b2} for each policy $\{\pi_k\}_{k = 0}^{K-1}$} \\
&= b - 6f(\mathcal{B}).
\end{align*}
Next, we prove $\vsr(\rho)- \vbkr(\rho) \leq 4f(\mathcal{B}).$ We have
\begin{align*}
\vsr(\rho)- \vbkr(\rho) & = [\vsr(\rho) - \bmvr(\rho)] + [\bmvr(\rho) - \vbkr(\rho)] \\
& \leq 3f(\mathcal{B}) + \vert \bmvr(\rho) - \vbkr(\rho) \vert \tag{By Lemma \ref{lem_pd}} \\
& \leq 3f(\mathcal{B}) + f(\mathcal{B}) \tag{By Assumption \ref{asmp_b2} for each policy $\{\pi_k\}_{k = 0}^{K-1}$}\\
& = 4f(\mathcal{B}).
\end{align*}
Now we prove the strict feasibility statement.
By Lemma \ref{lem_pd}, we have $\bmvc(\rho) \geq b + f(\mathcal{B})$, and thus,
\begin{align*}
\vbkc(\rho) &= \vbkc(\rho) - \bmvc(\rho) + \bmvc(\rho)\\
&\geq b + f(\mathcal{B})- |\vbkc(\rho) - \bmvc(\rho)|\\
&\geq b + f(\mathcal{B}) - f(\mathcal{B}) \tag{By Assumption \ref{asmp_b2} for each policy $\{\pi_k\}_{k = 0}^{K-1}$}\\
&\geq b,
\end{align*}
which satisfies the constraint. Next, we prove $ \vsr(\rho)- \vbkr(\rho) \leq 28f(\mathcal{B})$. We define $ \pi^{*+}\in \mathrm{argmax}_{\pi}  V^{\pi}_r(\rho) \text{ s.t. } V^{\pi}_c(\rho) \geq b + 6f(\mathcal{B})$. Note that such a policy exists by the definition of $\zeta$ and the assumption that $f(\mathcal{B}) \leq \frac{\zeta}{6}$.
By Lemma \ref{lem_boundsense} and Lemma \ref{lem_bounddualv}, we know that 
\begin{align*}
|\vsr(\rho) - \vssr (\rho)| \leq 12f(\mathcal{B}) \lambda^* \leq \frac{12f(\mathcal{B})}{\zeta(1-\gamma)}.
\end{align*}
Applying Lemma \ref{lem_pd} and Assumption \ref{asmp_b2} as before, we have
\begin{align*}
\vsr(\rho)- \vbkr(\rho) &= [\vsr(\rho) - \vssr(\rho)] + [\vssr(\rho) - \bmvr(\rho)] + [\bmvr(\rho) - \vbkr(\rho)] \\
&\leq \frac{12f(\mathcal{B})}{\zeta(1-\gamma)} + 3f(\mathcal{B}) + f(\mathcal{B}) \\
&\leq \frac{16f(\mathcal{B})}{\zeta(1-\gamma)}. \tag{$\zeta(1-\gamma)\leq \frac{1-\gamma}{1-\gamma}=1$}
\end{align*}
This completes the proof.
\end{proof}

\subsection{Proof of Lemma \ref{lem_pd} (Primal-Dual Guarantees for Algorithm \ref{alg_CMDPL})}

\begin{restatable}[Primal-Dual Guarantees for Algorithm \ref{alg_CMDPL}]{lemma}{PDGuarantee}
\label{lem_pd}
Suppose~\cref{asmp_b1,asmp_b2} hold and let $f(\mathcal{B}) := \max\{ f_{\mathrm{mdp}}(\mathcal{B}), f_{\mathrm{eva}}(\mathcal{B}) \}$. For $\delta \in (0,1)$, when ~\cref{alg_CMDPL} is run with $U=\frac{2}{\zeta(1-\gamma)}$, $\eta= \frac{U(1-\gamma)}{\sqrt{K}}$, $K=\frac{U^2}{[f(\mathcal{B})]^2 \, (1-\gamma)^2}$ and $b^{\prime}=b-2f(\mathcal{B})$, the following condition holds with probability $1-\delta$,
\begin{align*}
\frac{1}{K}\sum_{k=0}^{K-1} \hat{V}_r^k (\rho) \geq \vsr(\rho) - 3f(\mathcal{B}) \quad \text{,} \quad \frac{1}{K}\sum_{k=0}^{K-1} \hat{V}_c^k(\rho) \geq b - 5f(\mathcal{B}).
\end{align*} 
With the same algorithm parameters, but with $b^{\prime} = b+4f(\mathcal{B})$, the following condition holds with probability $1-\delta$,
\begin{align*}
\frac{1}{K}\sum_{k=0}^{K-1} \hat{V}_r^k (\rho) \geq \vssr(\rho) - 3{f(\mathcal{B})} \quad \text{,} \quad \frac{1}{K}\sum_{k=0}^{K-1} \hat{V}_c^k(\rho) \geq b + f(\mathcal{B}). 
\end{align*} 
\end{restatable}

\begin{proof}
We begin by proving the first part of the lemma. Since both $r$ and $c$ are bounded by 1, we note that $r(s,a)+\lambda_kc(s,a) \leq 1 + \lambda_k$ for all $(s,a) \in \mathcal{S}\times \mathcal{A}$. Define $\pi^*_k := \argmax_{\pi} V^{\pi}_{r+\lambda_k \, c}$ as an optimal policy in the MDP with rewards $r + \lambda_k \, c$. 
For each iteration $k$ in \cref{alg_CMDPL}, by Assumption \ref{asmp_b1} with $R = 1 + \lambda_k$, we have
\begin{align*}
\vkr(\rho) +  \lambda_{k} \vkc(\rho) - V^{\pi_k}_{r+\lambda_k c} \leq f_{\mathrm{mdp}}(\mathcal{B}) (1+ \lambda_k).
\end{align*}
By Assumption \ref{asmp_b2} for policy $\pi_k$, we have
\begin{align*}
V^{\pi_k}_{r+\lambda_kc}(\rho) - \hat{V}_r^k(\rho) - \lambda_k \hat{V}_c^k (\rho) &= V^{\pi_k}_{r}(\rho) + \lambda_k V^{\pi_k}_{c} (\rho)  - \hat{V}_r^k(\rho) - \lambda_k \hat{V}_c^k (\rho)\\
&= V^{\pi_k}_{r} (\rho) - \hat{V}_r^k(\rho) + \lambda_k (V^{\pi_k}_{c}(\rho)    -  \hat{V}_c^k (\rho))\\
&\leq f_{\mathrm{eva}}(\mathcal{B}) (1+ \lambda_k).
\end{align*}
Combining the above inequalities and letting $f(\mathcal{B})=\max\{f_{\mathrm{mdp}}(\mathcal{B}),f_{\mathrm{eva}}(\mathcal{B})\}$, we obtain
\begin{align*}
\vkr(\rho) +  \lambda_{k} \vkc(\rho) - ( \hat{V}_r^k(\rho) + \lambda_k \hat{V}_c^k (\rho)) &\leq (f_{\mathrm{mdp}}(\mathcal{B})+f_{\mathrm{eva}}(\mathcal{B})) (1 + \lambda_k) \leq 2f(\mathcal{B}) \, (1 + \lambda_k).
\end{align*}
By the definition of $\pi^*_k$,
\begin{align*}
\vsr(\rho) +  \lambda_{k} \vsc(\rho) \leq \vkr(\rho) +  \lambda_{k} \vkc(\rho).
\end{align*}
Therefore, by combining the above inequalities, 
\begin{align}
\vsr(\rho) +  \lambda_{k} \vsc(\rho) &\leq \hat{V}_r^k(\rho) +  \lambda_k \hat{V}_c^k(\rho) + 2f(\mathcal{B}) \, (1 + \lambda_k) \label{eq:pd_inter-1} \\
\Longrightarrow \vsr(\rho) - \hat{V}_r^k(\rho) &\leq \lambda_k (\hat{V}_c^k(\rho) - \vsc(\rho) + 2f(\mathcal{B})) + 2f(\mathcal{B}). \nonumber
\end{align}
Since $\vsc(\rho) \geq b$ and $\lambda_k \geq 0$, we obtain
\begin{align*}
\vsr(\rho) - \hat{V}_r^k(\rho) &\leq \lambda_k (\hat{V}_c^k(\rho) - b + 2f(\mathcal{B})) + 2f(\mathcal{B}).
\end{align*}
By taking the average, letting $b^{\prime} = b - 2f(\mathcal{B})$, and adding both sides by the same term $\frac{\lambda}{K} \sum_{k=0}^{K-1} \left[  b^{\prime} - \hat{V}_c^k(\rho) \right]$,
\begin{align*}
\frac{1}{K} \sum_{k=0}^{K-1} \left[ \vsr(\rho) - \hat{V}_r^k(\rho)  \right] + \frac{\lambda}{K} \sum_{k=0}^{K-1} \left[  b^{\prime} - \hat{V}_c^k(\rho) \right] \leq  \frac{1}{K} \sum_{k=0}^{K-1} (\lambda_k - \lambda) (\hat{V}_c^k(\rho) - b^{\prime} ) + 2f(\mathcal{B}).
\end{align*}
Now we define $R(\lambda,K) := \sum_{k=0}^{K-1} (\lambda_k - \lambda) (\hat{V}_c^k(\rho) - b^{\prime} )$ as the dual regret and denote $\bmvd = \frac{1}{K} \sum_{k=0}^{K-1} \hat{V}^k_{\diamond} $ (where $\diamond=r$ or $c$). Thus, for any $\lambda \in [0,U]$, 
\begin{align}
\vsr(\rho) - \bmvr(\rho) + \lambda(  b^{\prime} - \bmvc(\rho)) \leq  \frac{R(\lambda,K)}{K}  + 2f(\mathcal{B}).  \label{lem_PD_3}
\end{align}
Below we show that for any $\lambda \in[0, U]$, the following bound holds for the dual regret:
\begin{align*}
R(\lambda,K) \leq \frac{U \sqrt{K}}{1-\gamma}.
\end{align*} 
Using the dual update in~\cref{alg_CMDPL}, we observe that, 
\begin{align*}
\left|\lambda_{k+1}-\lambda\right|^{2} &\leq \left|\lambda_{k}-\eta\left(\hat{V}_c^k(\rho) -b^{\prime}\right)-\lambda\right|^{2} \tag{by non-expansiveness of projection} \\
& = \left|\lambda_{k}-\lambda\right|^{2}-2 \eta\left(\lambda_{k}-\lambda\right)\left(\hat{V}_c^k(\rho)-b^{\prime} \right)+\eta^{2}\left(\hat{V}_c^k(\rho)-b^{\prime}\right)^{2} \\
&\overset{(a)}{\leq} \left|\lambda_{k}-\lambda\right|^{2}-2 \eta\left(\lambda_{k}-\lambda\right)\left(\hat{V}_c^k(\rho)-b^{\prime} \right)+\frac{\eta^{2}}{(1-\gamma)^{2}},
\end{align*}
where (a) follows because $b$ and the constraint value are in the $[0,1 /(1-\gamma)]$ interval. Rearranging and dividing by $2 \eta$, we get
$$
\left(\lambda_{k}-\lambda\right)\left(\hat{V}_c^k(\rho)-b^{\prime} \right) \leq \frac{\left|\lambda_{t}-\lambda\right|^{2}-\left|\lambda_{k+1}-\lambda\right|^{2}}{2 \eta}+\frac{\eta}{2(1-\gamma)^{2}} .
$$
Summing from $k=0$ to $K-1$ and using the definition of the dual regret,
$$
R(\lambda, K) \leq \frac{1}{2 \eta} \sum_{k=0}^{K-1}\left[\left|\lambda_{k}-\lambda\right|^{2}-\left|\lambda_{k+1}-\lambda\right|^{2}\right]+\frac{\eta K}{2(1-\gamma)^{2}}.
$$
Telescoping, bounding $\left|\lambda_{0}-\lambda\right|$ by $U$ and dropping a negative term gives
$$
R(\lambda, K) \leq \frac{U^{2}}{2 \eta}+\frac{\eta K}{2(1-\gamma)^{2}},
$$
Setting $\eta=\frac{U(1-\gamma)}{\sqrt{K}}$,
\begin{align}
R(\lambda, K) \leq \frac{U \sqrt{K}}{1-\gamma}. \label{lem_PD_R}
\end{align}
Next, in order to bound the reward optimality gap, setting $\lambda=0$ in \cref{lem_PD_3} and using the above bound on the dual regret, we obtain
\begin{align}
\vsr(\rho) - \bmvr(\rho)  \leq \frac{U}{(1-\gamma) \sqrt{K}} + 2f(\mathcal{B}). \label{eqn_rgua}
\end{align}
In order to bound the constraint violation, we consider two cases. The first case is when $b^{\prime} - \bmvc(\rho) \leq 0$. Consequently, $b - 2f(\mathcal{B}) - \bmvc(\rho)   \leq 0$ and hence, $\bmvc(\rho) \geq b - 2f(\mathcal{B}) \geq b - 5f(\mathcal{B})$, which completes the proof. 

The second case is when $b^{\prime} - \bmvc(\rho) >0$. In this case, using the notation $[x]_{+}=\max \{x, 0\}$ and~\cref{lem_PD_3} with $\lambda = U$, we have
$$
\vsr(\rho) - \bmvr(\rho) + U\left[  b^{\prime} - \bmvc(\rho)\right]_{+} \leq \frac{R(U, K)}{K} + 2f(\mathcal{B}).
$$
Since $U$ has been set such that $U>\lambda^{*}$, we can use Lemma \ref{lemma_10} and obtain that, 
$$
\left[  b^{\prime} - \bmvc(\rho)\right]_{+} \leq \frac{R(U, K)}{K\left(U-\lambda^{*}\right)} + \frac{2f(\mathcal{B})}{U-\lambda^{*}}
$$
Combining the above inequality with \cref{lem_PD_R} gives
\begin{align}
\label{ineq_leminduct}
 b^{\prime} - \bmvc(\rho)  \leq \left[  b^{\prime} - \bmvc(\rho)\right]_{+}   & \leq \frac{U}{\left(U-\lambda^{*}\right)(1-\gamma) \sqrt{K}} + \frac{2f(\mathcal{B})}{U-\lambda^{*}}.
\end{align}
By Lemma \ref{lem_bounddualv}, we know $\lambda^* \leq \frac{1}{\zeta(1-\gamma)}$. By letting $U=\frac{2}{\zeta(1-\gamma)}$, we have $U-\lambda^* \geq  \frac{1}{\zeta(1-\gamma)} \geq 1$ as the Slater constant $\zeta \in (0,\frac{1}{1-\gamma}]$. Thus, $\frac{1}{U-\lambda^*} \leq 1$.
Now, setting $K$ to to be
$$
K= \frac{U^{2}}{[f(\mathcal{B})]^2(1-\gamma)^{2}}
$$
and substituting into \cref{eqn_rgua,ineq_leminduct}, we obtain
\begin{align}
\bmvr(\rho) \geq \vsr(\rho) - 3f(\mathcal{B}), \quad \text{and} 
\quad  \bmvc(\rho) \geq b^{\prime} - 3f(\mathcal{B}). \label{ineq_samec}
\end{align}
This establishes the first claim by substituting $b^{\prime} = b- 2 f(\mathcal{B})$. 

Next, we prove the second claim. We define $ \pi^{*+}\in \mathrm{argmax}_{\pi}  V^{\pi}_r(\rho) \text{ s.t. } V^{\pi}_c(\rho) \geq b + 6f(\mathcal{B})$.
From~\cref{eq:pd_inter-1}, recall that
\begin{align*}
\vkr(\rho) +  \lambda_{k} \vkc(\rho) - (\hat{V}_r^k(\rho) +  \lambda_k \hat{V}_c^k(\rho)) \leq 2f(\mathcal{B}) \, (1 + \lambda_k)
\end{align*}
As before, using the definition of $\pi^*_k$, we have
\begin{align*}
\vkr(\rho) +  \lambda_{k} \vkc(\rho) \geq \vssr(\rho) +  \lambda_{k} \vssc(\rho),
\end{align*}
Therefore, by combining the above inequalities, 
\begin{align}
\vssr(\rho) +  \lambda_{k} \vssc(\rho) &\leq \hat{V}_r^k(\rho) +  \lambda_k \hat{V}_c^k(\rho) + 2f(\mathcal{B}) \, (1 + \lambda_k) \label{eq:pd_inter-1} \\
\Longrightarrow \vssr(\rho) - \hat{V}_r^k(\rho) &\leq \lambda_k (\hat{V}_c^k(\rho) - \vssc(\rho) + 2f(\mathcal{B})) + 2f(\mathcal{B}). \nonumber
\end{align}
Since $\vssc(\rho) \geq b+6f(\mathcal{B})$, we obtain, 
\begin{align*}
\vssr(\rho) - \hat{V}_r^k(\rho) &\leq \lambda_k [\hat{V}_c^k(\rho) - (b + 3f(\mathcal{B}))] + 2f(\mathcal{B}).
\end{align*}
As before, by taking the average, letting $b^{\prime} = b + 4f(\mathcal{B}) $, and adding both sides by the same term $\frac{\lambda}{K} \sum_{k=0}^{K-1} \left[  b^{\prime} - \hat{V}_c^k(\rho) \right]$, we obtain that for $\lambda \in [0,U]$, 
\begin{align*}
\vssr(\rho) - \bmvr(\rho) + \lambda(  b^{\prime} - \bmvc(\rho)) \leq  \frac{R(\lambda,K)}{K}  + 2f(\mathcal{B}).
\end{align*}
The remainder of the proof proceeds in the same manner as before.
Setting $K$ to to be
$$
K=\frac{U^{2}}{[f(\mathcal{B})]^2(1-\gamma)^{2}}
$$
the algorithm ensures that
\begin{align}
\bmvr(\rho) \geq \vssr(\rho) - 3f(\mathcal{B}), \quad \text{and} 
\quad  \bmvc(\rho) \geq b^{\prime}  - 3f(\mathcal{B}). \label{ineq_samec}
\end{align}
This establishes the second claim by substituting $b^{\prime} = b + 4 f(\mathcal{B})$. 
\end{proof}

\section{Proofs for Section \ref{sec_linearCMDP}}

The proofs in Section \ref{Appendix_Table-Tabular}, Section \ref{sec_LP1} and Section \ref{sec_LP2} are adapted from \citet{kitamura2023regularization, kozuno2022kl} with modifications to fit our setting. Specifically, the analysis in \cite{kitamura2023regularization} applies to the \textit{non-stationary policies} returned by \texttt{MDVI} \textit{with} entropy regularization. In contrast, our analysis applies to the \textit{stationary policy} returned by MDVI \textit{without} entropy regularization. Furthermore, we also require additional analysis of the value functions returned by the \texttt{LS-PE} algorithm.

Throughout, we treat $\pi$ as an operator that returns an $|\mathcal{S}|$-dimensional vector s.t. for an arbitrary $|\mathcal{S}||\mathcal{A}|$-dimensional vector $u$ such that $(\pi u)(s):=\sum_{a\in \mathcal{A}}\pi(a|s) \, u(s,a)$. Furthermore, we define $P_{\pi}:= \pi P$ where $P_{\pi} \in \mathbb{R}^{|\mathcal{S}|\times |\mathcal{S}|}$ and denotes the transition probability matrix induced by policy $\pi$.

\subsection{Deriving \texttt{LS-MDVI} from Entropic Mirror Descent}
\label{Appendix_Table-Tabular}
We show that the \texttt{LS-MDVI} update can be derived as a limiting case of entropic mirror descent. At iteration $t$, given $Q_t$, if $\kappa$ is the entropy regularization parameter and $\tau$ is the KL regularization parameter, then, the entropic mirror descent policy update~\citet{kitamura2023regularization} is:
\begin{align*}
\pi_t(\cdot|s) & = \arg\max_{p \in \Delta(\mathcal{A})} \sum_{a \in \mathcal{A}} p(a) \left( Q^t(s,a) - \tau \log \frac{p(a)}{\pi_{t-1}(a|s)} - \kappa \log p(a) \right), \quad \text{for all } s \in \mathcal{S}, \\
\intertext{The above policy update can be rewritten in a closed-form solution as follows~\citep[Equation 5]{kozuno2019theoretical}),}
\pi_t(a|s) & = \frac{[\pi_{t-1}(a|s)]^\alpha \exp\left(\beta Q^t(s,a)\right)}{\sum_{b \in \mathcal{A}} [\pi_{t-1}(b|s)]^\alpha \exp\left(\beta Q^t(s,b)\right)}, \text{ where $\alpha := \tau / (\tau + \kappa)$, $\beta := 1 / (\tau + \kappa)$} \\ 
\implies \pi_t(a|s) & = \frac{\exp(\beta \sum_{i=0}^t \alpha^{t-i} \, Q^i(s,a))}{\sum_{b \in \mathcal{A}} \exp(\beta \sum_{i=0}^t \alpha^{t-i} \, Q^i(s,b))}.
\intertext{Since $\texttt{LS-MDVI}$ does not use entropy regularization $\kappa = 0$ implying $\alpha = 1$, the resulting update is:}
\pi_t(a|s) & = \frac{\exp(\beta \sum_{i=0}^t Q^i(s,a))}{\sum_{b \in \mathcal{A}} \exp(\beta \sum_{i=0}^t Q^i(s,b)} = \frac{1}{1 + \sum_{b \neq a} \exp(\beta (\bar{Q}^t(s,b) - \bar{Q}^t(s,a))} \tag{where $\bar{Q}^t := \sum_{i=0}^t Q^i$}
\end{align*}
For \texttt{LS-MDVI}, we take the limit $\tau \to 0$, $\beta \to \infty$ and consider two cases. 

\textbf{Case 1}: If $a = \argmax_{b} \bar{Q}^t(s,b)$, then, $\beta (\bar{Q}^t(s,b) - \bar{Q}^t(s,a)) < 0$ for all $b \neq a$. Hence, as $\beta \to \infty$, $\sum_{b \neq a} \exp(\beta (\bar{Q}^t(s,b) - \bar{Q}^t(s,a)) \to 0$ and $\pi_t(a|s) \to 1$. 

\textbf{Case 2}: If $a \neq \argmax_{b} \bar{Q}^t(s,b)$, then, $\beta (\bar{Q}^t(s,b) - \bar{Q}^t(s,a))  > 0$ for the action $b$ corresponding to the $\argmax$ action. Hence, as $\beta \to \infty$, $\sum_{b \neq a} \exp(\beta (\bar{Q}^t(s,b) - \bar{Q}^t(s,a)) \to \infty$ and $\pi_t(a|s) \to 0$. 

Hence, as $\kappa = 0$ and $\tau \to 0$, $\pi_t$ is a greedy policy and for all $s \in \cS$, $\pi_t(a|s) = 1$ for $a = \argmax_{b}  \sum_{i=0}^t Q^i(s,b)$, which recovers the policy update for \texttt{LS-MDVI}. 

For entropic mirror descent, the value update is given as~\citep{kitamura2023regularization}, for all $s \in \cS$, 
\begin{align*}
V^t(s) & =  \sum_{a} \pi_t(a|s) \left(Q^t(s,a) - \tau \log\left(\frac{\pi_t(a|s)}{\pi_{t-1}(a|s)}\right) - \kappa \, \ln(\pi_t(a|s)) \right) \\
& = (\pi_{t} Q^{t})(s) - \tau \mathrm{KL}(\pi_{t}(\cdot|s)\| \pi_{t-1}(\cdot|s)) + \kappa \, \mathcal{H} (\pi_{t}(\cdot|s)). \\
\intertext{Plugging the entropic mirror descent policy update and simplifying similar to~\citep[App. B]{kozuno2022kl}, we get,}
V^t(s) &= \frac{1}{\beta} \log \sum_{a \in \mathcal{A}} \exp\left( \beta Q^t(s,a) + \alpha \log \pi_{t-1}(a|s) \right) \\
&= \frac{1}{\beta} \log \sum_{a \in \mathcal{A}} \exp\left( \beta \sum_{i=0}^t \alpha^{t-i} Q^i(s,a) \right) - \frac{\alpha}{\beta} \log \sum_{a \in \mathcal{A}} \exp\left( \beta \sum_{i=0}^{t-1} \alpha^{t-i} Q^i (s,a) \right).
\intertext{Since $\texttt{LS-MDVI}$ does not use entropy regularization i.e. $\kappa = 0$ implying $\alpha = 1$, the update is:}
V^t(s) &= \frac{1}{\beta} \log \sum_{a \in \mathcal{A}} \exp\left( \beta \sum_{i=0}^t Q^i(s,a) \right) - \frac{1}{\beta} \log \sum_{a \in \mathcal{A}} \exp\left( \beta \sum_{i=0}^{t-1}  Q^i (s,a) \right). 
\end{align*}
For \texttt{LS-MDVI}, we take the limit $\tau \to 0$, $\beta \to \infty$. Using L\'Hopital's rule for the two terms, we get that, 
\begin{align*}
V^t(s) &= \sum_{a} \pi_t(a|s) \,  \sum_{i=0}^t Q^i(s,a) - \sum_{a} \pi_{t-1}(a|s) \, \sum_{i=0}^{t-1} Q^i(s,a) \\
& = \left(\pi_t \sum_{i=0}^t Q^i\right)(s) - \left(\pi_{t-1} \sum_{i=0}^{t-1} Q^i\right)(s)\,,
\end{align*}
which recovers the value update for \texttt{LS-MDVI}.

\subsection{Proof of Lemma \ref{lem_optGL} (Optimality Guarantees for Algorithm \ref{alg_MDVICLMDP} - Linear CMDP)}
\label{sec_LP1}
Note that for each $\lambda_k$ where $k \in [K]$, we run Algorithm \ref{alg_MDVICLMDP} with $\Box = r + \lambda_k \, c$. We define 
\begin{align}
\pi^*_k & := \argmax_{\pi} V^{\pi}_{r+\lambda_kc} \label{pi_sk} \\
\bvkb & := \frac{1}{T} \sum_{i=1}^T \hat{V}_{\Box}^i \overset{\text{(by telescoping)}}{=} \frac{1}{T}  (\pi_T \tilde{Q}_{\Box}^T) \overset{\text{(by definition)}}{=} \frac{1}{T} \left( \pi_T \left\langle \phi, \sum_{i=0}^T \theta^i_{\Box} \right\rangle \right). \label{ave_empv} 
\end{align}

Throughout the proof, for any $|\mathcal{S}||\mathcal{A}|$-dimensional vector $z$, we let $W(z)$ denote the solution to a weighted linear regression problem over the core set,
\begin{align}
W(z) := \argmin_{\theta} \sum_{(x,b) \in \mathcal{C}} \tilde{\rho}(x,b) \left( z(x,b) - \langle \phi(x,b), \theta \rangle\right)^2.
\label{eq:W-def-1}
\end{align}
The above problem can be solved as
\begin{align}
W(z) = G^{-1} \sum_{(x,b) \in \mathcal{C}} {\tilde{\rho}(x,b) \phi(x,b) z(x,b)}
\label{eq:W-def-2}
\end{align}
where $G := \sum_{(x,b) \in \mathcal{C}} \tilde{\rho}(x, b) \, \phi(x, b) \phi(x, b)^\top$.

Using this definition and the definition of ${\theta}_{\Box}^{i}$ in~\cref{alg_MDVICLMDP}, we have ${\theta}_{\Box}^{i}=W(\hqkib)$.

The linear MDP assumption ensures that there exists a vector $\boldsymbol{\theta}_{\Box}^{t}$ such that $\langle \phi, \boldsymbol{\theta}_{\Box}^{t} \rangle := \Box + \gamma P \hat{V}_{\Box}^{t-1}$. Therefore, using the definition of $W$, we have ${\boldsymbol{\theta}}_{\Box}^{i}=W(\langle \phi, {\boldsymbol{\theta}}_{\Box}^{i} \rangle)$.

We now present the proof of Lemma \ref{lem_optGL}.

\OPGuaranteeaL*

\begin{proof}
Using the definition of $\pi^*_k$ and that $\Box = r + \lambda_k \, c$, we decompose the sub-optimality as:
\begin{align*}
V^{\pi^*_k}_{r+\lambda_kc}(\rho) -  V^{\pi_T}_{r+\lambda_kc}(\rho) &=[V^{\pi^*_k}_{r+\lambda_kc}(\rho)  - \bvkb(\rho)] +
[\bvkb(\rho) -  V^{\pi_T}_{r+\lambda_kc}(\rho)]\\
\intertext{Bounding the first term by Lemma \ref{lem_conLMDP} and the second by Lemma \ref{lem_conLMDP2},}
&\leq \tilde{O}\left( \frac{H^2(1+\lambda_k)}{T} +{ H^2(1+\lambda_k)}  \sqrt{\frac{d}{TM}} \right)
\end{align*}
with probability at least $1-2\delta$. 
Setting $M=\tilde{O}\left( \frac{dH^2}{\epsilon} \right)$, $T=O(\frac{H^2}{\epsilon})$, and appropriately rescaling the confidence parameter $\delta$ completes the proof.
\end{proof}

We now prove~\cref{lem_conLMDP,lem_conLMDP2}. 
\begin{lemma}
\label{lem_conLMDP}
Let $\pi^*_k$ and $\bvkb$ be defined as in \cref{pi_sk,ave_empv}. For any $k \in [K]$,  with $\Box = r+\lambda_k \, c$ and $M \geq \tilde{O}\left( {dH^2}\right)$, we have
\begin{align*}
V^{\pi^*_k}_{r+\lambda_kc}(\rho)  - \bvkb(\rho)  \leq  \tilde{O}\left( \frac{H^2(1+\lambda_k)}{T} +{ H^2(1+\lambda_k)}  \sqrt{\frac{d}{TM}} \right)
\end{align*}
with probability at least $1-\delta$.
\end{lemma}

\begin{proof} 
We first recall that $\vkb =V^{\pi^*_k}_{r+\lambda_kc}$ and $\bvkb = \bar{V}_{r+\lambda_kc}^T$ by the definition of $\Box$. By the value difference lemma, we have that, 
\begin{align}
\vkb - \bvkb &=  (I-\gamma P_{\pi^*_k})^{-1} ( (\pi^*_k \Box) + \gamma P_{\pi^*_k} \, \bvkb - \bvkb)  \label{eqn_pioutL}    
\end{align}

Next, from Line 6 in Algorithm \ref{alg_MDVICLMDP}, by the telescoping sum, and by the greediness of $\pi_{T}$, we have
\begin{align}
\bvkb &=  \frac{1}{T} \, (\pi_{T} \tilde{Q}_{\Box}^T)\label{eqn_VQL0} \\
&\geq  \frac{1}{T} \,  (\pi^*_k  \tilde{Q}_{\Box}^T). \label{eqn_VQL}
\end{align}
Now, we have
\begin{align*}
\vkb - \bvkb &=  (I-\gamma P_{\pi^*_k})^{-1} ( (\pi^*_k \Box) + \gamma P_{\pi^*_k} \, \bvkb - \bvkb) \tag{By \cref{eqn_pioutL}} \\
&\leq (I-\gamma P_{\pi^*_k})^{-1} ( (\pi^*_k \Box) + \gamma P_{\pi^*_k} \, \bvkb - \frac{1}{T} \,  (\pi^*_k  \tilde{Q}_{\Box}^T)) \tag{By \cref{eqn_VQL}}\\
&= (I-\gamma P_{\pi^*_k})^{-1} ( (\pi^*_k \Box) + \gamma P_{\pi^*_k} \, \frac{1}{T} \,  (\pi_T  \tilde{Q}_{\Box}^T) - \frac{1}{T} \,  (\pi^*_k  \tilde{Q}_{\Box}^T)) \tag{By~\cref{eqn_VQL0}} \\
&= (I-\gamma P_{\pi^*_k})^{-1} \, \left[ (\pi^*_k \Box) + \gamma P_{\pi^*_k} \, \frac{1}{T} \, (\pi_{T} \tilde{Q}_{\Box}^T) -  (\pi^*_k \Box) - \gamma P_{\pi^*_k} \, \frac{1}{T} (\pi_{T-1} \tilde{Q}^{T-1}_{\Box})   \right.\\
&\left. \indent - \left(\pi^*_k  \left\langle \phi, W \left(\frac{1}{T} \sum_{i=0}^{T}(\hqkib-\langle \phi, {\boldsymbol{\theta}}_{\Box}^{i}) \rangle)\right) \right\rangle \right) \right] \tag{Using Lemma \ref{lem_sumtheta} for $\frac{1}{T} \tilde{Q}_{\Box}^T$}\\
&= (I-\gamma P_{\pi^*_k})^{-1} \left[ \gamma P_{\pi^*_k} \, \frac{1}{T}  \, (\pi_{T} \tilde{Q}_{\Box}^T) -  \gamma P_{\pi^*_k } \, \frac{1}{T} \, (\pi_{T-1} \tilde{Q}^{T-1}_{\Box}) \right. \\
&\left. \indent - \left(\pi^*_k  \left\langle \phi, W \left(\frac{1}{T} \sum_{i=0}^{T}(\hqkib-\langle \phi, {\boldsymbol{\theta}}_{\Box}^{i}) \rangle)\right) \right\rangle \right) \right].
\end{align*}
By defining $\mathcal{H}_{\pi^*_k} := (I-\gamma P_{\pi^*_k})^{-1}$, taking the infinity norm and using the triangle inequality, we obtain
\begin{align}
\norminf{\vkb - \bvkb} &\leq \underbrace{\norminf{\gamma\mathcal{H}_{\pi^*_k} P_{\pi^*_k}\left(\frac{1}{T} \,  (\pi_{T} \tilde{Q}_{\Box}^T) - \frac{1}{T} \, (\pi_{T-1} \tilde{Q}^{T-1}_{\Box}) \right)}}_{\text{Term (i)}} \nonumber \\
& +  \underbrace{\norminf{ \mathcal{H}_{\pi^*_k} \left(\pi^*_k  \left\langle \phi, W \left(\frac{1}{T} \sum_{i=0}^{T}(\hqkib-\langle \phi, {\boldsymbol{\theta}}_{\Box}^{i}) \rangle)\right) \right\rangle \right)}}_{\text{Term (ii)}}. \label{eqn_boxdec1}
\end{align}
In order to bound Term (i), we use Holder's inequality i.e. for a matrix $A$ and vector $x$, \\ $\norminf{Ax} \leq \norm{A}_{1, \infty} \, \norminf{x}$, and that $\| \mathcal{H}_{\pi^*_k} \, P_{\pi^*_k} \|_{1,\infty} \leq H$ to obtain, 
\begin{align*}
\norminf{\gamma \, \mathcal{H}_{\pi^*_k} P_{\pi^*_k} \, \left(\frac{1}{T} \,  (\pi_{T} \tilde{Q}_{\Box}^T) - \frac{1}{T} \, (\pi_{T-1} \tilde{Q}^{T-1}_{\Box}) \right)} & \leq H \, \left\|\left(\frac{1}{T} \,  (\pi_{T} \tilde{Q}_{\Box}^T) - \frac{1}{T} \, (\pi_{T-1} \tilde{Q}^{T-1}_{\Box}) \right) \right\|_\infty \\
& \leq \frac{4H^2(1+\lambda_k)}{T}\tag{Using~\cref{lem_vmvL}}
\end{align*}
with probability at least $1-\delta$. For term (ii), 
\begin{align*}
& \norminf{\mathcal{H}_{\pi^*_k} \left(\pi^*_k  \left\langle \phi, W \left(\frac{1}{T} \sum_{i=0}^{T}(\hqkib-\langle \phi,{\boldsymbol{\theta}}_{\Box}^{i}\rangle)\right) \right\rangle \right)} \\
& \leq \norm{\mathcal{H}_{\pi^*_k}}_{1,\infty} \, \norm{\left(\pi^*_k  \left\langle \phi, W \left(\frac{1}{T} \sum_{i=0}^{T}(\hqkib-\langle \phi,{\boldsymbol{\theta}}_{\Box}^{i}\rangle)\right) \right\rangle \right)}_\infty \, \tag{By Holder's inequality} \\
& \leq H \, \norm{\left(\pi^*_k  \left\langle \phi, W \left(\frac{1}{T} \sum_{i=0}^{T}(\hqkib-\langle \phi,{\boldsymbol{\theta}}_{\Box}^{i}\rangle)\right) \right\rangle \right)}_\infty \tag{Since $\| \mathcal{H}_{\pi^*_k}  \|_{1,\infty} \leq H$} \\
& \leq H \, \left\| \phi^{\top} W \left(\frac{1}{T} \sum_{i=1}^{T}(\hqkib- \langle \phi,{\boldsymbol{\theta}}_{\Box}^{i}\rangle) \right)   \right\|_{\infty} \tag{By definition of the $\pi$ operator} \\ 
& \leq \tilde{O}\left( { H^2(1+\lambda_k)}  \sqrt{\frac{d}{TM}} \right) \tag{By Lemma~\ref{lem_thetaconcentL}}
\end{align*}
Combining the above relations, 
\begin{align*}
\norminf{\vkb - \bvkb} & \leq \frac{4H^2(1+\lambda_k)}{T} + \tilde{O}\left( { H^2(1+\lambda_k)} \sqrt{\frac{d}{TM}} \right)  
\end{align*}
Using that for any $|\cS|$-dimensional vector $V$, $V(\rho) = \mathbb{E}_{s \sim \rho} V(s) \leq \| V \|_{\infty}$, we get that, 
\begin{align*}
V^{\pi^*_k}_{r+\lambda_kc}(\rho)  - \bvkb(\rho)   \leq \frac{4H^2(1+\lambda_k)}{T} + \tilde{O}\left( { H^2(1+\lambda_k)} \sqrt{\frac{d}{TM}} \right)
\end{align*}
with probability at least $1-\delta$. 
\end{proof}

\begin{lemma}
\label{lem_conLMDP2}
Let $\bvkb$ be defined as in \cref{ave_empv}. For any $k \in [K]$,  with $\Box = r + \lambda_k \, c$ and $M \geq \tilde{O}\left( {dH^2}\right)$, we have
\begin{align*}
\bvkb(\rho) -  V^{\pi_T}_{r+\lambda_kc}(\rho) \leq  \tilde{O}\left( \frac{H^2(1+\lambda_k)}{T} +{ H^2(1+\lambda_k)}  \sqrt{\frac{d}{TM}} \right)
\end{align*}
with probability at least $1-\delta$.
\end{lemma}

\begin{proof}
The proof is similar as for the above lemma. By the value difference lemma, we have that, 
\begin{align}
\bvkb -  V^{\pi_T}_{r+\lambda_kc} &=  (I-\gamma P_{\pi_T})^{-1} ( \bvkb - (\pi_T \Box) - \gamma P_{\pi_T} \, \bvkb  )  \label{eqn_pioutL2}    
\end{align}
Now, we have
\begin{align*}
\bvkb -  V^{\pi_T}_{r+\lambda_kc} &=  (I-\gamma P_{\pi_T})^{-1} \left( \frac{1}{T} \,  (\pi_T  \tilde{Q}_{\Box}^T) - (\pi_T \Box) - \gamma P_{\pi_T} \, \bvkb \right)  \\
&=  (I-\gamma P_{\pi_T})^{-1} \left( \frac{1}{T} \,  (\pi_T  \tilde{Q}_{\Box}^T) - (\pi_T \Box) -  \gamma P_{\pi_T} \, \frac{1}{T} \, (\pi_{T} \tilde{Q}_{\Box}^T)  \right)  \\
&= (I-\gamma P_{\pi_T})^{-1} \, \left[(\pi_T \Box) + \gamma P_{\pi_T} \, \frac{1}{T} (\pi_{T-1} \tilde{Q}^{T-1}_{\Box}) + \left(\pi_T  \left\langle \phi, W \left(\frac{1}{T} \sum_{i=0}^{T}(\hqkib-\langle \phi,{\boldsymbol{\theta}}_{\Box}^{i}\rangle)\right) \right\rangle \right) \right.\\
&\left. \indent - (\pi_T \Box) - \gamma P_{\pi_T} \, \frac{1}{T} \, (\pi_{T} \tilde{Q}_{\Box}^T)  \right] \tag{Using Lemma \ref{lem_sumtheta} for $\frac{1}{T} \tilde{Q}^{T}$}\\
&= (I-\gamma P_{\pi_T})^{-1} \, \left[ \gamma P_{\pi_T} \, \frac{1}{T} (\pi_{T-1} \tilde{Q}^{T-1}_{\Box}) -  \gamma P_{\pi_T} \, \frac{1}{T} \, (\pi_{T} \tilde{Q}_{\Box}^T)  \right.\\
&\left. \indent + \left(\pi_T  \left\langle \phi, W \left(\frac{1}{T} \sum_{i=0}^{T}(\hqkib-\langle \phi,{\boldsymbol{\theta}}_{\Box}^{i}\rangle)\right) \right\rangle \right)  \right].\\
\end{align*}
By defining $\mathcal{H}_{\pi_T} := (I-\gamma P_{\pi_T})^{-1}$, taking the infinity norm and using the triangle inequality, we obtain
\begin{align}
\norminf{\bvkb -  V^{\pi_T}_{r+\lambda_k \, c}} &\leq \underbrace{\norminf{\gamma\mathcal{H}_{\pi_T} P_{\pi_T}\left(  \frac{1}{T} \, (\pi_{T-1} \tilde{Q}^{T-1}_{\Box}) - \frac{1}{T} \,  (\pi_{T} \tilde{Q}_{\Box}^T) \right) }}_{\text{Term (i)}} \nonumber \\
& +  \underbrace{\norminf{\mathcal{H}_{\pi_T} \left(\pi_T  \left\langle \phi, W \left(\frac{1}{T} \sum_{i=0}^{T}(\hqkib-\langle \phi,{\boldsymbol{\theta}}_{\Box}^{i}\rangle)\right) \right\rangle \right)}}_{\text{Term (ii)}}. \label{eqn_boxdec1}
\end{align}
In order to bound Term (i), we use Holder's inequality and that $\norm{\mathcal{H}_{\pi_T} \, P_{\pi_T}}_{1, \infty} \leq H$, 
\begin{align*}
\norminf{\gamma \, \mathcal{H}_{\pi_T} P_{\pi_T} \left( \frac{1}{T} \, (\pi_{T-1} \tilde{Q}^{T-1}_{\Box}) - \frac{1}{T} \,  (\pi_{T} \tilde{Q}_{\Box}^T) \right)} & \leq H \, \left\| \frac{1}{T} \, (\pi_{T-1} \tilde{Q}^{T-1}_{\Box})  - \frac{1}{T} \,  (\pi_{T} \tilde{Q}_{\Box}^T)   \right\|_\infty \\
& \leq \frac{4H^2(1+\lambda_k)}{T} \tag{Using~\cref{lem_vmvL}}
\end{align*}
with probability at least $1-\delta$. For term (ii), 
\begin{align*}
& \norminf{\mathcal{H}_{\pi_T} \left(\pi_T  \left\langle \phi, W \left(\frac{1}{T} \sum_{i=0}^{T}(\hqkib-\langle \phi,{\boldsymbol{\theta}}_{\Box}^{i}\rangle)\right) \right\rangle \right) }\\
& \leq \norm{\mathcal{H}_{\pi_T}}_{1,\infty} \, \norm{\left(\pi_T  \left\langle \phi, W \left(\frac{1}{T} \sum_{i=0}^{T}(\hqkib-\langle \phi,{\boldsymbol{\theta}}_{\Box}^{i}\rangle)\right) \right\rangle \right)}_\infty \tag{By Holder's inequality} \\
& \leq H \, \norm{\left(\pi_T  \left\langle \phi, W \left(\frac{1}{T} \sum_{i=0}^{T}(\hqkib-\langle \phi,{\boldsymbol{\theta}}_{\Box}^{i}\rangle)\right) \right\rangle \right)}_\infty \tag{Since $\|\mathcal{H}_{\pi_T}\|_{1, \infty} \leq H$} \\
& \leq H \, \left\| \left\langle \phi, W \left(\frac{1}{T} \sum_{i=1}^{T}(\hqkib-\langle \phi,{\boldsymbol{\theta}}_{\Box}^{i}\rangle)\right)   \right\rangle \right\|_{\infty} \tag{By definition of the $\pi$ operator} \\    
& \leq \tilde{O}\left( { H^2(1+\lambda_k)}  \sqrt{\frac{d}{TM}} \right) \tag{By Lemma~\ref{lem_thetaconcentL}}
\end{align*}
Combining the above relations and using that for any $|\cS|$-dimensional vector $V$, $V(\rho) = \mathbb{E}_{s \sim \rho} V(s) \leq \| V \|_{\infty}$, we get that, 
\begin{align*}
\bvkb(\rho) -  V^{\pi_T}_{r+\lambda_kc}(\rho) \leq \frac{4H^2(1+\lambda_k)}{T} + \tilde{O}\left( { H^2(1+\lambda_k)} \sqrt{\frac{d}{TM}} \right)
\end{align*}
with probability at least $1-\delta$. 
\end{proof}

\subsubsection{Auxiliary Lemmas}

\begin{lemma}
\label{lem_rangeVL}
For any $k \in [K]$ and $t \in [T]$, with $\Box = r + \lambda_k \, c$ and $M \geq \tilde{O}\left( {dH^2} \right)$, we have
\begin{align*}
\| \hvkttb{t} \|_{\infty} \leq 2H(1+\lambda_k)
\end{align*}
with probability at least $1-\delta$.
\end{lemma}

\begin{proof}
First, we note that from the last line in Algorithm \ref{alg_MDVICLMDP}, 
\begin{align}
\hvkttb{t} &= \ (\pi_{t}\tilde{Q}_{\Box}^{t}) -  (\pi_{t-1} \tilde{Q}_{\Box}^{t-1}) \nonumber \\
&= \left( \pi_{t} \left\langle \phi, \sum_{i=0}^t \theta_{\Box}^{i} \right\rangle \right) -  \left( \pi_{t-1} \left\langle \phi, \sum_{i=0}^{t-1} \theta_{\Box}^{i} \right\rangle \right) \nonumber \\
&\leq \left( \pi_{t} \left\langle \phi, \sum_{i=0}^t \theta_{\Box}^{i} \right\rangle
 \right)- \left( \pi_{t} \left\langle \phi, \sum_{i=0}^{t-1} \theta_{\Box}^{i} \right\rangle \right) \tag{By the greediness of $\pi_{t-1}$} \nonumber \\
&= (\pi_{t} \, \langle \phi, \theta_{\Box}^{t} \rangle). \label{eqn_lem_rangeVL1}
\end{align}
Next, we bound the term $\langle \phi, \theta_{\Box}^{t}\rangle$. We have
\begin{align*}
\left|\langle \phi, \theta_{\Box}^{t} \rangle \right| &= \left| \langle \phi,  W(\hqktb{t}) \rangle \right| \tag{By the definition of $W$ in~\cref{eq:W-def-2}}\\
&\leq \left|\langle \phi, W( \langle \phi, \boldsymbol{\theta}_{\Box}^{t}\rangle) \rangle \right| + \left| \langle \phi, W(\hqktb{t}) - W(\langle \phi, \boldsymbol{\theta}_{\Box}^{t}\rangle) \rangle\right| \tag{By triangle inequality} \\
&= \left| \langle \phi, \boldsymbol{\theta}_{\Box}^{t} \rangle\right| + \left| \langle \phi, W(\hqktb{t} - \langle \phi, \boldsymbol{\theta}_{\Box}^{t}\rangle ) \rangle \right| \tag{Since $W(z)$ is linear in $z$} \\
&= \left| \Box + \gamma P \hvkttb{t-1} \right| + \left| \langle \phi, W(\hqktb{t} - \langle \phi, \boldsymbol{\theta}_{\Box}^{t}\rangle) \rangle \right| \tag{By the definition of $\boldsymbol{\theta}_{\Box}^{t}$} \\
&\leq ( 1+\lambda_k + \gamma \| \hvkttb{t-1} \|_{\infty} )\mathbf{1} + \left| \langle \phi, W(\hqktb{t} - \langle \phi, \boldsymbol{\theta}_{\Box}^{t}\rangle) \rangle \right| \tag{Since $\Box(s,a) \leq 1+\lambda_k$}  \\
&\leq ( 1+\lambda_k + \gamma \| \hvkttb{t-1} \|_{\infty} )\mathbf{1}  + {\sqrt{2d}} \max_{(s,a)\in \mathcal{C}} \left|  \hqktb{t}(s,a) - (\langle \phi, \boldsymbol{\theta}_{\Box}^{t} \rangle)(s,a)  \right| \mathbf{1} \tag{By Lemma \ref{lem_43} with $z = \hqktb{t} - \langle \phi, \boldsymbol{\theta}_{\Box}^{t}\rangle$} \\
& = ( 1+\lambda_k + \gamma \| \hvkttb{t-1} \|_{\infty} )\mathbf{1} \\ & + {\sqrt{2d}} \max_{(s,a)\in \mathcal{C}} \left| \Box (s,a) + \gamma (\hat{P}_{t-1} \hvkttb{t-1})(s,a) - \Box (s,a) -   \gamma (P \hvkttb{t-1})(s,a)   \right| \mathbf{1} \tag{By definition of $\hat{Q}^t_{\Box}$ and $\boldsymbol{\theta}_{\Box}^{t}$} \\
&= ( 1+\lambda_k + \gamma \| \hvkttb{t-1} \|_{\infty} )\mathbf{1} + {\sqrt{2d}} \max_{(s,a)\in \mathcal{C}} \left|  \gamma (\hat{P}_{t-1} \hvkttb{t-1})(s,a) -   \gamma (P \hvkttb{t-1}) (s,a)  \right| \mathbf{1} \\
\implies \left|\langle \phi, \theta_{\Box}^{t} \rangle \right| &\leq ( 1+\lambda_k + \gamma \| \hvkttb{t-1} \|_{\infty} )\mathbf{1} + {\sqrt{2d}} \frac{\| \hvkttb{t-1} \|_{\infty}}{\| \hvkttb{t-1} \|_{\infty}} \max_{(s,a)\in \mathcal{C}} \left|  \gamma (\hat{P}_{t-1} \hvkttb{t-1})(s,a) -   \gamma (P \hvkttb{t-1})(s,a)   \right| \mathbf{1}
\end{align*}
Next, we bound the term 
\[
\frac{1}{\| \hvkttb{t-1} \|_{\infty}} \max_{(s,a)\in \mathcal{C}} \left|  \gamma (\hat{P}_{t-1} \hvkttb{t-1})(s,a) -   \gamma (P \hvkttb{t-1})(s,a)   \right|.
\]
We first note that this term is upper bounded by $2$. Now, using the Azuma-Hoeffding inequality (Lemma \ref{lem_Azuma_Hoeffding}) and taking a union bound over $(s,a) \in \mathcal{C}$ and $t\in[T]$, we have
\begin{align*}
\mathbb{P}\left(\exists (s,a,t) \in \mathcal{C} \times [T] 
\ \text{ s.t. } \frac{1}{\| \hvkttb{t-1} \|_{\infty}} \max_{(s,a)\in \mathcal{C}} \left|  \gamma (\hat{P}_{t-1} \hvkttb{t-1})(s,a) -   \gamma (P \hvkttb{t-1})(s,a)   \right| \geq \tilde{O}\left( \gamma \sqrt{\frac{1}{M}} \right) \right) \leq \delta.
\end{align*}
Therefore, with probability at least $1-\delta$, we have
\begin{align}
\left| \langle \phi, \theta_{\Box}^{t} \rangle \right| \leq ( 1+\lambda_k + \gamma \| \hvkttb{t-1} \|_{\infty} )\mathbf{1} + {\| \hvkttb{t-1} \|_{\infty}}  \tilde{O} \left( \gamma\sqrt{\frac{d}{M}} \right) \mathbf{1}. \label{eqn_lem_rangeVL2}
\end{align}
Given the above inequality, we can prove the claim by induction on $t$. Since $\hat{V}^{0}_\Box = 0$, the base case is satisfied. We assume that $\| \hvkttb{t-1} \|_{\infty}  \leq 2H(1+\lambda_k)$. By combining \cref{eqn_lem_rangeVL1} and \cref{eqn_lem_rangeVL2}, we have
\begin{align*}
\| \hvkttb{t} \|_{\infty} &\leq\|( \pi_{t}  \langle \phi, \theta_{\Box}^{t} \rangle) \|_{\infty}\\
&\leq\| \langle \phi, \theta_{\Box}^{t} \rangle \|_{\infty} \tag{By definition of the $\pi$ operator}\\ 
&\leq \left( 1 + \lambda_k + \gamma \| \hvkttb{t-1} \|_{\infty} + { \| \hvkttb{t-1} \|_{\infty}}  \tilde{O} \left( \gamma\sqrt{\frac{d}{M}} \right) \right) \\
&\leq  \left(1 + 2H \, \gamma  + {2 H} \, \tilde{O} \left( \gamma\sqrt{\frac{d}{M}} \right)\right)(1+\lambda_k). \tag{Induction hypothesis}
\end{align*}
By taking $M \geq \tilde{O}\left( {dH^2} \right)$, we have
\begin{align*}
\| \hvkttb{t} \|_{\infty} &\leq  (1 + 2H \, \gamma + 1)(1+\lambda_k) \\
&= (2 + 2H \, \gamma )(1+\lambda_k)\\
&\leq 2H(1+\lambda_k) \tag{Since $H=1/(1-\gamma)$}
\end{align*}
which completes the proof.
\end{proof}

The following corollary is a direct consequence of~\cref{eqn_lem_rangeVL2} and the above lemma.

\begin{corollary}
\label{cor_rangetheta}
For any $k \in [K]$ and $t \in [T]$, with $\Box = r + \lambda_k \, c$ and $M \geq \tilde{O}\left( {dH^2} \right)$, we have
\begin{align*}
\left\| \langle \phi, \theta_{\Box}^{t}  \rangle \right\|_{\infty} \leq 2H(1+\lambda_k)
\end{align*}
with probability at least $1-\delta$ respectively.
\end{corollary}

\begin{lemma}
\label{lem_vmvL}
For any $k \in [K]$, with $M \geq \tilde{O}\left( {dH^2} \right)$, we have
\begin{align*}
\left\| \frac{1}{T} (\pi_{T}  \tilde{Q}_{\Box}^T) - \frac{1}{T} ( \pi_{T-1}  \tilde{Q}^{T-1}_{\Box}) \right\|_{\infty} &\leq \frac{2H(1+\lambda_k)}{T}
\end{align*}
with probability at least $1-\delta$.
\end{lemma}

\begin{proof}
By the definition of $\tilde{Q}_{\Box}^t$ and due to the greediness of $\pi_{T-1}$, we have
\begin{align*}
\frac{1}{T} (\pi_{T}  \tilde{Q}_{\Box}^T )- \frac{1}{T}( \pi_{T-1}  \tilde{Q}^{T-1}_{\Box}) &\leq  \frac{1}{T}(\pi_{T}  \tilde{Q}_{\Box}^T) - \frac{1}{T} (\pi_{T} \tilde{Q}^{T-1}_{\Box}) \\
&=  \left( \pi_{T} \left\langle \phi,  \frac{1}{T} \sum_{i=0}^{T} \theta_{\Box}^{i} -   \frac{1}{T} \sum_{i=0}^{T-1} \theta_{\Box}^{i} \right\rangle \right) \\
&= \frac{1}{T} ( \pi_{T} \langle \phi, \theta_{\Box}^{T}\rangle ) \\
&\leq \frac{1}{T} \left\| \langle \phi, \theta_{\Box}^{T}\rangle \right\|_{\infty} \mathbf{1} \tag{By definition of the $\pi$ operator} \\
&\leq \frac{2H(1+\lambda_k)}{T} \mathbf{1}  \tag{By Corollary \ref{cor_rangetheta}}
\end{align*}
with probability at least $1-\delta$. Similarly, by the greediness of $\pi_{T}$, we have
\begin{align*}
\frac{1}{T} (\pi_{T-1}  \tilde{Q}^{T-1}_{\Box}) - \frac{1}{T} (\pi_{T}  \tilde{Q}_{\Box}^T)  &\leq  \frac{1}{T} ( \pi_{T-1} \tilde{Q}^{T-1}_{\Box}) - \frac{1}{T}(\pi_{T-1}  \tilde{Q}_{\Box}^T) \\
&= \left( \pi_{T-1} \left\langle \phi,  \frac{1}{T} \sum_{i=0}^{T-1} \theta_{\Box}^{i} - \frac{1}{T} \sum_{i=0}^{T} \theta_{\Box}^{i}  \right\rangle \right) \\
&= -\frac{1}{T} ( \pi_{T} \langle \phi, \theta_{\Box}^{T}\rangle ) \leq \frac{1}{T} \norminf{( \pi_{T} \langle \phi, \theta_{\Box}^{T}\rangle )} \\
&\leq \frac{1}{T} \left\| \langle \phi, \theta_{\Box}^{T}\rangle \right\|_{\infty} \mathbf{1} \tag{By definition of the $\pi$ operator} \\
&\leq \frac{2H(1+\lambda_k)}{T} \mathbf{1}  \tag{By Corollary \ref{cor_rangetheta}}
\end{align*}
with probability at least $1-\delta$.
\end{proof}

\begin{lemma}
\label{lem_thetaconcentL}
For any $k\in[K]$ and $t \in [T]$, with $\Box = r + \lambda_k \, c$ and $M \geq \tilde{O}\left( {dH^2} \right)$, we have
\begin{align*}
\left\| \left\langle \phi, W \left(\frac{1}{t} \sum_{i=1}^{t}(\hqkib- \langle \phi, {\boldsymbol{\theta}}_{\Box}^{i} \rangle)\right) \right\rangle  \right\|_{\infty} \leq \tilde{O}\left( { H(1+\lambda_k)}  \sqrt{\frac{d}{tM}} \right)
\end{align*}
with probability at least $1-\delta$.
\end{lemma}

\begin{proof}
\begin{align*}
& \left\| \left\langle \phi, W \left(\frac{1}{t} \sum_{i=1}^{t}(\hqkib- \langle \phi, {\boldsymbol{\theta}}_{\Box}^{i} \rangle)\right) \right\rangle    \right\|_{\infty} \\
& \leq {\sqrt{2d}} \max_{(s,a)\in \mathcal{C}} \, \left| \frac{1}{t} \sum_{i=0}^{t-1} \left[ \hqktb{i}(s,a) - (\langle \phi, \boldsymbol{\theta}_{\Box}^{i} \rangle)(s,a) \right] \right| \tag{By Lemma \ref{lem_43} with $z = \frac{1}{t} \sum_{i=0}^{t-1} [\hqktb{i} - \langle \phi, \boldsymbol{\theta}_{\Box}^{i}\rangle]$} \\
&= {\sqrt{2d}} \max_{(s,a)\in \mathcal{C}} \left| \frac{1}{t} \sum_{i=0}^{t-1} \left[  \gamma (\hat{P}_i \hvkttb{i})(s,a)  -   \gamma (P \hvkttb{i})(s,a) \right]  \right| \tag{By definition of $\hat{Q}^i_{\Box}$ and $\boldsymbol{\theta}_{\Box}^{i}$}
\end{align*}
By Lemma \ref{lem_rangeVL}, we have that, with probability at least $1 - \delta$, the bound $\norminf{\hat{V}^i_{\Box}} \leq 2H(1 + \lambda_k)$ holds for all $i \in [T]$. Now, using Lemma \ref{lem_concentrate} and taking the union bound over $(s,a) \in \mathcal{C}$, we have
\begin{align*}
\mathbb{P}\left(\exists (s,a) \in \mathcal{C} 
\ \text{ s.t. } \frac{1}{t} \sum_{i=0}^{t-1} \left[ (\hat{P}_i \hvkttb{i})(s,a) -  (P \hvkttb{i})(s,a) \right] \geq \tilde{O}\left( H(1+\lambda_k) \sqrt{\frac{1}{tM}} \right) \right) \leq \delta.
\end{align*}
Therefore, by appropriately rescaling $\delta$, we have that with probability at least $1-\delta$, 
\begin{align*}
\left\| \left\langle \phi, W \left(\frac{1}{t} \sum_{i=1}^{t}(\hqkib- \langle \phi, {\boldsymbol{\theta}}_{\Box}^{i} \rangle)\right) \right\rangle \right\|_{\infty} \leq \tilde{O}\left( { H(1+\lambda_k)}  \sqrt{\frac{d}{tM}} \right).
\end{align*}
\end{proof}

\begin{lemma}
\label{lem_thetaconcentLfix}
For any $k\in[K]$ and $i \in [T]$, with $\Box = r + \lambda_k \, c$ and $M \geq \tilde{O}\left( {dH^2} \right)$, we have
\begin{align*}
\left\| \langle \phi, W ( \hqkib- \langle\phi, {\boldsymbol{\theta}}_{\Box}^{i} \rangle ) \rangle   \right\|_{\infty} \leq \tilde{O}\left( { H(1+\lambda_k)}  \sqrt{\frac{d}{M}} \right)
\end{align*}
with probability at least $1-\delta$.
\end{lemma}

\begin{proof}
By following a similar proof as that for the above lemma, 
\begin{align}
\left\| \langle \phi, W ( \hqkib- \langle\phi, {\boldsymbol{\theta}}_{\Box}^{i} \rangle ) \rangle    \right\|_{\infty}  &\leq {\sqrt{2d}} \max_{(s,a)\in \mathcal{C}} \left|   \gamma \hat{P}_i \hvkttb{i}(s,a)  -   \gamma P \hvkttb{i}(s,a)   \right|.
\label{eq:d6-inter}
\end{align}
By Lemma \ref{lem_rangeVL}, we have that, with probability at least $1 - \delta$, $(\hat{P}_i\hat{V}^i_{\Box})(s,a) \leq 2H(1 + \lambda_k)$ holds for all $i \in [T]$ and all $(s,a)$. We note that by the definition of $\hat{P}_i$, $(\hat{P}_i \hvkttb{i})(s,a)$ is the empirical average of $M$ value functions.   Now, using Lemma \ref{lem_Azuma_Hoeffding} with $N=M$ and taking the union bound over $(s,a) \in \mathcal{C}$ and $i \in [T]$, we have
\begin{align*}
\mathbb{P}\left(\exists (s,a,t) \in \mathcal{C} \times [T] 
\ \text{ s.t. } (\hat{P}_i \hvkttb{i})(s,a) -  (P \hvkttb{i})(s,a)  \geq \tilde{O}\left( H(1+\lambda_k) \sqrt{\frac{1}{M}} \right) \right) \leq \delta
\end{align*}
Combining the above inequality with~\cref{eq:d6-inter} and appropriately rescaling $\delta$ completes the proof.
\end{proof}

\begin{lemma}
\label{lem_avetheta}
For any $t\in[T]$, we have
\begin{align*}
\frac{1}{t} \left\langle \phi, \sum_{i=0}^{t} {\boldsymbol{\theta}}_{\Box}^{i} \right\rangle= \Box +  \gamma \, \frac{1}{t} \, P \, (\pi_{t-1} \tilde{Q}^{t-1}_{\Box}).
\end{align*}
\end{lemma}

\begin{proof}
We first recall that by definition, $\langle \phi, \boldsymbol{\theta}_{\Box}^{t} \rangle := \Box + \gamma P \hat{V}_{\Box}^{t-1}$, $\hat{V}^0_{\Box} =\mathbf{0}$, and $\theta^0_{\Box}=\mathbf{0}$. Now we have
\begin{align*}
\frac{1}{t}\left\langle \phi, \sum_{i=0}^{t} {\boldsymbol{\theta}}_{\Box}^{i} \right\rangle &:= \frac{1}{t}\sum_{i=0}^{t-1} (\Box +\gamma P \hvkib)\\
&= \Box + \gamma P \left(\frac{1}{t}\sum_{i=0}^{t-1} \hvkib \right) \\
&= \Box + \gamma P \left(\frac{1}{t}\sum_{i=0}^{t-1} \left[ (\pi_i \tilde{Q}_{\Box}^i) - (\pi_{i-1} \tilde{Q}_{\Box}^{i-1}) \right] \right) \tag{From the last line in Algorithm \ref{alg_MDVICLMDP}}\\
&= \Box +  \gamma \frac{1}{t} \, P  \, (\pi_{t-1} \tilde{Q}^{t-1}_{\Box}).  \tag{Telescoping Sum}
\end{align*}
\end{proof}

\begin{lemma}
\label{lem_sumtheta}
We have
\begin{align*}
\frac{1}{T} \tilde{Q}_{\Box}^T= \left\langle \phi, W \left(\frac{1}{T} \sum_{i=0}^{T}(\hqkib- \langle \phi, {\boldsymbol{\theta}}_{\Box}^{i}) \rangle\right) \right\rangle +  \Box + \gamma P \frac{1}{T} ( \pi_{T-1} \tilde{Q}^{T-1}_{\Box} ).
\end{align*}
\end{lemma}

\begin{proof}
We first recall that by the definition of $W$, we have ${\theta}_{\Box}^{i}=W(\hqkib)$ and ${\boldsymbol{\theta}}_{\Box}^{i}=W(\langle \phi, {\boldsymbol{\theta}}_{\Box}^{i}) \rangle)$. Thus,
\begin{align*}
\frac{1}{T} \tilde{Q}_{\Box}^T &= \frac{1}{T} \tilde{Q}_{\Box}^T - \frac{1}{T} \sum_{i=0}^{T} \langle \phi, {\boldsymbol{\theta}}_{\Box}^{i}) \rangle + \frac{1}{T} \sum_{i=0}^{T} \langle \phi, {\boldsymbol{\theta}}_{\Box}^{i}) \rangle \\
&= \frac{1}{T} \sum_{i=0}^{T} \langle \phi, {\theta}_{\Box}^{i}\rangle - \frac{1}{T} \sum_{i=0}^{T} \langle \phi, {\boldsymbol{\theta}}_{\Box}^{i}) \rangle + \frac{1}{T} \sum_{i=0}^{T} \langle \phi, {\boldsymbol{\theta}}_{\Box}^{i}) \rangle \tag{By definition of $\tilde{Q}_{\Box}^T$} \\
&= \frac{1}{T} \sum_{i=0}^{T} \langle \phi, {\theta}_{\Box}^{i} - {\boldsymbol{\theta}}_{\Box}^{i} \rangle +\frac{1}{T} \sum_{i=0}^{T} \langle \phi, {\boldsymbol{\theta}}_{\Box}^{i}) \rangle\\
&= \frac{1}{T} \sum_{i=0}^{T} \left\langle \phi, W(\hqkib) - W(\langle \phi, {\boldsymbol{\theta}}_{\Box}^{i} \rangle) \right\rangle + \frac{1}{T} \sum_{i=0}^{T} \langle \phi, {\boldsymbol{\theta}}_{\Box}^{i}) \rangle \tag{Since ${\theta}_{\Box}^{i}=W(\hqkib)$ and ${\boldsymbol{\theta}}_{\Box}^{i}=W(\langle \phi, {\boldsymbol{\theta}}_{\Box}^{i}) \rangle)$}\\
&=  \left\langle \phi, W \left(\frac{1}{T} \sum_{i=0}^{T}(\hqkib- \langle \phi, {\boldsymbol{\theta}}_{\Box}^{i}) \rangle\right) \right\rangle + \frac{1}{T} \sum_{i=0}^{T} \langle \phi, {\boldsymbol{\theta}}_{\Box}^{i}) \rangle \tag{$W(z)$ is linear in $z$} \\
&= \left\langle \phi, W \left(\frac{1}{T} \sum_{i=0}^{T}(\hqkib- \langle \phi, {\boldsymbol{\theta}}_{\Box}^{i}) \rangle\right) \right\rangle +  \Box + \frac{1}{T}  \, \gamma P \,  (\pi_{T-1} \tilde{Q}^{T-1}_{\Box}) \tag{By Lemma \ref{lem_avetheta} with $t = T -1$}.
\end{align*}
\end{proof}

The following lemma bounds the extrapolation error due to the least-squares regression. It is the unweighted version (i.e., uniform weighting with $f = \mathbf{1}$) of Lemma~4.3 in \cite{kitamura2023regularization}.

\begin{lemma}[KW Bound]
\label{lem_43}
Let \( z \) be a function defined over \( \mathcal{C} \). Then, there exists \( \tilde{\rho} \in \Delta(\mathcal{S} \times \mathcal{A}) \) with a finite support \( \mathcal{C} := \text{Supp}(\tilde{\rho}) \) of size less than or equal to \( u_{\mathcal{C}} \) such that
\[
\max_{(s,a)\in \mathcal{S}\times \mathcal{A}} \left[ \langle \phi(s,a), W(z) \rangle \right] \leq \sqrt{2d} \max_{(x', b') \in \mathcal{C}} \left| {z(x', b')} \right|,
\]
where \( W(z) := G^{-1} \sum_{(x,b) \in \mathcal{C}} {\tilde{\rho}(x,b) \phi(x,b) z(x,b)}.
\)
\end{lemma}

\subsection{Proof of Lemma \ref{lem_optG2L} (Optimality Guarantees for Algorithm \ref{alg_CVIL} - Linear CMDP)}
\label{sec_LP2}
We define 
\begin{align}
\bmvkd & := \frac{1}{T} \sum_{i=1}^T \hat{\mathcal{V}}_{\diamond}^i \overset{\text{(a)}}{=}  \frac{1}{T} \left(\pi 
 \left\langle \phi, \sum_{i=1}^T \omega^i_{\diamond} \right\rangle \right) \label{ave_empvm} \\
\langle \phi,\boldsymbol{\omega}_{\diamond}^{t} \rangle & := \diamond + \gamma P \hat{\mathcal{V}}_{\diamond}^{t-1}
\end{align}
where (a) is from the last line in Algorithm \ref{alg_CVIL}.

\OPGuaranteebL*

\begin{proof} 
Using the value difference lemma, 
\begin{align*}
 \bmvkd - V^{\pi}_{\diamond} &=  (I-\gamma P_{\pi})^{-1} ( \bmvkd -  (\pi\diamond) - \gamma P_{\pi} \bmvkd)  .
\end{align*}
We now have
\begin{align*}
\bmvkd - V^{\pi}_{\diamond}  &= (I-\gamma P_{\pi})^{-1} \left[  \left( \pi \left\langle \phi, \frac{1}{T}\sum_{i=1}^T \omega^i_{\diamond} \right\rangle \right)   -  (\pi\diamond) - \gamma P_{\pi} \bmvkd \right] \tag{By definition of $\bmvkd$ in~\cref{ave_empvm}} \\
&= (I-\gamma P_{\pi})^{-1}  \left[ \left( \pi \left\langle \phi, W \left(\frac{1}{T} \sum_{i=1}^{T}(\hmqkid- \langle \phi,{\boldsymbol{\omega}}_{\diamond}^{i} \rangle)\right) \right\rangle \right) +  (\pi\diamond) + \gamma P_{\pi} \left(\pi \left\langle \phi, \frac{1}{T}\sum_{i=1}^{T-1} \omega^i_{\diamond} \right\rangle \right) \right.\\
&\left. \indent - (\pi\diamond) - \gamma P_{\pi} \bmvkd \right] \tag{By Lemma \ref{lem_sumthetaM}}\\
&= (I-\gamma P_{\pi})^{-1}  \left[ \left( \pi \left\langle \phi, W \left(\frac{1}{T} \sum_{i=1}^{T}(\hmqkid- \langle \phi,{\boldsymbol{\omega}}_{\diamond}^{i} \rangle)\right) \right\rangle \right) + \gamma P_{\pi} \left(\pi \left\langle \phi, \frac{1}{T}\sum_{i=1}^{T-1} \omega^i_{\diamond} \right\rangle \right) -  \gamma P_{\pi} \bmvkd \right] \\
&= (I-\gamma P_{\pi})^{-1}  \left[ \left( \pi \left\langle \phi, W \left(\frac{1}{T} \sum_{i=1}^{T}(\hmqkid- \langle \phi,{\boldsymbol{\omega}}_{\diamond}^{i} \rangle)\right) \right\rangle \right) + \gamma P_{\pi} \left(\pi \left\langle \phi, \frac{1}{T}\sum_{i=1}^{T-1} \omega^i_{\diamond} \right\rangle \right) \right.\\
&\left. \indent -  \gamma P_{\pi} \left(\pi \left\langle \phi, \frac{1}{T}\sum_{i=1}^{T} \omega^i_{\diamond} \right\rangle \right) \right] \tag{By definition of $\bmvkd$ in~\cref{ave_empvm}} \\
&= (I-\gamma P_{\pi})^{-1}  \left[ \left( \pi \left\langle \phi, W \left(\frac{1}{T} \sum_{i=1}^{T}(\hmqkid- \langle \phi,{\boldsymbol{\omega}}_{\diamond}^{i} \rangle)\right) \right\rangle \right) -  \gamma P_{\pi} \frac{1}{T} (\pi \langle \phi, \omega^T_{\diamond}\rangle) \right]. 
\end{align*}
Taking the infinity norm and using the triangle inequality, 
\begin{align*}
\norminf{\bmvkd - V^{\pi}_{\diamond}} & \leq \norminf{(I-\gamma P_{\pi})^{-1}   \left( \pi \left\langle \phi, W \left(\frac{1}{T} \sum_{i=1}^{T}(\hmqkid- \langle \phi,{\boldsymbol{\omega}}_{\diamond}^{i} \rangle)\right) \right\rangle \right)} \\
& + \norminf{(I-\gamma P_{\pi})^{-1} \gamma P_{\pi} \frac{1}{T} (\pi \langle \phi, \omega^T_{\diamond}\rangle)} \\
& \leq \norm{(I-\gamma P_{\pi})^{-1}}_{1, \infty} \, \norm{\pi \left\langle \phi, W \left(\frac{1}{T} \sum_{i=1}^{T}(\hmqkid- \langle \phi,{\boldsymbol{\omega}}_{\diamond}^{i} \rangle)\right) \right\rangle}_\infty \\
& + \norm{ (I-\gamma P_{\pi})^{-1} \gamma P_{\pi} }_{1, \infty} \, \norm{\frac{1}{T} (\pi \langle \phi, \omega^T_{\diamond}\rangle)}_\infty \tag{By Holder's inequality} \\
& \leq H \left[\norm{\pi \left\langle \phi, W \left(\frac{1}{T} \sum_{i=1}^{T}(\hmqkid- \langle \phi,{\boldsymbol{\omega}}_{\diamond}^{i} \rangle)\right) \right\rangle}_\infty + \norm{\frac{1}{T} (\pi \langle \phi, \omega^T_{\diamond}\rangle)}_\infty  \right] \tag{Since $\norm{(I-\gamma P_{\pi})^{-1}}_{1, \infty} \leq H$, $\norm{(I-\gamma P_{\pi})^{-1} \gamma P_{\pi}}_{1, \infty} \leq H$} \\
& \leq H \left[\norm{\left\langle \phi, W \left(\frac{1}{T} \sum_{i=1}^{T}(\hmqkid- \langle \phi,{\boldsymbol{\omega}}_{\diamond}^{i} \rangle)\right) \right\rangle}_\infty + \norm{\frac{1}{T} \langle \phi, \omega^T_{\diamond}\rangle}_\infty  \right]  \tag{By definition of the $\pi$ operator} \\
&\leq \tilde{O} \left( H^2 \sqrt{\frac{d}{TM}} + \frac{H^2}{T} \right). \tag{By Lemma \ref{lem_rangeVL2} and Lemma \ref{lem_thetaconcentLM}}
\end{align*}
Using that for any $|\cS|$-dimensional vector $V$, $V(\rho) = \mathbb{E}_{s \sim \rho}|V(s)| \leq \| V \|_{\infty}$ completes the proof.
\end{proof}

\subsubsection{Auxiliary Lemmas}
Since the updates in Algorithm \ref{alg_CVIL} are a special case of those in Algorithm \ref{alg_MDVICLMDP}, the proofs of the auxiliary lemmas are analogous.  We therefore present lemmas analogous to Lemma \ref{lem_rangeVL}, Lemma \ref{lem_sumtheta} and Lemma \ref{lem_thetaconcentL}, whose proofs follow by the same reasoning.

\begin{lemma}
\label{lem_rangeVL2}
For any  $t \in [T]$, with $\diamond = r \text{ or } c$ and $M \geq \tilde{O}\left( {dH^2} \right)$, we have
\begin{align*}
\| \langle \phi, \omega^t_{\diamond}\rangle\|_{\infty} \leq 2H \quad \text{and} \quad \| \hmvkttd{t} \|_{\infty} \leq 2H
\end{align*}
with probability at least $1-\delta$.
\end{lemma}

\begin{lemma}
\label{lem_sumthetaM}
For any $t\in [T]$, $\diamond = r \text{ or } c$, we have
\begin{align*}
\left\langle \phi, \frac{1}{T} \sum_{i=1}^T \omega^i_{\diamond} \right\rangle &=  \left\langle \phi, W \left(\frac{1}{T} \sum_{i=1}^{T}(\hmqkid- \langle \phi,{\boldsymbol{\omega}}_{\diamond}^{i} \rangle)\right) \right\rangle +  \diamond + \gamma P \left( \pi \left\langle \phi, \frac{1}{T} \sum_{i=1}^{T-1} \omega^i_{\diamond} \right\rangle \right).
\end{align*}
\end{lemma}

\begin{lemma}
\label{lem_thetaconcentLM}
With $\diamond = r \text{ or } c$ and $M \geq \tilde{O}\left( {dH^2} \right)$, we have
\begin{align*}
\left\|  \left\langle \phi, W \left(\frac{1}{T} \sum_{i=1}^{T}(\hmqkid- \langle \phi,{\boldsymbol{\omega}}_{\diamond}^{i} \rangle)\right) \right\rangle \right\|_{\infty} \leq \tilde{O}\left( { H}  \sqrt{\frac{d}{TM}} \right)
\end{align*}
with probability at least $1-\delta$.
\end{lemma}

\subsection{Proof of Corollary \ref{thm_main2C}}
\label{appendix_maincor}

\MainresultLC*

\begin{proof}
By Lemma \ref{lem_optGL} and Lemma \ref{lem_optG2L}, the sample complexity required to ensure $f(\mathcal{B}) \leq O(\epsilon)$ is $TM|\mathcal{C}| = \tilde{O}\left( \frac{d^2 H^4}{\epsilon^2} \right)$. Therefore, the guarantee for the relaxed feasibility setting follows directly from our meta-theorem (Theorem~\ref{thm_main2}). For the strict feasibility setting, we rescale $\epsilon$ by a factor of $O(\zeta(1 - \gamma))$. Since $\epsilon \leq 1$ and $1-\gamma \leq 1$, the condition of $f(\mathcal{B}) \leq \zeta/6$ in Theorem~\ref{thm_main2} can be satisfied. The rescaling increases the sample complexity by a multiplicative factor of $\frac{1}{\zeta^2 (1 - \gamma)^2}$, thereby completing the proof. 
\end{proof}

\subsection{Instantiating the \texttt{MDP-Solver}: G-Sampling-and-Stop}
\label{app:gss}

Instead of \texttt{LS-MDVI}, we can instantiate the linear \texttt{MDP-Solver} in Algorithm \ref{alg_CMDPL} with the \texttt{GSS} algorithm \citep{taupin2023best}. The \texttt{GSS} algorithm begins by computing a G-optimal sampling distribution over state-action pairs that minimizes the worst-case variance of value estimates. It then repeatedly samples transitions and rewards according to this distribution and uses regularized least-squares estimators to learn the reward and transition parameters of the MDP. For an arbitrary distribution $\tilde \rho$ over $\mathcal{S} \times \mathcal{A}$, let $G \in \mathbb{R}^{d \times d}$ and $g(\tilde\rho) \in \mathbb{R}$ be defined as:
\begin{align*}
G := \sum_{(x,b) \in \mathcal{C}} \tilde{\rho}(x, b) \, \phi(x, b) \phi(x, b)^\top 
\qquad \text{and} \qquad
g(\tilde\rho) := \max_{(s,a) \in \mathcal{S} \times \mathcal{A}} \langle \phi(s,a), G^{-1} \phi(s,a) \rangle,
\end{align*}
The \texttt{GSS} method samples one state-action pair $(s_t,a_t)\sim \rho^*$ in an iteration $t$ where $\rho^* := \argmin_{\rho \in \Delta_{\mathcal{S}\times\mathcal{A}}} g(\rho)$. We denote this data collection procedure as \texttt{DataCollection-GSS}. Note that this  is different than the sampling scheme used in~\cref{appendix_A}.

For solving a linear MDP, the \texttt{GSS} algorithm uses a stopping rule based on confidence bounds derived from matrix concentration inequalities, and determines when the estimates are accurate enough to ensure that the returned policy is $\epsilon$-optimal for the true MDP with high probability. The stopping time is denoted by
\begin{align*}
\tau= \inf \{t \geq 1: Z(t) \geq \beta(t)  \}
\end{align*}
where $\beta(t)$ is a certain threshold and $Z(t)$ is the quantity we seek to control in order to achieve the desired sample complexity.
Their main result in the setting of infinite-horizon $\gamma$-discounted linear unconstrained MDPs is stated below.
\begin{theorem}[Theorem 2 and Theorem 3 in \citep{taupin2023best}]
Let $\epsilon, \delta \in (0,1)$. The \texttt{GSS} algorithm returns an $\epsilon$-optimal policy with probability at least $1-\delta$, and the expected number of samples used is bounded by
\begin{align*}
O\left( \frac{d}{(1 - \gamma)^4  \epsilon^2} \left( \log\left( \frac{1}{\delta} \right) + d \log\left( \frac{d}{(1 - \gamma)^4 \epsilon^2} \right) \right)\right).
\end{align*}
\end{theorem}
Using the \texttt{GSS} algorithm as an alternative instantiation of \texttt{MDP-Solver}$(r + \lambda_k c, \mathcal{B}, \phi)$, we have that, with $N = \tilde{O}\left( \frac{d^2 H^4}{\epsilon^2} \right)$, the \texttt{GSS} algorithm satisfies Assumption~\ref{asmp_b1} with $f_{\mathrm{mdp}}(\mathcal{B}) = O(\epsilon)$. Hence, instantiating the three oracles by \texttt{DataCollection-GSS}, the \texttt{GSS} algorithm and using the same \texttt{PolicyEvaluation} oracle as in~\cref{alg_CVIL}, we can use our meta-theorem (Theorem~\ref{thm_main2}) to obtain the same sample complexity bounds as in~\cref{thm_main2C}.

\section{Lower Bound}
\label{sec_LB}

\LB*

\begin{proof}
We construct a lower bound by using the ideas from \cite{weisz2022confident,vaswani2022near}.

\paragraph{Hard Instance.} We construct a class of hard linear CMDPs, where each individual CMDP, denoted as $\mathcal{M}_{\beta}$, is parameterized by a vector $\beta$. Each MDP in $\mathcal{M}_{\beta}$ has four states: $\mathcal{S}=\{o, s_0,s_1, z\}$ with $o$ being the initial state. The action space is $\mathcal{A}= \bar{\mathcal{A}} \times \{0,1\}  = \{\pm1/\sqrt{d-5} \}^{d-5}\times \{0,1\}$, where $\bar{\mathcal{A}}$ is a subset of a $(d-5)$-dimensional hypercube. Thus, each action can be written as $a = (\bar{a}^{\top},a^{\prime})^{\top} \in \mathcal{A}$ where $\bar{a}$ is a $d-5$ dimensional vector and $a^{\prime} \in \{0,1\}$ is a scalar. Recall from Assumption \ref{asm_32} that the reward $r(s,a)=\langle \phi(s,a), \psi_r \rangle$, cost $ c(s,a)=\langle \phi(s,a), \psi_c \rangle$, and transition probability $\mathcal{P}(\cdot|s,a)=\langle \phi(s,a), \mu \rangle$ are all defined as linear functions. 
Let $\beta \in \bar{\mathcal{A}}$, $\Delta\in (0,0.2(1-\gamma)]$, $u=\frac{1-\gamma^2-\gamma\Delta}{\gamma}(b-\zeta)$. Note that the parameter $\zeta$ is the Slater constant for this CMDP instance. As can be verified from \cref{slater}, there exists a policy that achieves a constraint value of $b+\zeta$, satisfying the definition. Now, we define the following parameters.
\begin{align*}
\phi(s_0,a)&=(1,0,0,0,0,\bar{a}^{\top})^{\top}, & \phi(s_1,a)&=(0,1,0,0,0,\dots,0)^{\top}, \\
\phi(o,a) &= (0,0,a^{\prime},1-a^{\prime},0,\dots,0)^{\top}, &  \phi(z,a) &=(0,0,0,1,1,0,\dots,0)^{\top}\\
\mu(s_0)&=(\gamma,0,1,0,0,\Delta\beta^{\top})^{\top}, & \mu(s_1)&=(1-\gamma,1,0,0,0,-\Delta\beta^{\top})^{\top}, \\
\mu(o)&= (0,\dots,0)^{\top}, & \mu(z)&= (0,0,0,1,0,\dots,0)^{\top},\\
\psi_r&= (1,0,\dots,0)^{\top} , & \psi_c&= (u,0,0,0,(b+\zeta)(1-\gamma)/\gamma,0\dots,0)^{\top}.
\end{align*}
Under the above parameter settings, the transition dynamics is defined as follows.
\begin{align*}
\mathcal{P}_{\beta}(s_0|s_0,a) &= \gamma + \Delta\beta^\top \bar{a}, & \mathcal{P}_{\beta}(s_1|s_0,a) &= 1 - \gamma - \Delta\beta^\top \bar{a}, \\
\mathcal{P}_{\beta}(s_0|o,a) &= a^{\prime}, & \mathcal{P}_{\beta}(z|o,a) &= 1-a^{\prime},
\end{align*}
States $s_1$ and $z$ are absorbing. The reward and constraint functions are then specified as follows. For any action $a$,
\begin{align*}
r(s_0,a) &= 1, & r(z,a) &= 0,\\
c(s_0,a) &= u, & c(z,a) &= (b+\zeta)(1-\gamma)/\gamma.
\end{align*}
Rewards and costs for other states are zero. By solving the Bellman equation $V^{\pi}_{r,\beta}(s_0)=1+\gamma(\gamma+\Delta \mathbb{E}_{\bar{a}\sim \pi}[\bar{a}^{\top}\beta])V^{\pi}_{r,\beta}(s_0)$, we can define the value functions for a policy $\pi$ in the MDP $\mathcal{M}_{\beta}$. The notation $V^{\pi}_{r,\beta}$ represents the reward value function for policy $\pi$ within the specific MDP instance defined by $\beta$. Similarly, $V^{\pi}_{c,\beta}$ is the cost value function. The value functions starting from the initial state $o$ are
\begin{align}
V^{\pi}_{r,\beta}(o) &= \gamma \sum_{a\in\mathcal{A}}a^{\prime} \pi(a|o) V^{\pi}_{r,\beta}(s_0)= \frac{\gamma \sum_{a\in\mathcal{A}}a^{\prime}\pi(a|o)}{1-\gamma^2-\gamma \Delta \mathbb{E}_{\bar{a}\sim \pi}[\bar{a}^{\top}\beta]}, \label{eqn_vpibeta} \\
V^{\pi}_{c,\beta}(o) &= \gamma \sum_{a\in\mathcal{A}}a^{\prime}\pi(a|o) V^{\pi}_{c,\beta}(s_0) + \gamma \sum_{a\in\mathcal{A}} (1-a^{\prime})\pi(a|o)V^{\pi}_{c,\beta}(z) \nonumber \\
&= \frac{\gamma u \sum_{a\in\mathcal{A}} a^{\prime}\pi(a|o)}{1-\gamma^2-\gamma \Delta \mathbb{E}_{\bar{a}\sim \pi}[\bar{a}^{\top}\beta]} + \sum_{a\in\mathcal{A}}\pi(a|o)(1-a^{\prime})(b+\zeta). 
\end{align}
Note that $V_{r,\beta}^{\pi}(s_1)=0$ and $V^{\pi}_{c,\beta}(z)=\frac{1}{1-\gamma}\cdot\frac{(b+\zeta)(1-\gamma)}{\gamma}=\frac{b+\zeta}{\gamma}$. We now define the probability of a policy $\pi$ transitioning from $o$ to $s_0$ as $p:=\mathcal{P}^{\pi}_{\beta} [o\rightarrow s_0]$. According to the hard instance construction, $\mathcal{P}_{\beta}(s_0|o,a)=a^{\prime}$. Thus, $p$ is the marginal probability under $\pi$ that $a^{\prime}=1$ when in state $o$. This is equivalent to stating that the policy $\pi$ must satisfy:
\begin{align*}
\sum_{\bar{a}\in \bar{\mathcal{A}}}\pi((\bar{a}^{\top},1)^{\top}|o)=p \quad \text{and} \quad \sum_{\bar{a}\in \bar{\mathcal{A}}}\pi((\bar{a}^{\top},0)^{\top}|o)=1-p.
\end{align*}
Since we are in the strict feasibility setting, the constraint has to be satisfied as $V^{\pi}_{c,\beta}(o) \geq b$ for a feasible policy $\pi$. We denote the optimal feasible policy for the CMDP $\mathcal{M}_{\beta}$ as $\pi^*_{\beta}$.
Next, we show that the policy $\pi^*_{\beta}$ must satisfy: 
\begin{align*}
\sum_{\bar{a}\in \bar{\mathcal{A}}} \pi^*_{\beta}((\bar{a}^{\top},1)^{\top}|o)=\frac{1}{2} \quad \text{and} \quad \sum_{a^{\prime}\in \{0,1\}} \pi^*_{\beta}((\beta^{\top},a^{\prime})^{\top}|s_0)=1.
\end{align*}
Since states $z$ and $s_1$ are absorbing, the choice of policies on them is irrelevant. It is clear from \cref{eqn_vpibeta} that for any fixed $p$, $V^{\pi}_{r,\beta}(o)$ is maximized when $\mathbb{E}_{\bar{a}\sim {\pi}}[\bar{a}^{\top}\beta]={\beta}^{\top}\beta=1$. In this case, $V^{\pi}_{r,\beta}(o) = \frac{\gamma p}{1-\gamma^2-\gamma \Delta}$. Next, when $\mathbb{E}_{\bar{a}\sim {\pi}}[\bar{a}^{\top}\beta]={\beta}^{\top}\beta$, we observe that
\begin{align}
V^{\pi}_{c,\beta}(o) &= \frac{\gamma u p}{1-\gamma^2-\gamma \Delta} + (1-p)(b+\zeta) \nonumber \\
&= p(b-\zeta) + (1-p)(b+\zeta) \tag{$u=\frac{1-\gamma^2-\gamma\Delta}{\gamma}(b-\zeta)$} \nonumber \\
&= b + (1-2p) \zeta. \label{slater}
\end{align}
Therefore, for a policy to be feasible, it needs to ensure $p \leq 1/2$. The value function is then maximized when $p=1/2$. It means that, at the initial state $o$, the optimal policy should maintain an equal probability of moving into $s_0$ and $z$. At state $s_0$, the optimal policy should pick an action from the set $\{(\beta^{\top},0)^{\top},(\beta^{\top},1)^{\top}\}$ with probability 1. Note that for an action $a=(\bar{a}^{\top},a^{\prime})^{\top}$, the constraint satisfaction depends only on the action component $a^{\prime}$, and the value function at state $s_0$ is determined by $\bar a$. Thus, $\bar{a}$ and $a^{\prime}$ are optimized independently. For this optimal policy $\pi^*_{\beta}$, we have
\begin{align*}
V^{\pi^*_{\beta}}_{r,\beta}(o) = \frac{\gamma}{2(1-\gamma^2-\gamma \Delta)}, \quad \text{and} \quad V^{\pi^*_{\beta}}_{c,\beta}(o) = b.
\end{align*}
Now, consider an algorithm which incorrectly identifies the true parameter $\beta$ as a different parameter $\beta^{\prime}$. The algorithm then outputs a policy $\hat{\pi}$. According to this misestimated parameter, the policy satisfies $\mathbb{E}_{\bar{a}\sim \hat{\pi}}[\bar{a}^{\top}\beta]={\beta^{\prime}}^{\top}\beta$. Hence, in the true CMDP $\mathcal{M}_{\beta}$, its corresponding value function involves ${\beta^{\prime}}^{\top}\beta$. Suppose the probability of transition to $s_0$ from $o$ with this policy is $\hat{p}$. We have
\begin{align*}
V^{\hat{\pi}}_{c,\beta}(o) &= \frac{\gamma u \hat{p}}{1-\gamma^2-\gamma \Delta{\beta^{\prime}}^{\top}\beta} + (1-\hat{p})(b+\zeta)\\
&= \frac{\gamma u \hat{p}}{1-\gamma^2-\gamma \Delta} + (1-\hat{p})(b+\zeta) - \frac{\gamma u \hat{p}}{1-\gamma^2-\gamma \Delta}  + \frac{\gamma u \hat{p}}{1-\gamma^2-\gamma \Delta{\beta^{\prime}}^{\top}\beta}\\
&= \hat{p}(b-\zeta) + (1-\hat{p})(b+\zeta)  - \epsilon_c \tag{$u=\frac{1-\gamma^2-\gamma\Delta}{\gamma}(b-\zeta)$}\\
&= b + (1- 2\hat{p})\zeta - \epsilon_c 
\end{align*}
where
\begin{align*}
\epsilon_c = \frac{\gamma u \hat{p}}{1-\gamma^2-\gamma \Delta} -  \frac{\gamma u \hat{p}}{1-\gamma^2-\gamma \Delta{\beta^{\prime}}^{\top}\beta} = \frac{\gamma^2 u \hat{p}\Delta(1-{\beta^{\prime}}^{\top}\beta)}{(1-\gamma^2-\gamma \Delta)(1-\gamma^2-\gamma \Delta{\beta^{\prime}}^{\top}\beta)}.
\end{align*}
Therefore, in order for $\hat{\pi}$ to be feasible in the original CMDP $\mathcal{M}_{\beta}$, we require that
\begin{align}
(1- 2\hat{p})\zeta \geq \epsilon_c \Rightarrow \hat{p} \leq \frac{1}{2} - \frac{\epsilon_c}{2\zeta}. \label{ps_1}
\end{align}
In the meanwhile, for a policy $\hat{\pi}$ to be $(\epsilon,\delta)$-sound, it must satisfied that, with probability $1-\delta$,
\begin{align*}
\epsilon \geq V^{\pi^*_{\beta}}_{r,\beta}(o) - V^{\hat{\pi}}_{r,\beta}(o) = \frac{\gamma}{2(1-\gamma^2-\gamma \Delta)} - \frac{\hat{p}\gamma}{(1-\gamma^2-\gamma \Delta {\beta^{\prime}}^{\top}\beta)}\\
\end{align*}
which implies
\begin{align*}
\hat{p} &\geq \frac{(1-\gamma^2-\gamma \Delta {\beta^{\prime}}^{\top}\beta)}{ 2(1-\gamma^2-\gamma \Delta)} - \epsilon (1-\gamma^2-\gamma \Delta {\beta^{\prime}}^{\top}\beta)/\gamma  \\
&\geq \frac{(1-\gamma^2-\gamma \Delta)}{ 2(1-\gamma^2-\gamma \Delta)} - \epsilon (1-\gamma^2-\gamma \Delta {\beta^{\prime}}^{\top}\beta)/\gamma \tag{${\beta^{\prime}}^{\top}\beta \leq 1$} \\
&= \frac{1}{2} - \epsilon (1-\gamma^2-\gamma \Delta {\beta^{\prime}}^{\top}\beta)/\gamma.
\end{align*}
This lower bound implies that achieving higher rewards requires the policy to transition less frequently into the safe state.
Since $\epsilon \in (0,0.002)$, $\gamma \in [7/12,1)$, $\Delta \leq 0.2(1-\gamma)$, and ${\beta^{\prime}}^{\top}\beta \geq -1$, we have
\begin{align*}
\epsilon (1-\gamma^2-\gamma \Delta {\beta^{\prime}}^{\top}\beta)/\gamma \leq \epsilon (1-\gamma^2 + 0.2\gamma (1-\gamma))/\gamma  \leq \frac{3}{2}\epsilon \leq 0.003
\end{align*}
and thus
\begin{align}
\hat{p} \geq 0.5 - 0.003 \geq 0.4. \label{ps_2}
\end{align}
Now we analyze parameter $u$. Recall that $\zeta\in(0,49/2280)$ $\Delta\leq 1-\gamma$, $b\in[H/2,H]$, and $H = \frac{1}{1-\gamma}\in [12/5,\infty)$. We obtain
\begin{align*}
u=\frac{1-\gamma^2-\gamma\Delta}{\gamma}(b-\zeta) \geq  \frac{1-\gamma}{\gamma}\left(\frac{1}{2(1-\gamma)}- \frac{49}{2280}\right) \geq \frac{1}{2} - \frac{49}{2280} \geq\frac{1}{4}.
\end{align*}
Next, we denote $\tau$ as the total number of queries a planner sends to the generative model before it stops its process and outputs a final policy. For any $(\epsilon,\delta)$-sound algorithm outputting a policy $\hat{\pi}$ with query complexity $\tau$, we have (with probability $1-\delta$)
\begin{align*}
\epsilon &\geq V^{\pi^*_{\beta}}_{r,\beta}(o) - V^{\hat{\pi}}_{r,\beta}(o) \\
&= \frac{\gamma}{2(1-\gamma^2-\gamma \Delta)} - \frac{\hat{p}\gamma}{(1-\gamma^2-\gamma \Delta {\beta^{\prime}}^{\top}\beta)}\\
&\geq \frac{\gamma}{2(1-\gamma^2-\gamma \Delta)} - \frac{\hat{p}\gamma}{(1-\gamma^2-\gamma \Delta)} \tag{${\beta^{\prime}}^{\top}\beta \leq 1$}\\
&= \frac{\gamma}{(1-\gamma^2-\gamma \Delta)} \left( \frac{1}{2} - \hat{p} \right)\\
&\geq \frac{\gamma}{(1-\gamma^2)} \left( \frac{1}{2} - \hat{p} \right) \tag{$\Delta \geq 0$}\\
&= \frac{\gamma}{(1-\gamma)(1+\gamma)} \left( \frac{1}{2} - \hat{p} \right)\\
&\geq \frac{7}{19(1-\gamma)} \left( \frac{1}{2} - \hat{p} \right) \tag{$\gamma \in [7/12,1)$}\\
&\geq  \frac{7\epsilon_c}{19\zeta(1-\gamma)} \tag{by \cref{ps_1}}\\
&= \frac{7}{19\zeta(1-\gamma)}\left( \frac{\gamma u \hat{p}}{1-\gamma^2-\gamma \Delta} -  \frac{\gamma u \hat{p}}{1-\gamma^2-\gamma \Delta{\beta^{\prime}}^{\top}\beta} \right)\\
&= \frac{7\gamma u \hat{p}}{19\zeta(1-\gamma)} \left( \frac{1}{1-\gamma^2-\gamma \Delta} -  \frac{1}{1-\gamma^2-\gamma \Delta{\beta^{\prime}}^{\top}\beta} \right)\\
&= \frac{7\gamma u \hat{p}}{19\zeta(1-\gamma)}  \left(V^{\pi^*_{\beta}}_{r,\beta}(s_0) - V^{\hat{\pi}}_{r,\beta}(s_0) \right)\\
&\geq \frac{49}{2280\zeta(1-\gamma)}\left(V^{\pi^*_{\beta}}_{r,\beta}(s_0) - V^{\hat{\pi}}_{r,\beta}(s_0) \right). \tag{by the lower bound on $\gamma,u,$ and $\hat{p}$}
\end{align*}
Note that once the agent enters state $s_0$, the CMDP framework no longer plays a role. At this point, the problem has been reduced to establishing a lower bound for unconstrained linear MDPs where the starting state distribution is $s_0$, a result proven in \cite{weisz2022confident}. Note that the hard instance constructed in \cite{weisz2022confident} is identical to ours after the initial transition from initial state $o$ to $s_0$. 

Specifically, as shown in their Theorem H.3, under the parameter settings $\delta\in(0,0.08]$, $\gamma\in[7/12,1)$, $H=1/(1-\gamma)$, $\alpha\in(0,0.005\gamma H/(1+\gamma)^2)$, $\Delta\leq 0.2(1-\gamma)$, and $d\geq 3$, if $\tau$ is the number of samples, then, 
\begin{align*}
\mathbb{P}_{\beta}\left( V^{\pi^*_{\beta}}_{r,\beta}(s_0) - V^{\hat{\pi}}_{r,\beta}(s_0) \geq \alpha \right) \geq \frac{3}{35} - \frac{8}{35} \sqrt{1-\exp\left(-\frac{5\Delta^2H\mathbb{E}_{\beta}[\tau]}{(d-2)^2}\right)}
\end{align*}
which implies that the algorithm is not $(\alpha,\delta)$ sound unless $\mathbb{E}_{\beta}[\tau] \geq \Omega(d^2H^3/\alpha^2)$.
Substituting $\alpha=\frac{2280\epsilon\zeta(1-\gamma)}{49}$ completes our proof.
We note that the range of $\alpha$ in \cite{weisz2022confident} is $(0,0.005\gamma H/(1+\gamma)^2)$ where $0.005\gamma H/(1+\gamma)^2 \geq 0.00279$. Replacing $\alpha=\frac{2280\epsilon\zeta(1-\gamma)}{49}\leq \epsilon \in (0,0.002)$ satisfies their choice of range. Except for the range of $\epsilon$ and the lower bound on $d$, other parameter choices exactly match those in \cite{weisz2022confident}.
\end{proof}
\section{Algorithms for Solving Tabular CMDPs}
\label{appendix_tabalg}

\begin{algorithm}[H]
\SetAlgoLined
\DontPrintSemicolon
\caption{Tabular Mirror Descent Value Iteration (\texttt{Tabular-MDVI})}
\label{alg_MDVICMDP}

\KwIn{$T$ (number of iterations), $M$ (number of next-state samples obtained per state-action pair in each iteration), $\Box$ (rewards in MDP), $\mathcal{B}=\mathcal{B}_0 \cup \cdots \cup \mathcal{B}_{T-1}$ (Buffer).}
\KwOut{$\pi_{T}$ where $\forall s \in \mathcal{S}: {\pi}_{T}(\cdot |s) \in {\argmax_{a}}  \tilde{Q}^{T}_{\Box}(s,a) $.}
\textbf{Initialize:} $\hvkttb{0} = \bm{0}$, $ \hqktb{-1}= \bm{0}$.

\BlankLine

\SetKwProg{Proc}{procedure}{}{end}
\Proc{\texttt{Tabular-MDVI}($T$, $M$, $\Box$, $\mathcal{B}$)}{
    \For{$t = 0,1,2 \dots,T-1$}{ 
        \ \  $\forall (s,a) \in \mathcal{S} \times \mathcal{A}:$ Access $(s,a,s^{\prime}_m)_{m=1}^{M}$ from the buffer $\mathcal{B}_t$.\\
        $\forall (s,a) \in \mathcal{S} \times \mathcal{A}: \hqktb{t} (s, a) = \Box (s, a) + \gamma \frac{1}{M} \sum_{m=1}^M \hvkttb{t}(s^{\prime}_m)$.\\
        Define $\tilde{Q}^{t}_{\Box} = \sum_{i=0}^{t} \hqktb{i}$; $\forall s \in \mathcal{S}: \hvkttb{t+1} (s) = \underset{{a}}{\max} \{ \tilde{Q}^{t}_{\Box}(s,a) \} -  \underset{{a}}{\max} \{ \tilde{Q}^{t-1}_{\Box}(s,a) \}$. 
    }
}
\end{algorithm}

\begin{algorithm}[H]
\SetAlgoLined
\DontPrintSemicolon
\caption{Tabular Policy Evaluation (\texttt{Tabular-PE})}
\label{alg_CVI}

\KwIn{$T$ (number of iterations), $M$ (number of next-state samples obtained per state-action pair in each iteration), $\diamond$ (either r or c), $\mathcal{B} = \mathcal{B}_0 \cup \cdots \cup \mathcal{B}_{T-1}$ (Buffer), $\pi$ (policy to be evaluated).}
\KwOut{$\bmvkd(\rho) = \frac{1}{T}\sum_{i=1}^T\hmvkttd{i}(\rho)$.}
\textbf{Initialize:} $\hmvkttd{0} = \bm{0}$.

\BlankLine

\SetKwProg{Proc}{procedure}{}{end}
\Proc{\texttt{Tabular-PE}($T$, $M$, $\diamond$, $\mathcal{B}$, $\pi$)}{
    \For{$t = 0,1,2 \dots,T-1$}{ 
        \ \  $\forall (s,a) \in \mathcal{S} \times \mathcal{A}:$ Access $(s,a,s^{\prime}_m)_{m=1}^{M}$ from the buffer $\mathcal{B}_t$.\\
        $\forall (s,a) \in \mathcal{S} \times \mathcal{A}: \hmqktd{t} (s, a) = \diamond (s, a) + \gamma \frac{1}{M} \sum_{m=1}^M \hmvkttd{t}(s^{\prime}_m)$.\\
        $\hmvkttd{t+1}  =  \pi \hmqktd{t}$.
    }
}
\end{algorithm}
\section{Instantiating the Framework for Tabular Constrained MDPs}
\label{sec_tabularCMDP}

We now instantiate the framework for tabular CMDPs, and prove that the resulting algorithm attains near-optimal sample complexity. In contrast to the linear setting, we set $\mathcal{C} = \mathcal{S} \times \mathcal{A}$ as the input to the \texttt{DataCollection} oracle. For the \texttt{MDP-Solver} and \texttt{PolicyEvaluation}, we adapt~\cref{alg_MDVICLMDP,alg_CVIL} to the tabular setting. In particular, for both these algorithms, we set the features to be $|\mathcal{S}||\mathcal{A}|$ dimensional one-hot encodings of the state-action space implying that the feature map $\phi$ is an $|\mathcal{S}||\mathcal{A}|$-dimensional identity matrix. Consequently, the resulting algorithm does not require linear regression to estimate the $Q$-function. We provide the pseudo-code for these two instantiations is provided in~\cref{appendix_tabalg}. Their corresponding optimality guarantees are proved in~\cref{proofs_sec5} and stated below.
\begin{restatable}{lemma}{OPGuaranteea}
\label{lem_mdvi}
For a fixed $\epsilon \in (0,1/H^2]$, $\delta \in (0,1)$, any $k \in [K]$, and $T\geq 2\log(T)/\gamma$, when using~\cref{alg_MDVICMDP} at iteration $k$ of~\cref{alg_CMDPL} with $\Box = r+\lambda_k c$, $M= \tilde{O}\left( \frac{H}{\epsilon} \right)$ and $T = {O}\left( \frac{H^2}{\epsilon} \right)$, the output policy $\pi_{T}$ satisfies the following condition with probability $1-\delta$,  
\begin{align*}
\max_\pi V^{\pi}_{r+\lambda_k c}(\rho) -  V^{\pi_{T}}_{r+\lambda_k c}(\rho) \leq {O}((1+\lambda_k)\epsilon)\,,
\end{align*}

The resulting sample complexity is $N = T \, M |\mathcal{C}|= \tilde{O}\left( \frac{|\mathcal{S}||\mathcal{A}| H^3}{\epsilon^2}\right)$.
\end{restatable}
\begin{restatable}{lemma}{OPGuaranteeb}
\label{lem_concen}
For a fixed $\epsilon \in (0,H]$, $\delta \in (0,1)$,~\cref{alg_CVI} with $M= \tilde{O}\left( \frac{H}{\epsilon} \right)$ and $T = {O}\left( \frac{H^2}{\epsilon} \right)$, the output $\bar{\mathcal{V}}^T_{\diamond}$ satisfies the following condition with probability $1-\delta$,  
\begin{align*}
|\bar{\mathcal{V}}^T_{\diamond}(\rho) - V^{\pi}_{\diamond}(\rho) | \leq {O}\left( \epsilon \right)\,,
\end{align*}
The resulting sample complexity is $N = T \, M |\mathcal{C}| = \tilde{O}\left( \frac{|\mathcal{S}||\mathcal{A}| H^3}{\epsilon^2} \right)$.  
\end{restatable}

The proofs of~\cref{lem_mdvi,lem_concen} can use the total variance technique and a Bernstein-type concentration argument~\citep{azar2013minimax,kozuno2022kl} and result in near-optimal bounds in the tabular setting. Moreover, the corresponding algorithms do not require constructing coresets or using (variance-weighted) linear regression. Consequently, unlike the linear setting in~\cref{sec_linearCMDP}, the same buffer $\mathcal{B}$ can be reused across all iterations of~\cref{alg_CMDPL}. This allows the near-optimal sample complexities of both~\cref{alg_MDVICMDP,alg_CVI} to be preserved for tabular CMDPs. In particular, we prove the following result in~\cref{appendix_maincor2T}.
\begin{restatable}{corollary}{MainresultLCT}
\label{thm_main2CT}
Let \cref{alg_MDVICMDP} and \cref{alg_CVI} be the instantiations of the \texttt{MDP-Solver} and \texttt{PolicyEvaluation} in \cref{alg_CMDPL}. For a fixed $\epsilon \in (0,1/H^2]$, $\delta \in (0,1)$, \cref{alg_CMDPL} with $\tilde{O}\left( \frac{|\mathcal{S}||\mathcal{A}| H^3}{\epsilon^2} \right)$ samples, $U=O\left(\frac{1}{\zeta(1-\gamma)}\right)$, $\eta= \frac{U(1-\gamma)}{\sqrt{K}}$, $K=O\left(\frac{1}{\epsilon^2 \, (1-\gamma)^2} \right)$, and $b^{\prime}=b-O(\epsilon)$, returns a policy $\bar{\pi}$ satisfying the following condition with probability $1-\delta$,  
\begin{align*}
\vbkr(\rho) \geq \vsr(\rho) - O(\epsilon), \quad \text{and} \quad \vbkc(\rho) \geq b - O(\epsilon).
\end{align*} 

Under the same conditions, but with $b^{\prime} = b+O(\epsilon)$ and $\tilde{O}\left( \frac{|\mathcal{S}||\mathcal{A}| H^5}{\zeta^2\epsilon^2} \right)$ samples,~\cref{alg_CMDPL} returns a policy $\bar{\pi}$ satisfying the following condition with probability $1-\delta$,  
\begin{align*}
\vbkr(\rho) \geq \vsr(\rho) - O({\epsilon}), \quad \text{and} \quad \vbkc(\rho) \geq b.
\end{align*} 
\end{restatable}

The above result matches the near-optimal sample complexity bounds attained by the model-based algorithm in~\citet{vaswani2022near}. Furthermore, instantiating the \texttt{MDP-Solver} to be the model-based algorithm~\citep{agarwal2020model,li2020breaking} and using~\cref{alg_CMDPL} will result in a near-optimal sample complexity for solving tabular CMDPs (see~\cref{app:model-based} for details). Note that the \texttt{MDP-Solver} can also be instantiated by a range of model-free algorithms for solving unconstrained MDPs with access to a generative model~\citep{azar2013minimax, sidford2023variance,sidford2018near,wang2017randomized, jin2024truncated}. Consequently, our framework can be interpreted as a generalization of the the primal-dual approach in~\citep{vaswani2022near} to handle model-free algorithms and linear function approximation.  

\section{Proofs for Section \ref{sec_tabularCMDP}}
\label{proofs_sec5}

Throughout, we treat $\pi$ as an operator that returns an $|\mathcal{S}|$-dimensional vector s.t. for an arbitrary $|\mathcal{S}||\mathcal{A}|$-dimensional vector $u$ such that $(\pi u)(s):=\sum_{a\in \mathcal{A}}\pi(a|s) \, u(s,a)$. Furthermore, we define $P_{\pi}:= \pi P$ where $P_{\pi} \in \mathbb{R}^{|\mathcal{S}|\times |\mathcal{S}|}$ and denotes the transition probability matrix induced by policy $\pi$. We also recall that $\pi^*_k := \argmax_{\pi} V^{\pi}_{r+\lambda_kc}$ and define $\bvkb := \frac{1}{T} \sum_{i=1}^T \hat{V}_{\Box}^i $. We define $y_{t,m,s,a}$ to be the $m$-th next-state sample $s'_m$ corresponding to the state-action pair $(s, a)$ at iteration $t$. For a value function $V$, $\mathrm{Var}(V)$ denote the function
\begin{align*}
\mathrm{Var}(V): (s,a) \mapsto (PV^2)(s,a) - (PV)^2(s,a)
\end{align*}
and $\sigma(V):= \sqrt{\mathrm{Var}(V)}$.

\subsection{Proof of Lemma \ref{lem_mdvi} (Optimality Guarantees for Algorithm \ref{alg_MDVICMDP} - Tabular CMDP)}
\label{sec_mdvigua}

\OPGuaranteea*

\begin{proof}
By Lemma \ref{lem_conMDPb} and Lemma \ref{lem_B2b}, we have
\begin{align*}
\vkb(\rho) - \vktb(\rho) &= \vkb(\rho) - \bvkb(\rho) + \bvkb(\rho) - \vktb(\rho)\\
&\leq  \frac{7H^2(1+\lambda_k)}{T} + \sqrt{\frac{6H^3}{TM}} + \sqrt{\frac{6H^6(1+\lambda_k)^2}{TM}\left( \frac{50H^2}{T^2} +\frac{4\iota^2}{M} \right)}
\end{align*}
with probability at least $1-\delta$. By letting $M=\frac{6H\iota^2}{\epsilon}$, $T = \frac{82H^2}{\epsilon}$, and $\epsilon \in (0,1/H^2]$, we have
\begin{align*}
\vkb(\rho) - \vktb(\rho) &\leq (1+\lambda_k)\epsilon/12 + \epsilon/9\iota + (1+\lambda_k)H^{3/2} \epsilon \left( \frac{\epsilon}{9H} + \sqrt{\epsilon/81H} \right)  \\
&= (1+\lambda_k)\epsilon/12 + \epsilon/9\iota + (1+\lambda_k) \epsilon^2\sqrt{H}/9  + (1+\lambda_k)H \epsilon^{3/2} /9  \\
&\leq (1+\lambda_k)\epsilon/12 + \epsilon/9 + (1+\lambda_k) \epsilon/9  + (1+\lambda_k)\epsilon /9 \tag{$\epsilon \in (0,1/H^2]$ and $\iota = \log(2|\mathcal{S}||\mathcal{A}|/\delta)\geq1$} \\
&\leq (1+\lambda_k)\epsilon
\end{align*}
which completes the proof. 
\end{proof}

\subsubsection{Proof of Lemma \ref{lem_conMDP} and Lemma \ref{lem_B2} (Proofs with Hoeffding's Inequality)}

\begin{lemma}
\label{lem_conMDP}
Let $\pi^*_k$ be defined as in \cref{pi_sk}, and let $\bvkb$ denote the averaged empirical value function in Algorithm \ref{alg_MDVICMDP} when run with $\lambda_k$. For any $k \in [K]$, we have
\begin{align*}
\vkr(\rho) + \lambda_k \vkc(\rho)  - \bvkr(\rho) - \lambda_k \bvkc(\rho)  \leq \frac{3H^2(1+\lambda_k)}{T} + 2H(1+\lambda_k) \sqrt{\frac{\log(2|\mathcal{S}||\mathcal{A}|/\delta)}{TM}}
\end{align*}
with probability at least $1-\delta$.
\end{lemma}

\begin{proof} 
Since $(I-\gamma P_{\pi^*_k})\vkb = (\pi^*_k \Box )$, we have
\begin{align}
(I-\gamma P_{\pi^*_k})(\vkb - \bvkb) &= (\pi^*_k \Box) - (\bvkb - \gamma P_{\pi^*_k} \bvkb) \nonumber \\
&= (\pi^*_k \Box) + \gamma P_{\pi^*_k } \bvkb - \bvkb \nonumber\\
\Longrightarrow  \vkb - \bvkb &=  (I-\gamma P_{\pi^*_k})^{-1} ((\pi^*_k \Box) + \gamma P_{\pi^*_k } \bvkb - \bvkb)  \label{eqn_piout}
\end{align}
By Lemma \ref{lem_televq} and due to the greediness of $\pi_{t}$, for all $t\in[T]$, we have
\begin{align}
\bvktb  &=  \frac{1}{t} \sum_{i=0}^{t-1} (\pi_{t} \hqkib)   \nonumber \\
&\geq   \frac{1}{t} \sum_{i=0}^{t-1} (\pi^*_k \hqkib ). \label{eqn_VQ}
\end{align}
Now, we have
\begin{align*}
\vkb - \bvkb &= (I-\gamma P_{\pi^*_k})^{-1} ((\pi^*_k \Box) + \gamma P_{\pi^*_k } \bvkb - \bvkb)   \tag{By \cref{eqn_piout}} \\
&\leq (I-\gamma P_{\pi^*_k})^{-1}  \left( (\pi^*_k\Box) + \gamma P_{\pi^*_k} \bvkb -  \frac{1}{T} \sum_{i=0}^{T-1} (\pi^*_k \hqkib) \right) \tag{By \cref{eqn_VQ}}\\
&= (I-\gamma P_{\pi^*_k})^{-1} \left( (\pi^*_k \Box) + \gamma P_{\pi^*_k } \bvkb -  (\pi^*_k \Box) - \gamma P_{\pi^*_k }  \frac{1}{T} \sum_{i=0}^{T-2}   (\pi_{T-1}\hqkib)  - \frac{1}{T} \sum_{i=0}^{T-1} \left[\gamma \hat{P}^i_{\pi^*_k} \hvkib - \gamma P_{\pi^*_k} \hvkib  \right] \right) \tag{By Lemma \ref{lem_sumQ}}\\
&= (I-\gamma P_{\pi^*_k})^{-1} \left( \gamma P_{\pi^*_k} \bvkb - \gamma P_{\pi^*_k }  \frac{1}{T} \sum_{i=0}^{T-2}  (\pi_{T-1} \hqkib)  - \frac{1}{T} \sum_{i=0}^{T-1} \left[\gamma \hat{P}^i_{\pi^*_k} \hvkib - \gamma P_{\pi^*_k} \hvkib  \right] \right)\\
&= (I-\gamma P_{\pi^*_k})^{-1}  \left( \gamma P_{\pi^*_k} \bvkb - \gamma P_{\pi^*_k}  \frac{1}{T-1} \sum_{i=0}^{T-2} (\pi_{T-1}  \hqkib ) - \frac{1}{T} \sum_{i=0}^{T-1} \left[\gamma \hat{P}^i_{\pi^*_k} \hvkib - \gamma P_{\pi^*_k} \hvkib  \right] \right)\\
&\indent + (I-\gamma P_{\pi^*_k})^{-1}  \left(  \frac{1}{T(T-1)} \gamma P_{\pi^*_k}  \sum_{i=0}^{T-2}  (\pi_{T-1} \hqkib)  \right).
\end{align*}
We note that $ \frac{1}{T-1} \sum_{i=0}^{T-2}  (\pi_{T-1} \hqkib) = \bar{V}^{T-1}_{\Box}$ by Lemma \ref{lem_televq}. By defining $\mathcal{H}_{\pi^*_k} := \gamma (I-\gamma P_{\pi^*_k})^{-1} \pi^*_k \in \mathbb{R}^{|\mathcal{S}|\times|\mathcal{A}|}$, we obtain
\begin{align}
\vkb - \bvkb \leq \underbrace{\mathcal{H}_{\pi^*_k} P( \bvkb - \bvkttb{T-1})}_{\text{Term (i)}} +  \underbrace{\mathcal{H}_{\pi^*_k} \frac{1}{T} \sum_{i=0}^{T-1} \left[ P \hvkib -  \hat{P}_i \hvkib  \right]}_{\text{Term (ii)}} + \underbrace{\mathcal{H}_{\pi^*_k} \left(  \frac{1}{T(T-1)}  P  \sum_{i=0}^{T-2}  (\pi_{T-1} \hqkib)  \right)}_{\text{Term (iii)}}. \label{eqn_boxdec1}
\end{align}
Note that for any vector $Q\in \mathbb{R}^{|\mathcal{S}|\times|\mathcal{A}|}$, 
\begin{align*}
\|\mathcal{H}_{\pi_{k}^*} Q\|_{\infty} &= \|\gamma (I-\gamma P_{\pi_{k}^*})^{-1} \pi_{k}^* Q\|_{\infty} \\
&\leq \|\gamma (I-\gamma P_{\pi_{k}^*})^{-1}\|_1 \| \pi_{k}^* Q\|_{\infty} \tag{By Holder’s inequality}\\
&\leq H \| \pi_{k}^* Q\|_{\infty} \tag{Since $\| \gamma (I-\gamma P_{\pi_{k}^*})^{-1} \|_1 \leq H$}\\
&\leq H \| Q\|_{\infty}. \tag{By definition of the $\pi$ operator}
\end{align*}
In order to bound Term (i), using~\cref{lem_vmv}, we have
\begin{align*}
\left\| \mathcal{H}_{\pi^*_k} P(\bvkb - \bar{V}^{T-1}_{\Box} ) \right\|_{\infty} \leq \frac{2H^2(1+\lambda_k)}{T}.
\end{align*}
For bounding Term (ii), letting $t=T$ in Lemma \ref{lem_concentrate} and invoking it twice for $r$ and $c$, we have
\begin{align*}
\left\| \mathcal{H}_{\pi^*_k}\frac{1}{T} \sum_{i=0}^{T-1} \left[ P \hvkib -  \hat{P}_i \hvkib  \right]   \right\|_{\infty} \leq 2H^2(1+\lambda_k) \sqrt{\frac{\iota}{TM}} 
\end{align*}
with probability at least $1-\delta$. 

Finally, we bound Term (iii) by noting that $\| \sum_{i=0}^{T-2}  (\pi_{T-1} \hqkib)  \|_{\infty}\leq (T-1)H(1+\lambda_k)$ due to Lemma \ref{lem_rangevq}. Hence,
\begin{align*}
\left\| \mathcal{H}_{\pi^*_k} \left(  \frac{1}{T(T-1)}  P  \sum_{i=0}^{T-2}   (\pi_{T-1}\hqkib)  \right)  \right\|_{\infty} \leq \frac{H^2(1+\lambda_k)}{T}.
\end{align*}
Note that for any vector $V$, $V(\rho) \leq \| V \|_{\infty}$. Putting everything together, we have
\begin{align*}
\vkr(\rho) + \lambda_k \vkc(\rho)  - \bvkr(\rho) - \lambda_k \bvkc(\rho)  \leq \frac{3H^2(1+\lambda_k)}{T} + 2H^2(1+\lambda_k) \sqrt{\frac{\iota}{TM}}
\end{align*}
with probability at least $1-\delta$.

\end{proof}

\begin{lemma}
\label{lem_B2}
Let $\pi_{T}$ be the output policy, and let $\bvkd$ denote the averaged empirical value function in Algorithm \ref{alg_MDVICMDP} when run with $\lambda_k$. For any $k \in [K]$, we have
\begin{align*}
 \bvkr(\rho) + \lambda_k \bvkc(\rho) - \vktr(\rho) - \lambda_k \vktc(\rho) \leq  \frac{2H^2(1+\lambda_k)}{T} + 2H^2(1+\lambda_k) \sqrt{\frac{\iota}{TM}}
\end{align*}
with probability at least $1-\delta$.
\end{lemma}

\begin{proof}  
The proof follows similar steps as before. Since $(I-\gamma P_{\pi_{T}})\vktb = \pi_{T} \Box $, we have
\begin{align}
(I-\gamma P_{\pi_{T}})(\bvkb - \vktb) &= (\bvkb - \gamma P_{\pi_{T}} \bvkb) - \pi_{T} \Box  \nonumber \\
&= \bvkb -  (\pi_{T}\Box) + \gamma P_{\pi_{T}} \bvkb  \nonumber\\
\Longrightarrow  \bvkb - \vktb  &=  (I-\gamma P_{\pi_{T}})^{-1} ( \bvkb -  (\pi_{T}\Box) + \gamma P_{\pi_{T}} \bvkb)   \label{eqn_piout2}
\end{align}
Recall that for all $t\in[T]$, we have
\begin{align}
\bvktd = \frac{1}{t} \sum_{i=1}^{t}  \hvkid = \frac{1}{t} \sum_{i=0}^{t-1} (\pi_{t} \hqkid) . \label{eqn_VQ2}
\end{align}
Now, we have
\begin{align*}
\bvkb - \vktb  &=  (I-\gamma P_{\pi_{T}})^{-1} (\bvkb - (\pi_{T}\Box) + \gamma P_{\pi_{T}} \bvkb) \tag{By \cref{eqn_piout2}} \\
&= (I-\gamma P_{\pi_{T}})^{-1} \left(  \frac{1}{T} \sum_{i=0}^{T-1} (\pi_{T} \hqkib)  -  (\pi_{T}\Box) + \gamma P_{\pi_{T}} \bvkb \right) \tag{By \cref{eqn_VQ2}}\\
&= (I-\gamma P_{\pi_{T}})^{-1}  \left(  \frac{1}{T} \sum_{i=0}^{T-1} (\pi_{T}\hqkib) - (\pi_{T} \Box) - \gamma P_{\pi_{T}} \bvkb  \right)\\
&= (I-\gamma P_{\pi_{T}})^{-1}  \left(  (\pi_{T}\Box) + \gamma P_{\pi_{T}} \frac{1}{T} \sum_{i=0}^{T-2}  (\pi_{T-1} \hqkib)  + \frac{1}{T} \sum_{i=0}^{T-1} \left[\gamma \hat{P}^i_{\pi_T} \hvkib - \gamma P_{\pi_T} \hvkib  \right] \right. \\
&\left. \quad - (\pi_{T} \Box) - \gamma P_{\pi_{T}} \bvkb \right) \tag{By Lemma \ref{lem_sumQ}}\\
&= (I-\gamma P_{\pi_{T}})^{-1}  \left( \gamma P_{\pi_{T}}  \frac{1}{T} \sum_{i=0}^{T-2}  (\pi_{T-1} \hqkib)  + \frac{1}{T} \sum_{i=0}^{T-1} \left[\gamma \hat{P}^i_{\pi_T} \hvkib - \gamma P_{\pi_T} \hvkib  \right]  - \gamma P_{\pi_{T}} \bvkb \right)\\
&\leq (I-\gamma P_{\pi_{T}})^{-1}  \left( \gamma P_{\pi_{T}}  \frac{1}{T-1} \sum_{i=0}^{T-2}  (\pi_{T-1} \hqkib ) + \frac{1}{T} \sum_{i=0}^{T-1} \left[\gamma \hat{P}^i_{\pi_T} \hvkib - \gamma P_{\pi_T} \hvkib  \right]  - \gamma P_{\pi_{T}} \bvkb \right).
\end{align*}
We note that $\frac{1}{T-1} \sum_{i=0}^{T-2}  (\pi_{T-1} \hqkib) = \bar{V}^{T-1}_{\Box}$. By letting $\mathcal{H}_{\pi_{T}} = \gamma (I-\gamma P_{\pi_{T}})^{-1} \pi_{T} $, we obtain
\begin{align}
 \bvkb - \vktb \leq \mathcal{H}_{\pi_{T}}P( \bar{V}^{T-1}_{\Box} - \bvkb ) +  \mathcal{H}_{\pi_{T}} \frac{1}{T} \sum_{i=0}^{T-1} \left[   \hat{P}_i \hvkib  - P \hvkib  \right] . \label{eqn_boxdec2}
\end{align}
Note that for any vector $Q\in \mathbb{R}^{|\mathcal{S}|\times|\mathcal{A}|}$, 
\begin{align*}
\|\mathcal{H}_{\pi_{T}} Q\|_{\infty} &= \|\gamma (I-\gamma P_{\pi_{T}})^{-1} \pi_{T} Q\|_{\infty} \\
&\leq \|\gamma (I-\gamma P_{\pi_{T}})^{-1}\|_1 \| \pi_{T} Q\|_{\infty} \tag{By Holder’s inequality}\\
&\leq H \| \pi_{T} Q\|_{\infty} \tag{Since $\| \gamma (I-\gamma P_{\pi_{T}})^{-1} \|_1 \leq H$}\\
&\leq H \| Q\|_{\infty}. \tag{By definition of the $\pi$ operator}
\end{align*}
Thus, letting $t=T$ in Lemma \ref{lem_concentrate}, we have
\begin{align*}
\left\| \mathcal{H}_{\pi_{T}} \frac{1}{T} \sum_{i=0}^{T-1} \left[ \hat{P}_i \hvkib -  P \hvkib 
  \right]   \right\|_{\infty} \leq 2H^2(1+\lambda_k) \sqrt{\frac{\iota}{TM}}
\end{align*}
with probability at least $1-\delta$. By Lemma \ref{lem_vmv},
\begin{align*}
\left\| \mathcal{H}_{\pi_{T}} P(\bar{V}^{T-1}_{\Box} - \bvkb  ) \right\|_{\infty} \leq \frac{2H^2(1+\lambda_k)}{T}.
\end{align*}
Note that for any vector $V$, $V(\rho) \leq \| V \|_{\infty}$. Putting everything together, we have
\begin{align*}
 \bvkr(\rho) + \lambda_k \bvkc(\rho) - \vktr(\rho) - \lambda_k \vktc(\rho)  \leq \frac{2H^2(1+\lambda_k)}{T} + 2H^2(1+\lambda_k) \sqrt{\frac{\iota}{TM}}
\end{align*}
with probability at least $1-\delta$.

\end{proof}

\subsubsection{Proof of Lemma \ref{lem_conMDPb} and Lemma \ref{lem_B2b} (Proofs with Bernstein's Inequality)}

\begin{lemma}
\label{lem_conMDPb}
Let $\pi^*_k$ be defined as in \cref{pi_sk}, and let $\bvkd$ denote the averaged empirical value function in Algorithm \ref{alg_MDVICMDP} when run with $\lambda_k$. For any $k \in [K]$ and $T\geq 2\log(T)/\gamma$, we have
\begin{align*}
\vkr(\rho) + \lambda_k \vkc(\rho)  - \bvkr(\rho) - \lambda_k \bvkc(\rho)  \leq \sqrt{\frac{3H^4(1+\lambda_k)^2}{TM}  \left( \frac{4}{T^2} + {\frac{16H^2\iota^2}{M}}  \right)} +  \sqrt{\frac{3H^3}{TM}}+ \frac{4H^2(1+\lambda_k)}{T}.
\end{align*}
with probability at least $1-\delta$.
\end{lemma}

\begin{proof}
From \cref{eqn_boxdec1}, we have
\begin{align*}
\vkb - \bvkb \leq \underbrace{\mathcal{H}_{\pi^*_k} P( \bvkb - \bvkttb{T-1})}_{\text{Term (i)}} +  \underbrace{\mathcal{H}_{\pi^*_k} \frac{1}{T} \sum_{i=0}^{T-1} \left[ P \hvkib -  \hat{P}_i \hvkib  \right]}_{\text{Term (ii)}} + \underbrace{\mathcal{H}_{\pi^*_k} \left(  \frac{1}{T(T-1)}  P  \sum_{i=0}^{T-2}  (\pi_{T-1} \hqkib)  \right)}_{\text{Term (iii)}}
\end{align*}
We bound Term (i) and Term (iii) the same way as before. Thus, we only have Term (ii) remains. By Lemma \ref{lem_bern5}, we know
\begin{align*}
\frac{1}{t}\sum_{i=1}^t \left[ \hat{P}_i \hvkttb{i} -  P \hvkttb{i}  \right] (s,a) \leq \frac{H(1+\lambda_k)\iota}{tM} + \sqrt{Z}
\end{align*}
where 
\begin{align*}
Z:=\frac{3H^2(1+\lambda_k)^2}{tM} \left( \frac{4}{T^2} + {\frac{16H^2\iota^2}{M}} \right) + \frac{3\mathrm{Var}(\vkb(s,a))}{tM}.
\end{align*}
Therefore,
\begin{align*}
\mathcal{H}_{\pi^*_k} \frac{1}{T} \sum_{i=0}^{T-1} \left[ P \hvkib -  \hat{P}_i \hvkib  \right] &\leq \mathcal{H}_{\pi^*_k} \sqrt{\frac{3H^2(1+\lambda_k)^2}{TM}  \left( \frac{4}{T^2} + {\frac{16H^2\iota^2}{M}}  \right)} \mathbf{1} + \mathcal{H}_{\pi^*_k} \frac{H(1+\lambda_k)\iota}{TM}\mathbf{1} +  \mathcal{H}_{\pi^*_k} \sqrt{\frac{3}{TM}} \sigma(\vkb)\\
&\leq \mathcal{H}_{\pi^*_k} \sqrt{\frac{3H^2(1+\lambda_k)^2}{TM} \left( \frac{4}{T^2} + {\frac{16H^2\iota^2}{M}}  \right)} \mathbf{1} + \mathcal{H}_{\pi^*_k} \frac{H(1+\lambda_k)\iota}{TM}\mathbf{1} +  \sqrt{\frac{3H^3}{TM}} \mathbf{1} \tag{By Lemma \ref{lem_TVL}}\\
&\leq \sqrt{\frac{3H^4(1+\lambda_k)^2}{TM}  \left( \frac{4}{T^2} + {\frac{16H^2\iota^2}{M}}  \right)} \mathbf{1}   + \frac{H^2(1+\lambda_k)\iota}{TM}\mathbf{1} +  \sqrt{\frac{3H^3}{TM}} \mathbf{1} . \\
\end{align*}
Lastly, combining the upper bounds for Term (i) and Term (iii), we have
\begin{align*}
\vkb - \bvkb &\leq \sqrt{\frac{3H^4(1+\lambda_k)^2}{TM}  \left( \frac{4}{T^2} + {\frac{16H^2\iota^2}{M}}  \right)} \mathbf{1}   + \frac{H^2(1+\lambda_k)\iota}{TM}\mathbf{1} +  \sqrt{\frac{3H^3}{TM}} \mathbf{1} + \frac{3H^2(1+\lambda_k)}{T} \mathbf{1}\\
&\leq \sqrt{\frac{3H^4(1+\lambda_k)^2}{TM}  \left( \frac{4}{T^2} + {\frac{16H^2\iota^2}{M}}  \right)} \mathbf{1}  +  \sqrt{\frac{3H^3}{TM}} \mathbf{1} + \frac{4H^2(1+\lambda_k)}{T} \mathbf{1}.
\end{align*}
\end{proof}

\begin{lemma}
\label{lem_B2b}
Let $\pi_{T}$ be the output policy, and let $\bvkd$ denote the averaged empirical value function in Algorithm \ref{alg_MDVICMDP} when run with $\lambda_k$. For any $k \in [K]$ and $T\geq 2\log(T)/\gamma$, we have
\begin{align*}
 \bvkr(\rho) + \lambda_k \bvkc(\rho) - \vktr(\rho) - \lambda_k \vktc(\rho) \leq \frac{3H^2(1+\lambda_k)}{T} + \sqrt{\frac{3H^3}{TM}} + \sqrt{\frac{3H^6(1+\lambda_k)^2}{TM}\left( \frac{50}{T^2} +\frac{4\iota^2}{M} \right)} 
\end{align*}
with probability at least $1-\delta$.
\end{lemma}

\begin{proof}
Similarly as before, we have
\begin{align*}
\bvkb - \vktb &\leq \mathcal{H}_{\pi_{T}} P( \bar{V}^{T-1}_{\Box} - \bvkb ) +  \mathcal{H}_{\pi_{T}} \frac{1}{T} \sum_{i=0}^{T-1} \left[   \hat{P}_i \hvkib  - P \hvkib  \right] \tag{By \cref{eqn_boxdec2}} \\
&\leq \frac{2H^2(1+\lambda_k)}{T} \mathbf{1} +  \mathcal{H}_{\pi_{T}} \frac{1}{T} \sum_{i=0}^{T-1} \left[   \hat{P}_i \hvkib  - P \hvkib  \right] \tag{By Lemma \ref{lem_vmv}}\\
&\leq \frac{2H^2(1+\lambda_k)}{T} \mathbf{1} +  \mathcal{H}_{\pi_{T}} \sqrt{\frac{3H^2(1+\lambda_k)^2}{TM}  \left( \frac{4}{T^2} + {\frac{4H^2\iota^2}{M}}  \right)} \mathbf{1} \\
& \indent + \mathcal{H}_{\pi_{T}} \frac{H(1+\lambda_k)\iota}{TM}\mathbf{1} +  \mathcal{H}_{\pi_{T}} \sqrt{\frac{3}{TM}} \sigma(\vkb) \tag{By Lemma \ref{lem_bern5}}\\
&\leq \frac{2H^2(1+\lambda_k)}{T} \mathbf{1} +  \sqrt{\frac{3H^4(1+\lambda_k)^2}{TM}  \left( \frac{4}{T^2} + {\frac{4H^2\iota^2}{M}}  \right)} \mathbf{1} \\
& \indent + \frac{H^2(1+\lambda_k)\iota}{TM}\mathbf{1} +  \mathcal{H}_{\pi_{T}} \sqrt{\frac{3}{TM}} \sigma(\vkb) \\
&\leq \frac{3H^2(1+\lambda_k)}{T} \mathbf{1} +  \sqrt{\frac{3H^4(1+\lambda_k)^2}{TM}  \left( \frac{4}{T^2} + {\frac{4H^2\iota^2}{M}}  \right)} \mathbf{1} + \mathcal{H}_{\pi_{T}} \sqrt{\frac{3}{TM}} \sigma(\vkb).
\end{align*}
Now, it remains to bound the last term. We first observe that
\begin{align*}
\sigma(\vkb) &\leq \left( \vkb - \vktb \right) + \sigma \left( \vktb \right) \tag{By Lemma \ref{lem_VX_VXY}}\\
&\leq |\vkb - \vktb| + \sigma \left( \vktb \right) \tag{By Lemma \ref{lem_Popoviciu}}\\
&\leq \frac{5H^2(1+\lambda_k)}{T}\mathbf{1} + 4H^2(1+\lambda_k) \sqrt{\frac{\iota}{TM}}\mathbf{1} + \sigma \left(\vktb \right) \tag{By combining Lemma \ref{lem_conMDP} and Lemma \ref{lem_B2}}
\end{align*}
Therefore,
\begin{align*}
\mathcal{H}_{\pi_{T}} \sqrt{\frac{3}{TM}} \sigma(\vkb) &\leq \mathcal{H}_{\pi_{T}} \sqrt{\frac{3}{TM}} \left( \frac{5H^2(1+\lambda_k)}{T} + 4H^2(1+\lambda_k) \sqrt{\frac{\iota}{TM}}\  \right)\mathbf{1} + \sqrt{\frac{3}{TM}} \mathcal{H}_{\pi_{T}} \sigma \left(\vktb \right) \\
&\leq \sqrt{\frac{3H^2}{TM}} \left( \frac{5H^2(1+\lambda_k)}{T} + 4H^2(1+\lambda_k) \sqrt{\frac{\iota}{TM}}\  \right)\mathbf{1} + \sqrt{\frac{3}{TM}} \mathcal{H}_{\pi_{T}} \sigma \left(\vktb \right)  \\
&\leq \sqrt{\frac{3H^2}{TM}} \left( \frac{5H^2(1+\lambda_k)}{T} + 4H^2(1+\lambda_k) \sqrt{\frac{\iota}{TM}}\  \right)\mathbf{1} + \sqrt{\frac{3H^3}{TM}} \mathbf{1}  \tag{By Lemma \ref{lem_TVL}}\\
&= \sqrt{\frac{3H^6(1+\lambda_k)^2}{TM}} \left( \frac{5}{T} + \sqrt{\frac{16\iota}{TM}}\  \right)\mathbf{1} + \sqrt{\frac{3H^3}{TM}}\mathbf{1}  .
\end{align*}
By combining the above results and consolidating like terms, we conclude the proof.
\end{proof}

\subsubsection{Auxiliary Lemmas}

\begin{lemma}
\label{lem_rangevq}
Denote $\Box = r+\lambda_k c$. For any $k \in [K]$ and any $t \in [T]$, $\hat{Q}^{t}_{\Box}(s,a)$ and $\hat{V}^{t}_{\Box}(s)$ are bounded by $(1+\lambda_k)H$.
\end{lemma}

\begin{proof}
We prove it by induction. By initialization, $\hat{Q}^{1}_{\Box}(s,a) = r(s,a) + \lambda_k c(s,a) \leq 1+\lambda_k$ and $\hat{V}^{1}_{\Box}(s) =  \hat{Q}^{1}_{\Box}(s, \pi_{1}(\cdot|s)) \leq 1+\lambda_k \leq (1+\lambda_k)H$. Now, suppose $\hat{V}^{t-1}_{\Box}(s)$ is bounded by $(1+\lambda_k)H$ for some $t \geq 1$. We have
\begin{align*}
\hat{V}^{t}_{\Box} &=  \sum_{i=0}^{t} (\pi_{t}\hat{Q}^{i}_{\Box}) -  \sum_{i=0}^{t-1} (\pi_{t-1}\hat{Q}^{i}_{\Box}) \tag{From the last line in Algorithm \ref{alg_MDVICMDP}} \\
&\leq  \sum_{i=0}^{t} (\pi_{t}\hat{Q}^{i}_{\Box}) -  \sum_{i=0}^{t-1} (\pi_{t}\hat{Q}^{i}_{\Box}) \tag{By the greediness of $\pi_{t-1}$}\\
&= (\pi_{t}\hat{Q}^{t-1}_{\Box}) \\
&\leq (\pi_{t}\Box) + \gamma (\hat{P}^{t-1}_{\pi_t}\hat{V}^{t-1}_{\Box}) \tag{From the second last line in Algorithm \ref{alg_MDVICMDP}}\\
&\leq (1 +\lambda_k + \gamma (1 +\lambda_k) H)\mathbf{1} \tag{Induction hypothesis}\\
&= (1 +\lambda_k) H\mathbf{1} \tag{$1 +  \frac{\gamma}{1-\gamma} = \frac{1}{1-\gamma}$}.
\end{align*}
Therefore, $\hat{V}^{t}_{\Box}(s)$ is bounded by $(1+\lambda_k)H$. As a consequence, $\hat{Q}^{t}_{\Box}(s,a)$ is also bounded by $(1+\lambda_k)H$.
\end{proof}

\begin{lemma}
\label{lem_televq}
For any $k \in [K]$ and $t\in [T]$, we have
\begin{align*}
\bvktd := \frac{1}{t} \sum_{i=1}^{t}  \hvkid = \frac{1}{t} \sum_{i=0}^{t-1} (\pi_{t} \hqkid).
\end{align*}
\end{lemma}

\begin{proof}
\begin{align*}
\bvktd &:= \frac{1}{t} \sum_{i=1}^{t}  \hvkid \\
&= \frac{1}{t} \sum_{i=0}^{t-1}  \left( \sum_{j=0}^{i}(\pi_{i+1} \hqktd{j}) - \sum_{j=0}^{i-1} (\pi_{i}\hqktd{j} )\right)  \tag{From the last line in Algorithm \ref{alg_MDVICMDP}}\\
&= \frac{1}{t} \sum_{i=0}^{t-1} (\pi_{t} \hqkid) \tag{Due to telescoping sum}.
\end{align*}
\end{proof}

\begin{lemma}
\label{lem_sumQ}
For any $k \in [K]$ and $t\in [T]$, we have
\begin{align*}
\sum_{i=0}^{t-1} \hqkid =  t\diamond + \gamma P  \sum_{i=0}^{t-2}  (\pi_{t-1} \hqkid)  + \sum_{i=0}^{t-1} \left[\gamma \hat{P}_i \hvkid - \gamma P \hvkid  \right],
\end{align*}
and
\begin{align*}
\sum_{i=0}^{t-1} \hqkib =  t (r + \lambda_k c) + \gamma P  \sum_{i=0}^{t-2}  (\pi_{t-1} \hqkib)  + \sum_{i=0}^{t-1} \left[\gamma \hat{P}^i_{\pi_t} \hvkib - \gamma P_{\pi_t} \hvkib  \right].
\end{align*}
\end{lemma}

\begin{proof}
We prove the first equality. The second equality follows by linearity.
\begin{align*}
\sum_{i=0}^{t-1} \hqkid &= \sum_{i=0}^{t-1} \left[ \diamond + \gamma \hat{P}_i \hvkid \right] \tag{From the second last line in Algorithm \ref{alg_MDVICMDP}}  \\
&= \sum_{i=0}^{t-1} \left[ \diamond + \gamma P \hvkid - \gamma P \hvkid + \gamma \hat{P}_i \hvkid \right] \nonumber \\
&= t\diamond + \gamma P \sum_{i=0}^{t-1}   \hvkid  + \sum_{i=0}^{t-1} \left[\gamma \hat{P}_i \hvkid - \gamma P \hvkid  \right] \nonumber \\
&= t\diamond + \gamma P  \sum_{i=0}^{t-2} (\pi_{t-1}  \hqkid)  + \sum_{i=0}^{t-1} \left[\gamma \hat{P}_i \hvkid - \gamma P \hvkid  \right]. \tag{From Lemma \ref{lem_televq}}
\end{align*}
\end{proof}

\begin{lemma}
\label{lem_vmv}
For any $k \in [K]$,
\begin{align*}
\|\bvkb - \bvkttb{T-1}\|_{\infty} &\leq \frac{2H(1+\lambda_k)}{T}.
\end{align*}
\end{lemma}

\begin{proof}
We present the proof for the case of $\bvkb - \bvkttb{T-1}$. The proof for the another case is similar. By the definition of $\bvktb$ and due to the greediness of $\pi_{T-1}$, we have
\begin{align*}
\bvkb - \bvkttb{T-1} &=   \frac{1}{T} \sum_{i=0}^{T-1} (\pi_{T}\hqktb{i}) - \frac{1}{T-1}\sum_{i=0}^{T-2} (\pi_{T-1}  \hqktb{i}) \nonumber \\
&\leq  \frac{1}{T} \sum_{i=0}^{T-1} (\pi_{T} \hqktb{i}) - \frac{1}{T-1}\sum_{i=0}^{T-2} (\pi_{T} \hqktb{i} )\\
&\leq   \frac{1}{T} \sum_{i=0}^{T-1} (\pi_{T} \hqktb{i}) - \frac{1}{T}\sum_{i=0}^{T-2} (\pi_{T} \hqktb{i}) \\
&\leq \frac{1}{T} (\pi_{T} \hqktb{T-1})\\
&\leq \frac{2H(1+\lambda_k)}{T} \mathbf{1}. \tag{By Lemma \ref{lem_rangevq}}
\end{align*}

\end{proof}

\subsection{Proof of Lemma \ref{lem_concen} (Optimality Guarantees for Algorithm \ref{alg_CVI} - Tabular CMDP)}
\label{sec_concV}

\OPGuaranteeb*

\begin{proof} 
By Lemma \ref{lem_B4bern1}, we have
\begin{align*}
\left\| \bmvkd - V^{\pi}_{\diamond} \right\|_{\infty} \leq \tilde{O}\left( \frac{H^2}{T} + \frac{H}{tM} + \sqrt{\frac{H^4}{tM^2}}  + \sqrt{\frac{H^3}{tM}} \right)
\end{align*}
with probability at least $1-\delta$. By letting $M=\tilde{O}\left( \frac{H}{\epsilon} \right)$, $T=\tilde{O}\left( \frac{H^2}{\epsilon} \right)$, and $\epsilon \in (0,H]$, we have 
\begin{align*}
\left\| \bmvkd - V^{\pi}_{\diamond} \right\|_{\infty} \leq {O}\left( \epsilon + \frac{\epsilon^2}{H^2} + \epsilon  + \epsilon \right) \leq {O}\left( \epsilon  \right)
\end{align*}
with total sample complexity $N = TM|\mathcal{C}| = \tilde{O}\left( \frac{|\mathcal{S}||\mathcal{A}| H^3}{\epsilon^2} \right)$ and probability at least $1-\delta$.
\end{proof} 

\subsubsection{Auxiliary Lemmas}

Since Algorithm \ref{alg_CVI} is equivalent to running Algorithm \ref{alg_MDVICMDP} with a fixed policy, the following lemma follows directly from Lemma~\ref{lem_rangevq}.

\begin{lemma}
\label{lem_rangevq2}
For any $t \in [T]$, $\hat{\mathcal{Q}}^{t}_{\diamond}(s,a)$ and $\hat{\mathcal{V}}^{t}_{\diamond}(s)$ are bounded by $H$.
\end{lemma}

\begin{lemma}
\label{lem_B4}
$\bmvkd - V^{\pi}_{\diamond} \leq \mathcal{H}_{\pi} P(\bar{\mathcal{V}}^{T-1}_{\diamond} - \bmvkd ) +  \mathcal{H}_{\pi} \frac{1}{T} \sum_{i=0}^{T-1} \left[   \hat{P}_i \hmvkid  - P \hmvkid  \right].$
\end{lemma}

\begin{proof} 
First, we notice that
\begin{align}
\sum_{i=0}^{t-1} \hmqkid &= \sum_{i=0}^{t-1} \left[ \diamond + \gamma \hat{P}_i \hmvkid \right] \nonumber \\
&= \sum_{i=0}^{t-1} \left[ \diamond + \gamma P \hmvkid - \gamma P \hmvkid + \gamma \hat{P}_i \hmvkid \right] \nonumber \\
&= t\diamond + \gamma P \sum_{i=0}^{t-1}   \hmvkid  + \sum_{i=0}^{t-1} \left[\gamma \hat{P}_i \hmvkid - \gamma P \hmvkid  \right] \nonumber \\
&= t\diamond + \gamma P \sum_{i=0}^{t-2} ( \pi  \hmqkid ) + \sum_{i=0}^{t-1} \left[\gamma \hat{P}_i \hmvkid - \gamma P \hmvkid  \right]. \label{eqn_newlemsumQ}
\end{align}
It is different from Lemma \ref{lem_sumQ} because the policy is now fixed at each iteration. The rest of the proof follows the same set of steps as in the proof of Lemma \ref{lem_conMDP}.
Denote $\diamond = r$ or $c$. Since $(I-\gamma P_{\pi})V^{\pi}_{\diamond} = \pi \diamond $, we have
\begin{align}
(I-\gamma P_{\pi})(\bmvkd - V^{\pi}_{\diamond}) &= (\bmvkd - \gamma P_{\pi} \bmvkd) - \pi \diamond  \nonumber \\
&= \bmvkd - (\pi\diamond) + \gamma P_{\pi} \bmvkd  \nonumber\\
\Longrightarrow  \bmvkd - V^{\pi}_{\diamond}  &=  (I-\gamma P_{\pi})^{-1} ( \bmvkd -  (\pi\diamond) + \gamma P_{\pi} \bmvkd )  \label{eqn_piout3}
\end{align}
Recall that $\bmvktd = \frac{1}{t} \sum_{i=1}^{t}  \hmvkid = \pi\frac{1}{t} \sum_{i=0}^{t-1} \hmqkid $ for all $t\in[T]$.
Now, we have
\begin{align*}
\bmvkd - V^{\pi}_{\diamond}  &=  (I-\gamma P_{\pi})^{-1} (\bmvkd -  (\pi\diamond) + \gamma P_{\pi} \bmvkd)  \tag{By \cref{eqn_piout3}} \\
&= (I-\gamma P_{\pi})^{-1} \left( \frac{1}{T} \sum_{i=0}^{T-1} (\pi \hmqkid)  -  (\pi\diamond) + \gamma P_{\pi} \bmvkd \right) \\
&= (I-\gamma P_{\pi})^{-1}  \left(  (\pi \diamond) + \gamma P_{\pi} \frac{1}{T} \sum_{i=0}^{T-2}  (\pi \hmqkid)  + \frac{1}{T} \sum_{i=0}^{T-1} \left[\gamma \hat{P}_{\pi}^i \hmvkid - \gamma P_{\pi} \hmvkid  \right] - (\pi \diamond) - \gamma P_{\pi} \bmvkd \right) \tag{By \cref{eqn_newlemsumQ}}\\
&= (I-\gamma P_{\pi})^{-1} \left( \gamma P_{\pi} \frac{1}{T} \sum_{i=0}^{T-2}  (\pi \hmqkid)  + \frac{1}{T} \sum_{i=0}^{T-1} \left[\gamma \hat{P}_{\pi}^i \hmvkid - \gamma P_{\pi} \hmvkid  \right]  - \gamma P_{\pi} \bmvkd \right)\\
&\leq (I-\gamma P_{\pi})^{-1}  \left( \gamma P_{\pi}  \frac{1}{T-1} \sum_{i=0}^{T-2} (\pi  \hmqkid)  + \frac{1}{T} \sum_{i=0}^{T-1} \left[\gamma \hat{P}_{\pi}^i \hmvkid - \gamma P_{\pi} \hmvkid  \right]  - \gamma P_{\pi} \bmvkd \right).
\end{align*}
We note that $ \frac{1}{T-1} \sum_{i=0}^{T-2} (\pi  \hmqkid) = \bar{\mathcal{V}}^{T-1}_{\diamond}$, as it is an equivalent result of Lemma \ref{lem_televq} with a fixed policy. By letting $\mathcal{H}_{\pi} = \gamma (I-\gamma P_{\pi})^{-1} \pi $, we obtain
\begin{align}
 \bmvkd - V^{\pi}_{\diamond} \leq \mathcal{H}_{\pi} P(\bar{\mathcal{V}}^{T-1}_{\diamond} - \bmvkd ) +  \mathcal{H}_{\pi} \frac{1}{T} \sum_{i=0}^{T-1} \left[   \hat{P}_i \hmvkid  - P \hmvkid  \right] \label{eqn_mbernd}.
\end{align}

\end{proof}

\begin{lemma}
\label{lem_B4bern1}
We have
\begin{align*}
\left\| \bmvkd - V^{\pi}_{\diamond} \right\|_{\infty} \leq \tilde{O}\left( \frac{H^2}{T} + \frac{H}{tM} + \sqrt{\frac{H^4}{tM^2}}  + \sqrt{\frac{H^3}{tM}} \right)
\end{align*}
with probability at least $1-\delta$.
\end{lemma}

\begin{proof}
By Lemma \ref{lem_B4}, we have
\begin{align*}
\bmvkd - V^{\pi}_{\diamond} &\leq \mathcal{H}_{\pi} P( \bar{\mathcal{V}}^{T-1}_{\diamond} - \bmvkd ) +  \mathcal{H}_{\pi} \frac{1}{T} \sum_{i=0}^{T-1} \left[   \hat{P}_i \hmvkid  - P \hmvkid  \right] \\
&\leq \tilde{O}\left( \frac{H^2}{T} \right) \mathbf{1} + \mathcal{H}_{\pi} \frac{1}{T} \sum_{i=0}^{T-1} \left[   \hat{P}_i \hmvkid  - P \hmvkid  \right] \tag{By Lemma \ref{lem_mvmv}}.
\end{align*}
Thus, it remains to bound the second term. By Lemma \ref{lem_bern6} we have
\begin{align*}
\frac{1}{t} \sum_{i=0}^{t-1} \left[   \hat{P}_i \hmvkid  - P \hmvkid  \right] (s,a) \leq \frac{H\iota}{tM} + \sqrt{Z}
\end{align*}
where 
\begin{align*}
Z:=\frac{1}{tM}  \left( \frac{H^4}{M}  +  \mathrm{Var}(V^{\pi}_{\diamond}(s,a)) \right)
\end{align*}
with probability at least $1-\delta$. Therefore,
\begin{align*}
\bmvkd - V^{\pi}_{\diamond} &\leq \tilde{O}\left( \frac{H^2}{T} \right) \mathbf{1} + \frac{H\iota}{tM}\mathbf{1} + \sqrt{\frac{H^4}{tM^2}}\mathbf{1}  + \sqrt{\frac{1}{tM}}\mathcal{H}_{\pi}  \sigma(V^{\pi}_{\diamond}) \\
&\leq \tilde{O}\left( \frac{H^2}{T} \right) \mathbf{1} + \frac{H\iota}{tM}\mathbf{1} + \sqrt{\frac{H^4}{tM^2}}\mathbf{1}  + \sqrt{\frac{H^3}{tM}}\mathbf{1}  \tag{By Lemma \ref{lem_TVL}}\\
&\leq \tilde{O}\left( \frac{H^2}{T} + \frac{H\iota}{tM} + \sqrt{\frac{H^4}{tM^2}}  + \sqrt{\frac{H^3}{tM}} \right) \mathbf{1}
\end{align*}
which completes the proof.
\end{proof}

\begin{lemma}
\label{lem_mvmv}
$\|\bmvkd - \bmvkttd{T-1}\|_{\infty} \leq  \frac{H}{T}.$
\end{lemma}

\begin{proof}
Similar to the proof of Lemma \ref{lem_vmv}, we have
\begin{align*}
\bmvkd - \bmvkttd{T-1} &=   \frac{1}{T} \sum_{i=0}^{T-1} (\pi\hmqktd{i}) -  \frac{1}{T-1}\sum_{i=0}^{T-2}(\pi \hmqktd{i}) \nonumber \\
&=   \frac{1}{T} \sum_{i=0}^{T-1}(\pi \hmqktd{i}) - \frac{1}{T-1}\sum_{i=0}^{T-2} (\pi \hmqktd{i}) \\
&\leq  \frac{1}{T} \sum_{i=0}^{T-1} (\pi \hmqktd{i}) - \frac{1}{T}\sum_{i=0}^{T-2} (\pi \hmqktd{i} ) \\
&\leq  \frac{1}{T} (\pi \hmqktd{T-1}) \\
&\leq \frac{H}{T}\mathbf{1}. \tag{By Lemma \ref{lem_rangevq2}}
\end{align*}
\end{proof}

\subsection{Proof of Lemma \ref{lem_bern5} and Lemma \ref{lem_bern6} (Concentration Error Bounds with Bernstein's Inequality - Tabular CMDP)}
\label{Appendix_Bern}

All the proofs presented in this section are adapted from the proofs for Lemmas 5 to 8 in \cite{kozuno2022kl}, with substantial modifications to suit our setting.

\begin{lemma}
\label{lem_bern5}
For any $t \geq 2\log(t)/\gamma$ and $k \in [K]$, we have
\begin{align*}
\frac{1}{t}\sum_{i=1}^t \left[ \hat{P}_i \hvkttb{i} -  P \hvkttb{i}  \right] (s,a) \leq \frac{H(1+\lambda_k)\iota}{tM} + \sqrt{Z}
\end{align*}
where 
\begin{align*}
Z:=\frac{3H^2(1+\lambda_k)^2}{tM} \left( \frac{4}{t^2} + {\frac{16H^2\iota^2}{M}} \right) + \frac{3\mathrm{Var}(\vkb(s,a))}{tM}
\end{align*}
with probability at least $1-\delta$.
\end{lemma}

\begin{proof}
We have
\begin{align*}
\frac{1}{t}\sum_{i=1}^t \left[ \hat{P}_i \hvkttb{i} -  P \hvkttb{i}  \right](s,a) &= \frac{1}{t}\sum_{i=1}^t \frac{1}{M}\sum_{m=1}^M \left[ \hvkttb{i}(y_{i,m,s,a}) -  (P \hvkttb{i})(s,a)  \right]\\
&= \frac{1}{tM}\sum_{i=1}^t \sum_{m=1}^M \left[ \hvkttb{i}(y_{i,m,s,a}) -  (P \hvkttb{i})(s,a)  \right].
\end{align*}
The above is a sum of bounded martingale differences with respect to the filtraion $(\mathcal{F})_{i=1,m=1}^{t,M}$. Let $X_{i,m} = \frac{1}{tM} \left( \hvkttb{i}(y_{i,m,s,a}) -  (P \hvkttb{i})(s,a)  \right)$. It can be noted that $X_{i,m} \leq \frac{H(1+\lambda_k)}{tM}$ (by Lemma \ref{lem_rangevq}) and $\mathbb{E}[X_{i,m}]=0$. Next, we bound $Z^{\prime}$ as defined in Lemma \ref{lem_cond_Bernstein}
\begin{align*}
Z^{\prime} &=  \sum_{i=1}^t \sum_{m=1}^M \mathbb{E}\left[ X_{i,m}^2 \right]\\
&= \sum_{i=1}^t \sum_{m=1}^M \mathbb{E}\left[ \frac{1}{t^2M^2} \left( \hvkttb{i}(y_{i,m,s,a}) -  (P \hvkttb{i})(s,a)  \right)^2 \right]\\
&= \frac{1}{t^2M^2} \sum_{i=1}^t \sum_{m=1}^M \mathrm{Var}(\hvkttb{i}(y_{i,m,s,a}))\\
&\leq \frac{3}{t^2M^2} \sum_{i=1}^t \sum_{m=1}^M \left( \frac{4H^2(1+\lambda_k)^2}{t^2} + {\frac{16H^4(1+\lambda_k)^2\iota^2}{M}} +  \mathrm{Var}(\vkb(s,a)) \right)  \tag{By Lemma \ref{lem_bern4}}\\
&= \frac{3}{tM} \left( \frac{4H^2(1+\lambda_k)^2}{t^2} + {\frac{16H^4(1+\lambda_k)^2\iota^2}{M}} +  \mathrm{Var}(\vkb(s,a)) \right)  \\
&:= Z.
\end{align*}
By letting $U=H(1+\lambda_k)$ in Lemma \ref{lem_cond_Bernstein}, we have
\begin{align*}
\frac{1}{t}\sum_{i=1}^t \left[ \hat{P}_i \hvkttb{i} -  P \hvkttb{i}  \right]  (s,a) \leq \frac{H(1+\lambda_k)\iota}{tM} + \sqrt{Z}
\end{align*}
with probability at least $1-\delta$.
\end{proof}

\begin{lemma}
\label{lem_bern6}
For any $t \geq 2\log(t)/\gamma$, we have
\begin{align*}
\frac{1}{t} \sum_{i=0}^{t-1} \left[   \hat{P}_i \hmvkid  - P \hmvkid  \right] (s,a) \leq \frac{H\iota}{tM} + \sqrt{Z}
\end{align*}
where 
\begin{align*}
Z:=\frac{1}{tM}  \left( \frac{H^4}{M}  +  \mathrm{Var}(V^{\pi}_{\diamond}(s,a)) \right)
\end{align*}
with probability at least $1-\delta$.
\end{lemma}

\begin{proof}
We have
\begin{align*}
\frac{1}{t} \sum_{i=0}^{t-1} \left[   \hat{P}_i \hmvkid  - P \hmvkid  \right] (s,a) &= \frac{1}{t}\sum_{i=1}^t \frac{1}{M}\sum_{m=1}^M \left[ \hmvkid(y_{i,m,s,a}) -  (P \hmvkid)(s,a)  \right]\\
&= \frac{1}{tM}\sum_{i=1}^t \sum_{m=1}^M \left[ \hmvkid(y_{i,m,s,a}) -  (P \hmvkid)(s,a)  \right].
\end{align*}
The above is a sum of bounded martingale differences with respect to the filtraion $(\mathcal{F})_{i=1,m=1}^{t,M}$. Let $X_{i,m} = \frac{1}{tM} \left( \hmvkid(y_{i,m,s,a}) -  (P \hmvkid)(s,a)  \right)$. It can be noted that $X_{i,m} \leq \frac{H}{tM}$ and $\mathbb{E}[X_{i,m}]=0$. Next, we bound $Z^{\prime}$ as defined in Lemma \ref{lem_cond_Bernstein}
\begin{align*}
Z^{\prime} &=  \sum_{i=1}^t \sum_{m=1}^M \mathbb{E}\left[ X_{i,m}^2 \right]\\
&= \sum_{i=1}^t \sum_{m=1}^M \mathbb{E}\left[ \frac{1}{t^2M^2} \left( \hmvkid(y_{i,m,s,a}) -  (P \hmvkid)(s,a)  \right)^2 \right]\\
&= \frac{1}{t^2M^2} \sum_{i=1}^t \sum_{m=1}^M \mathrm{Var}(\hmvkid(y_{i,m,s,a}))\\
&\leq  \frac{1}{t^2M^2} \sum_{i=1}^t \sum_{m=1}^M \left( \frac{H^4}{M}  +  \mathrm{Var}(V^{\pi}_{\diamond}(s,a)) \right)\tag{By Lemma \ref{lem_varl}} \\
&= \frac{1}{tM}  \left( \frac{H^4}{M}  +  \mathrm{Var}(V^{\pi}_{\diamond}(s,a)) \right) \\
&:= Z
\end{align*}
with probability at least $1-\delta$.
Taking the union bound over $(s,a,i) \in \mathcal{S} \times \mathcal{A} \times [t]$ and by Lemma \ref{lem_cond_Bernstein}, we have
\begin{align*}
\frac{1}{t} \sum_{i=0}^{t-1} \left[   \hat{P}_i \hmvkid  - P \hmvkid  \right] (s,a) \leq \frac{H\iota}{tM} + \sqrt{Z}
\end{align*}
with probability at least $1-\delta$.
\end{proof}

\subsubsection{Auxiliary Lemmas for Lemma \ref{lem_bern5}}

\begin{lemma}
\label{lem_bern1}
For any $t\in [T]$ and $k \in [K]$,
\begin{align*}
\mathbf{0} \leq \vkb - \svktb{t} \leq \sum_{i=1}^{t} \left(   \prod_{j=1}^{t-i} [\gamma P_{\pi_{t-j}}]\pi_i - (\gamma P_{\pi^*_k})^{t-i}\pi^*_k \right)\frac{1}{i} \sum_{j=0}^{i-1} \left[\gamma \hat{P}_j \hvkttb{j} - \gamma P \hvkttb{j}  \right] + \frac{H(1+\lambda_k)}{t(t-1)} \mathbf{1}.
\end{align*}
\end{lemma}

\begin{proof}
The first inequality is due to the definition of $\pi^*_k$.
For the second inequality, since we have $\vkb - \svktb{t} = \underbrace{\vkb -  \frac{1}{t}\sum_{i=1}^{t} (\pi_{t} \hqktb{i})}_{\text{Term (i)}} + \underbrace{ \frac{1}{t}\sum_{i=1}^{t}(\pi_{t} \hqktb{i}) - \svktb{t}}_{\text{Term (ii)}}$, we first bound term (i)
\begin{align*}
\vkb - \pi_{t} \frac{1}{t}\sum_{i=1}^{t} \hqktb{i} &\leq   (\pi^*_kQ^{\pi^*_k}_{\Box}) - \frac{1}{t}\sum_{i=1}^{t} (\pi^*_k \hqktb{i})  \tag{By the greediness of $\pi_{t}$}\\
&=   (\pi^*_k\Box) + \gamma P_{\pi^*_k} \vkb - \frac{1}{t}\sum_{i=1}^{t} (\pi^*_k\hqktb{i})\\
&=  (\pi^*_k  \Box) + \gamma P_{\pi^*_k } \vkb - (\pi^*_k \Box) - \gamma P_{\pi^*_k }  \frac{1}{t} \sum_{i=0}^{t-2}   (\pi_{t-1}\hqkib )- \frac{1}{t} \sum_{i=0}^{t-1} \left[\gamma \hat{P}^i_{\pi^*_k} \hvkib - \gamma P_{\pi^*_k} \hvkib  \right]   \tag{By Lemma \ref{lem_sumQ}}\\
&=  \gamma P_{\pi^*_k}  \left( \vkb  - \frac{1}{t} \sum_{i=0}^{t-2}  (\pi_{t-1}  \hqkib) \right) - \frac{1}{t} \sum_{i=0}^{t-1} \left[\gamma \hat{P}^i_{\pi^*_k} \hvkib - \gamma P_{\pi^*_k} \hvkib  \right] \\
&=   \gamma P_{\pi^*_k} \left( \vkb  -  \frac{1}{t} \sum_{i=0}^{t-2}  (\pi_{t-1} \hqkib) +  \frac{1}{t-1} \sum_{i=0}^{t-2}   (\pi_{t-1}\hqkib) -  \frac{1}{t-1} \sum_{i=0}^{t-2}   (\pi_{t-1}\hqkib) \right)  \\
& \indent  - \frac{1}{t} \sum_{i=0}^{t-1} \left[\gamma \hat{P}^i_{\pi^*_k} \hvkib - \gamma P_{\pi^*_k} \hvkib  \right] \\
&\leq   \gamma P_{\pi^*_k} \left( \vkb  -  \frac{1}{t-1} \sum_{i=0}^{t-2}   (\pi_{t-1} \hqkib ) \right) - \frac{1}{t} \sum_{i=0}^{t-1} \left[\gamma \hat{P}^i_{\pi^*_k} \hvkib - \gamma P_{\pi^*_k} \hvkib  \right] + \frac{H(1+\lambda_k)}{t(t-1)} \mathbf{1} \tag{$\|\| \gamma\pi_k^* P\|_1 \| \pi_{t-1} \hat{Q}_{\Box}^{i} \|_{\infty} \leq H(1+\lambda_k)$ for all $k$ and $i$} \\
&\leq -\sum_{i=1}^{t} (\gamma P_{\pi^*_k})^{t-i}  \frac{1}{i} \sum_{j=0}^{i-1} \left[\gamma \hat{P}^j_{\pi^*_k} \hvkttb{j} - \gamma P_{\pi^*_k} \hvkttb{j}  \right] + \frac{H(1+\lambda_k)}{t(t-1)} \mathbf{1} \tag{By induction (Lemma \ref{lem_inductionlem}) and $\hvkttb{0}=\mathbf{0}$}.
\end{align*}
Next, we bound term (ii). We define $Q^{\pi}$ the Q-value function for a policy $\pi$ being its unique fixed point.
\begin{align*}
 \frac{1}{t}\sum_{i=1}^{t} (\pi_{t}\hqktb{i}) - \svktb{t} &\leq \frac{1}{t}\sum_{i=1}^{t} (\pi_{t}\hqktb{i}) - \pi_{t}\prod_{i=1}^{t-1}\mathcal{T}^{\pi_{i}} Q_{\Box}^{\pi_0} \tag{From the definition of $\pi_{t}$}\\
&=  \frac{1}{t}\sum_{i=1}^{t} (\pi_{t} \hqktb{i}) - (\pi_{t} \Box) - \gamma P_{\pi_{t}} \pi_{t-1} \prod_{i=1}^{t-2}\mathcal{T}^{\pi_{i}} Q_{\Box}^{\pi_0} \\
&=  (\pi_{t} \Box) + \gamma P_{\pi_{t} }  \frac{1}{t} \sum_{i=0}^{t-2}   (\pi_{t-1}\hqkib) + \frac{1}{t} \sum_{i=0}^{t-1} \left[\gamma \hat{P}^i_{\pi_t} \hvkib - \gamma P_{\pi_t} \hvkib  \right]  \\
& \quad- (\pi_{t}\Box) - \gamma P_{\pi_t} \pi_{t-1} \prod_{i=1}^{t-2}\mathcal{T}^{\pi_{i}} Q_{\Box}^{\pi_0}    \tag{By Lemma \ref{lem_sumQ}}\\
&=    \gamma P_{\pi_{t}} \left(   \frac{1}{t} \sum_{i=0}^{t-2}   (\pi_{t-1} \hqkib) - \svktb{t-1} \right) + \frac{1}{t} \sum_{i=0}^{t-1} \left[\gamma \hat{P}^i_{\pi_t} \hvkib - \gamma P_{\pi_t} \hvkib  \right] \\
&\leq  \gamma P_{\pi_{t}} \left(  \frac{1}{t-1} \sum_{i=0}^{t-2}   (\pi_{t-1}\hqkib) - \svktb{t-1} \right) + \frac{1}{t} \sum_{i=0}^{t-1} \left[\gamma \hat{P}^i_{\pi_t} \hvkib - \gamma P_{\pi_t} \hvkib  \right] \\
&\leq \sum_{i=1}^{t} \prod_{j=1}^{t-i} [\gamma P_{\pi_{t-j}}]  \frac{1}{i} \sum_{j=0}^{i-1} \left[\gamma \hat{P}^j_{\pi_i} \hvkttb{j} - \gamma P_{\pi_i} \hvkttb{j}  \right] \tag{By induction (Lemma \ref{lem_inductionlem}) and $\hvkttb{0}=\mathbf{0}$}.
\end{align*}
Thus, we obtain the second inequality.
\end{proof}

\begin{lemma}
\label{lem_bern2}
For any $t\in [T-1]$ and $k \in [K]$,
\begin{align*}
\hvkttb{t+1} \leq \svktb{t+1} + \gamma^{t+1} tH(1+\lambda_k) \mathbf{1} + \sum_{i=1}^{t} \gamma^i  \prod_{j=t-i+1}^{t} P_{\pi_{t-j}}  \gamma (\hat{P}_{\pi_{t-i}}^{t-i} \hvkttb{t-i} -  P_{\pi_{t-i}}\hvkttb{t-i}) 
\end{align*}
and
\begin{align*}
\hvkttb{t+1} \geq \svktb{t} - \gamma^{t+1} tH(1+\lambda_k) \mathbf{1} + \sum_{i=1}^{t} \gamma^i \prod_{j=t-i+1}^{t} P_{\pi_{t-j}}   \gamma (\hat{P}^{t-i}_{\pi_{t-i}} \hvkttb{t-i} - P_{\pi_{t-i}} \hvkttb{t-i}).
\end{align*}
\end{lemma}

\begin{proof}
We first note that
\begin{align*}
\hvkttb{t+1} &= \sum_{i=0}^{t} (\pi_{t+1} \hqktb{i}) -  \sum_{i=0}^{t-1} (\pi_{t} \hqktb{i}) \tag{From the last line in Algorithm \ref{alg_MDVICMDP}}\\
&\leq  \sum_{i=0}^{t} (\pi_{t+1} \hqktb{i}) -  \sum_{i=0}^{t-1} (\pi_{t+1} \hqktb{i} )  \tag{By the greediness of $\pi_{t}$}\\
&= (\pi_{t+1}  \hqktb{t}) \\
&=  (\pi_{t+1} \Box) + \gamma \hat{P}_{\pi_{t+1} }^{t} \hvkttb{t}\\
&=   (\pi_{t+1}\Box) + \gamma P_{\pi_{t+1}}\hvkttb{t} + \gamma (\hat{P}_{\pi_{t+1}}^{t}\hvkttb{t} -  P_{\pi_{t+1}} \hvkttb{t})\\
&\leq \sum_{i=1}^{t} \gamma^i  \prod_{j=t-i+1}^{t} P_{\pi_{t-j}} (\Box + \gamma (\hat{P}_{\pi_{t-i}}^{t-i} \hvkttb{t-i} - P_{\pi_{t-i}}  \hvkttb{t-i}) ).  \tag{By induction on $t$}
\end{align*}
Let $\mathcal{T}^{\pi}$ denote the Bellman operator with policy $\pi$, we have
\begin{align*}
\pi_{t+1}\prod_{i=0}^{t} \mathcal{T}^{\pi_{i}} Q^{\pi_{0}}_{\Box} &= \sum_{i=1}^{t} \gamma^i  \prod_{j=t-i-1}^{t} P_{\pi_{t-j}} (\pi_{i} \Box) + \gamma^{t+1} \prod_{j=0}^tP_{\pi_{t-j}} (\pi_{0}Q^{\pi_{0}}_{\Box})\\
\Longrightarrow \sum_{i=1}^{t} \gamma^i  \prod_{j=t-i-1}^{t} P_{\pi_{t-j}}(\pi_{i} \Box) &\leq \pi_{t+1} \prod_{i=0}^{t} \mathcal{T}^{\pi_{i}} Q^{\pi_{0}}_{\Box} + \gamma^{t+1} tH(1+\lambda_k) \mathbf{1}. \tag{Since $\prod_{j=0}^tP\pi_{j} Q^{\pi_{0}}_{\Box} \leq tH(1+\lambda_k) \mathbf{1}$}
\end{align*}
Combining all above, we obtain
\begin{align*}
\hvkttb{t+1} \leq \pi_{t+1} \prod_{i=0}^{t} \mathcal{T}^{\pi_{i}} Q^{\pi_{0}}_{\Box} + \gamma^{t+1} tH(1+\lambda_k) \mathbf{1} + \sum_{i=1}^{t} \gamma^i  \prod_{j=t-i+1}^{t} P_{\pi_{t-j}}\gamma (\hat{P}_{\pi_{t-i}}^{t-i}\hvkttb{t-i} - P_{\pi_{t-i}} \hvkttb{t-i}) 
\end{align*}
Denoting $\pi^{\prime}_{k,t}$ a non-stationary policy that follows $ \pi_{t+1}, \pi_{t}, \pi_{t-1},\dots$ sequentially, we simplify the above inequality as
\begin{align*}
\hvkttb{t+1} \leq \svktb{t+1} + \gamma^{t+1} tH(1+\lambda_k) \mathbf{1} + \sum_{i=1}^{t} \gamma^i  \prod_{j=t-i+1}^{t} P_{\pi_{t-j}}  \gamma (\hat{P}_{\pi_{t-i}}^{t-i} \hvkttb{t-i} - P_{\pi_{t-i}}  \hvkttb{t-i}).
\end{align*}
Similarly,
\begin{align*}
\hvkttb{t+1} &=  \sum_{i=1}^{t} (\pi_{t+1}\hqktb{i}) -  \sum_{i=1}^{t-1} (\pi_{t}\hqktb{i})\\
&\geq  \sum_{i=1}^{t} (\pi_{t} \hqktb{i}) -  \sum_{i=1}^{t-1} (\pi_{t} \hqktb{i})  \tag{By the greediness of $\pi_{t+1}$}\\
&= (\pi_{t}  \hqktb{t})\\
&= (\pi_{t}\Box) + \gamma \hat{P}_{\pi_{t}}^{t} \hvkttb{t}\\
&=   (\pi_{t}\Box) + \gamma P_{\pi_{t}}\hvkttb{t} + \gamma (\hat{P}^{t}_{\pi_{t}} \hvkttb{t} - P_{\pi_{t}} \hvkttb{t})\\
&\geq \sum_{i=1}^{t} \gamma^i \prod_{j=t-i+1}^{t} P_{\pi_{t-j}}  (\Box + \gamma (\hat{P}_{\pi_{t-i}}^{t-i}\hvkttb{t-i} -  P_{\pi_{t-i}} \hvkttb{t-i}) )  \tag{By induction on $t$}
\end{align*}
and
\begin{align*}
\pi_{t+1}\prod_{i=1}^{t+1} \mathcal{T}^{\pi_{i-1}} Q^{\pi_{0}}_{\Box} &= \sum_{i=1}^{t} \gamma^i  \prod_{j=t-i+1}^{t} P_{\pi_{t-j}} (\pi_{i} \Box) + \gamma^{t+1} \prod_{j=1}^t P_{\pi_{t-j+1}} (\pi_{0} Q^{\pi_{0}}_{\Box})\\
\Longrightarrow \sum_{i=1}^{t} \gamma^i  \prod_{j=t-i+1}^{t} P_{\pi_{t-j}} (\pi_{i} \Box) &\geq \pi_{t+1}\prod_{i=1}^{t+1} \mathcal{T}^{\pi_{i-1}} Q^{\pi_{0}}_{\Box} - \gamma^{t+1} tH(1+\lambda_k) \mathbf{1}. \tag{Since $\prod_{j=1}^tP\pi_{j-1} Q^{\pi_{0}}_{\Box} \leq \gamma^{t+1} tH(1+\lambda_k) \mathbf{1}$}
\end{align*}
Combining the above, we obtain
\begin{align*}
\hvkttb{t+1} &\geq \pi_{t} \prod_{i=1}^{t+1} \mathcal{T}^{\pi_{i-1}} Q^{\pi_{0}}_{\Box} - \gamma^{t+1} tH(1+\lambda_k) \mathbf{1} + \sum_{i=1}^{t} \gamma^i \prod_{j=t-i+1}^{t} P_{\pi_{t-j}}  \gamma (\hat{P}^{t-i}_{\pi_{t-i}} \hvkttb{t-i} - P_{\pi_{t-i}} \hvkttb{t-i}) \\
&= \svktb{t} - \gamma^{t+1} tH(1+\lambda_k) \mathbf{1} + \sum_{i=1}^{t} \gamma^i \prod_{j=t-i+1}^{t} P_{\pi_{t-j}}  \gamma (\hat{P}_{\pi_{t-i}}^{t-i} \hvkttb{t-i} - P_{\pi_{t-i}}  \hvkttb{t-i}) 
\end{align*}
\end{proof}

\begin{lemma}
\label{lem_bern3}
For any $t\in [T]$ and $k \in [K]$,
\begin{align*}
\vkb - \hvkttb{t} &\leq  \left(\gamma^{t}t + \frac{1+\lambda_k}{t(t-1)} \right) H \mathbf{1} - \sum_{i=1}^{t} \gamma^i \prod_{j=t-i+1}^{t} P_{\pi_{t-j}}   \gamma (P_{\pi_{t-i}}\hvkttb{t-i} - \hat{P}^{t-i}_{\pi_{t-i}}  \hvkttb{t-i})\\
&\indent + \sum_{i=1}^{t} \left( (\gamma P_{\pi^*_k})^{t-i}\pi^*_k -  \prod_{j=1}^{t-i} [\gamma P_{\pi_{t-j}}]\pi_{i} \right)\frac{1}{i} \sum_{j=0}^{i-1} \left[\gamma \hat{P}_j \hvkttb{j} - \gamma P \hvkttb{j}  \right] 
\end{align*}
and
\begin{align*}
\vkb - \hvkttb{t} \geq - \gamma^{t} tH(1+\lambda_k) \mathbf{1} - \sum_{i=1}^{t-1} \gamma^i \prod_{j=t-i}^{t-1} P_{\pi_{t-j}}  \gamma (P_{\pi_{t-i-1}}\hvkttb{t-i-1} - \hat{P}^{t-i-1}_{\pi_{t-i-1}}  \hvkttb{t-i-1}).
\end{align*}
\end{lemma}

\begin{proof}
From Lemma \ref{lem_bern2}, we know
\begin{align*}
\hvkttb{t} &\leq \svktb{t} + \gamma^{t} tH(1+\lambda_k) \mathbf{1} + \sum_{i=1}^{t-1} \gamma^i \prod_{j=t-i}^{t-1} P_{\pi_{t-j}}  \gamma (P_{\pi_{t-i-1}}\hvkttb{t-i-1} - \hat{P}^{t-i-1}_{\pi_{t-i-1}}  \hvkttb{t-i-1}) \\
&\leq \vkb + \gamma^{t} tH(1+\lambda_k) \mathbf{1} + \sum_{i=1}^{t-1} \gamma^i \prod_{j=t-i}^{t-1} P_{\pi_{t-j}} \gamma (P_{\pi_{t-i-1}}\hvkttb{t-i-1} - \hat{P}^{t-i-1}_{\pi_{t-i-1}}  \hvkttb{t-i-1}) 
\end{align*}
which gives us the second inequality. From Lemma \ref{lem_bern2} and Lemma \ref{lem_bern1} we have,
\begin{align*}
\svktb{t} - \hvkttb{t+1} \leq  \gamma^{t} tH(1+\lambda_k) \mathbf{1} - \sum_{i=1}^{t} \gamma^i \prod_{j=t-i+1}^{t} P_{\pi_{t-j}}   \gamma (P_{\pi_{t-i}}\hvkttb{t-i} - \hat{P}^{t-i}_{\pi_{t-i}}  \hvkttb{t-i})
\end{align*}
and 
\begin{align*}
\vkb - \svktb{t} \leq \sum_{i=1}^{t} \left(\prod_{j=1}^{t-i} [\gamma P_{\pi_{t-j}}]\pi_{i} - (\gamma P_{\pi^*_k})^{t-i}\pi^*_k \right)\frac{1}{i} \sum_{j=0}^{i-1} \left[\gamma \hat{P}_j \hvkttb{j} - \gamma P \hvkttb{j}  \right] + \frac{H(1+\lambda_k)}{t(t-1)} \mathbf{1}.
\end{align*}
Combining them gives us the upper bound.
\end{proof}
\begin{lemma}
\label{lem_bern4}
For any $t \geq 2\log(t)/\gamma$ and $k \in [K]$,
\begin{align*}
\sigma \left( \hvkttb{t} \right) \leq \left( \frac{2}{t} + 4H\sqrt{\frac{\iota}{M}}  \right)H(1+\lambda_k) \mathbf{1} + \sigma(\vkb)
\end{align*}
with probability at least $1-\delta$.
\end{lemma}
\begin{proof}
We denote $\iota=\log(2|\mathcal{S}||\mathcal{A}|/\delta)$ throughout the proof. By Lemma \ref{lem_bern3},
\begin{align} 
\vkb - \hvkttb{t} &\leq  \underbrace{\left(\gamma^{t}t + \frac{1}{t(t-1)} \right) H(1+\lambda_k) \mathbf{1}}_{\text{Term (i)}} - \underbrace{\sum_{i=1}^{t} \gamma^i \prod_{j=t-i+1}^{t} P_{\pi_{t-j}}   \gamma (\hat{P}_{\pi_{t-i}}^{t-i} \hvkttb{t-i} - P_{\pi_{t-i}}  \hvkttb{t-i})}_{\text{Term (ii)}} \nonumber \\
&\indent + \underbrace{ \sum_{i=1}^{t} \left( (\gamma P_{\pi^*_k})^{t-i}\pi^*_k -  \prod_{j=1}^{t-i} [\gamma P_{\pi_{t-j}}]\pi_{i} \right)\frac{1}{i} \sum_{j=0}^{i-1} \left[\gamma \hat{P}_j \hvkttb{j} - \gamma P \hvkttb{j}  \right] }_{\text{Term (iii)}} \label{eqn_bst0}
\end{align}
We first bound Term (ii) and Term (iii). By Azuma-Hoeffding's inequality (Lemma \ref{lem_Azuma_Hoeffding}), we have 
\begin{align*}
\left\| P_{\pi_{t-i}}\hvkttb{t-i} - \hat{P}^{t-i}_{{\pi_{t-i}}}  \hvkttb{t-i} \right\|_{\infty} \leq 2H(1+\lambda_k) \sqrt{\frac{\iota}{M}},
\end{align*}
and by Lemma \ref{lem_concentrate} with $t = i$, we have 
\begin{align*}
\left\| \frac{1}{i} \sum_{j=0}^{i-1} \left[\gamma \hat{P}^j_{\pi_i} \hvkttb{j} - \gamma P_{\pi_i} \hvkttb{j}  \right]  \right\|_{\infty} \leq 2H(1+\lambda_k) \sqrt{\frac{\iota}{iM}}
\end{align*}
each with probability at least $1-\delta$. Thus, to bound Term (ii), we have
\begin{align}
&\left\|\sum_{i=1}^{t} \gamma^i  \prod_{j=t-i+1}^{t} P_{\pi_{t-j}}  \gamma (\hat{P}^{t-i}_{\pi_{t-i}} \hvkttb{t-i} - P_{\pi_{t-i}} \hvkttb{t-i}) \right\|_{\infty} \\
&\leq \sum_{i=1}^{t} \gamma^i \left\|  \prod_{j=t-i+1}^{t} P_{\pi_{t-j}}  \gamma \right\|_{1} \left\| \hat{P}_{\pi_{t-j}}^{t-i}  \hvkttb{t-i} - P_{\pi_{t-i}} \hvkttb{t-i} \right\|_{\infty} \nonumber\\
&\leq \sum_{i=1}^{t} \gamma^i  \left\| \hat{P}_{\pi_{t-j}}^{t-i}  \hvkttb{t-i} - P_{\pi_{t-i}} \hvkttb{t-i} \right\|_{\infty} \nonumber\\
&\leq 2H^2(1+\lambda_k) \sqrt{\frac{\iota}{M}} \label{eqn_bst1}
\end{align}
and to bound Term (iii) we have
\begin{align}
&\left\| \sum_{i=1}^{t} \left( (\gamma P_{\pi^*_k})^{t-i}\pi^*_k -  \prod_{j=1}^{t-i} [\gamma P_{\pi_{t-j}}]\pi_{i} \right)\frac{1}{i} \sum_{j=0}^{i-1} \left[\gamma \hat{P}_j \hvkttb{j} - \gamma P \hvkttb{j}  \right] \right\|_{\infty} \nonumber \\
&\leq  \sum_{i=1}^{t} \gamma^{t-i} \left\| \prod_{j=1}^{t-i} [P_{\pi_{t-j}}]\pi_{i} -  (P_{\pi^*_k})^{t-i}\pi^*_k \right\|_1 \left\|\frac{1}{i} \sum_{j=0}^{i-1} \left[\gamma \hat{P}_j \hvkttb{j} - \gamma P \hvkttb{j}  \right] \right\|_{\infty} \nonumber \\
&\leq  \sum_{i=1}^{t} \gamma^{t-i} \left\|\frac{1}{i} \sum_{j=0}^{i-1} \left[\gamma \hat{P}_j \hvkttb{j} - \gamma P \hvkttb{j}  \right] \right\|_{\infty} \nonumber \\
&\leq \sum_{i=1}^{t} \gamma^{t-i} \sqrt{\frac{4H^2(1+\lambda_k)^2\iota}{iM}} \nonumber\\
&\leq \sum_{i=1}^{t} \gamma^{t-i} \sqrt{\frac{4H^2(1+\lambda_k)^2\iota}{M}} \nonumber\\
&\leq 2H^2(1+\lambda_k) \sqrt{\frac{\iota}{M}} \label{eqn_bst2}
\end{align}
each with probability at least $1-\delta$. Lastly we bound Term (i). For $t \geq 2\log(t)/\gamma$, we have
\begin{align*}
\gamma^t \leq \frac{1}{t^2}
\end{align*}
and thus
\begin{align}
\gamma^tt + \frac{1}{t(t-1)} \leq \frac{1}{t} + \frac{1}{t} = \frac{2}{t}. \label{eqn_bst3}
\end{align}
Combining \cref{eqn_bst0,eqn_bst1,eqn_bst2,eqn_bst3}, we have
\begin{align}
|\vkb - \hvkttb{t}| \leq \left( \frac{2}{t} + 4H\sqrt{\frac{\iota}{M}}  \right)(1+\lambda_k) H\mathbf{1}. \label{eqn_vardec}
\end{align}
Finally, we have
\begin{align*}
\sigma \left( \hvkttb{t} \right) &\leq \sigma \left( \vkb - \hvkttb{t} \right) + \sigma \left( \vkb \right) \tag{By Lemma \ref{lem_VX_VXY}}\\
&\leq |\vkb - \hvkttb{t}| + \sigma \left( \vkb \right) \tag{By Lemma \ref{lem_Popoviciu}}\\
&\leq \left( \frac{2}{t} + 4H\sqrt{\frac{\iota}{M}}  \right) \, (1+\lambda_k) \, H\mathbf{1} + \sigma \left( \vkb \right)
\end{align*}
with probability at least $1-\delta$, which completes the proof.
\end{proof}

\begin{lemma}[Induction Lemma]
\label{lem_inductionlem}
Assume \(X_k, A_k, B_k \geq 0\), \(k = 1, \dots,\) and
$X_{k+1} \leq A_k X_k + B_k,$ {then we have}
$X_{k+1} \leq \prod_{i=1}^{k} A_i X_1 + \sum_{i=1}^{k} \prod_{j=i+1}^{k} A_j B_i.$
\end{lemma}

\subsubsection{Auxiliary Lemmas for Lemma \ref{lem_bern6}}

\begin{lemma}
\label{lem_B4bern}
For any $t \in [T]$,
\begin{align*}
\left\| \hmvkttd{t} - V^{\pi}_{\diamond}\right\|_{\infty}  \leq \tilde{O}\left( \frac{H^2}{\sqrt{M}} + \gamma^tH\right)
\end{align*}
with probability at least $1-\delta$.
\end{lemma}

\begin{proof} 
\begin{align*}
\hmvkttd{t} &= (\pi \hmqktd{t-1})\\
&= (\pi\diamond) + \gamma \hat{P}^{t-1}_{\pi} \hmvkttd{t-1}\\
&=  (\pi\diamond) +  \gamma P_{\pi} \hmvkttd{t-1} + \gamma (\hat{P}^{t-1}_{\pi} \hmvkttd{t-1} -  P_{\pi} \hmvkttd{t-1})\\
&=  \sum_{i=0}^{t} \gamma^i (P_{\pi})^i ( (\pi  \diamond) + \gamma (\hat{P}^{t-i-1}_{\pi} \hmvkttd{t-i-1} -  P_{\pi} \hmvkttd{t-i-1})  ) \tag{By induction on $t$}\\
&=  \sum_{i=0}^{t} \gamma^i (P_{\pi})^i (\pi  \diamond) \sum_{i=0}^{t} \gamma^{i+1} (P_{\pi})^i (\hat{P}^{t-i-1}_{\pi} \hmvkttd{t-i-1} -  P_{\pi} \hmvkttd{t-i-1}) \\
&= V_{\diamond}^{\pi} - \sum_{i=t+1}^{\infty} \gamma^i (P_{\pi})^i (\pi  \diamond) + \sum_{i=0}^{t} \gamma^{i+1} (P_{\pi})^i (\hat{P}^{t-i-1}_{\pi} \hmvkttd{t-i-1} -  P_{\pi} \hmvkttd{t-i-1}).
\end{align*}
Note that with probability at least $1-\delta$
\begin{align*}
\left\| \sum_{i=0}^{t} \gamma^{i+1} (P_{\pi})^i (\hat{P}^{t-i-1}_{\pi} \hmvkttd{t-i-1} -  P_{\pi} \hmvkttd{t-i-1}) \right\|_{\infty} \leq 2H^2\sqrt{\frac{\iota}{M}}
\end{align*}
by a similar argument as in \cref{eqn_bst1}, and
\begin{align*}
\left\|  \sum_{i=t+1}^{\infty} \gamma^i  (P_{\pi})^i \pi \diamond \right\|_{\infty} \leq \gamma^tH.
\end{align*}
We conclude that 
\begin{align*}
\left\| \hmvkttd{t} - V^{\pi}_{\diamond}\right\|_{\infty}  \leq \tilde{O}\left( \frac{H^2}{\sqrt{M}} + \gamma^tH\right)
\end{align*}
with probability at least $1-\delta$.
\end{proof}

\begin{lemma}
\label{lem_varl}
For any $i \in [T]$, we have
\begin{align*}
\sigma(\hmvkid) \leq \tilde{O}\left( \frac{H^2}{\sqrt{M}}+ \gamma^tH \right) \mathbf{1} + \sigma(V^{\pi}_{\diamond})
\end{align*}
with probability at least $1-\delta$.
\end{lemma}

\begin{proof}
We have
\begin{align*}
\sigma(\hmvkid) &\leq \sigma(V^{\pi}_{\diamond}-\hmvkid) + \sigma(V^{\pi}_{\diamond}) \tag{By Lemma \ref{lem_VX_VXY}} \\
&\leq |V^{\pi}_{\diamond}-\hmvkid| + \sigma(V^{\pi}_{\diamond}) \tag{By Lemma \ref{lem_Popoviciu}}\\
&\leq \tilde{O}\left( \frac{H^2}{\sqrt{M}} + \gamma^tH \right) \mathbf{1} + \sigma(V^{\pi}_{\diamond}) \tag{By Lemma \ref{lem_B4bern}}
\end{align*}
with probability at least $1-\delta$.
\end{proof}

\subsection{Proof of Corollary \ref{thm_main2CT}}
\label{appendix_maincor2T}

\MainresultLCT*
 
\begin{proof}
By Lemma \ref{lem_mdvi} and Lemma \ref{lem_concen}, the sample complexity required to ensure $f(\mathcal{B}) \leq O(\epsilon)$ is $TM|\mathcal{C}| = \tilde{O}\left( \frac{|\mathcal{S}||\mathcal{A}| H^3}{\epsilon^2} \right)$. Therefore, the guarantee for the relaxed feasibility setting follows directly from our meta-theorem (Theorem~\ref{thm_main2}). For the strict feasibility setting, we rescale $\epsilon$ by a factor of $O(\zeta(1 - \gamma))$. Since $\epsilon \leq 1$ and $1-\gamma \leq 1$, the condition of $f(\mathcal{B}) \leq \zeta/6$ in Theorem~\ref{thm_main2} can be satisfied. The rescaling increases the sample complexity by a multiplicative factor of $\frac{1}{\zeta^2 (1 - \gamma)^2}$, thereby completing the proof. 
\end{proof}

\subsection{Instantiating the \texttt{MDP-Solver}: Model-based algorithm~\citep{li2020breaking}}
\label{app:model-based}

Instead of using \texttt{MDVI-Tabular}, the tabular \texttt{MDP-Solver} subroutine in Algorithm~\ref{alg_CMDPL} can be instantiated with any model-based method that computes an optimal policy with respect to the estimated model. In this subsection, we adapt the framework analyzed in \citep{li2020breaking} to show that, when combined with our overall framework, certain model-based \texttt{MDP-Solver} algorithms can recover the near-optimal sample complexity for solving tabular constrained MDPs.  

Since we are using model-based methods, we denote $\hat{P}$ as the probability transition kernel form by 
\begin{align*}
\forall s' \in \mathcal{S}, \quad \widehat{P}(s' \mid s, a) = \frac{1}{N} \sum_{i=1}^N \mathbbm{1} \{s_{s,a}^i = s'\}
\end{align*}
where $(s_{s,a}^i)_{i=1}^N$ are the next-state samples from $\mathcal{B}=\texttt{DataCollection}(\mathsf{Gen}, \mathcal{S}\times \mathcal{A}, N)$.
Denote the perturbed reward by
\begin{align*}
r_{\mathrm{p}}(s,a)=r(s,a)+\zeta(s,a), \quad \zeta(s,a) \sim \text{Unif}(0,\xi)
\end{align*}
where $\text{Unif}(0,\xi)$ denotes the uniform distribution.
For any policy $\pi$, denote $\hat{V}_{\mathrm{p}}^{\pi}$ the corresponding value function of the perturbed empirical MDP $\widehat{\mathcal{M}}_{\mathrm{p}} = (\mathcal{S},\mathcal{A}, \hat{P}, r_{\mathrm{p}}, \gamma)$. Denote $\hat{\pi}_p^*$ the optimal policy w.r.t. $\widehat{\mathcal{M}}_{\mathrm{p}}$ (i.e. $\hat{\pi}_p^* = \argmax_{\pi} \hat{V}^{\pi}_{\mathrm{p}}$). Their main result is stated as follows.

\begin{theorem}[Theorem 1 in \citep{li2020breaking}]
\textit{There exist some universal constants $c_0, c_1 > 0$ such that: for any $\delta > 0$ and any $0 < \varepsilon \leq \frac{1}{1 - \gamma}$, the policy $\hat{\pi}^*_{\textrm{p}}$ defined in (9) obeys}
\[
\forall(s,a) \in \mathcal{S} \times \mathcal{A}, \quad V^{\hat{\pi}^*_{\textrm{p}}}(s) \geq V^*(s) - \varepsilon \quad \text{and} \quad Q^{\hat{\pi}^*_{\textrm{p}}}(s,a) \geq Q^*(s,a) - \gamma \varepsilon, \tag{11}
\]
\textit{with probability at least $1 - \delta$, provided that the perturbation size is $\xi = \frac{c_1 (1 - \gamma)\varepsilon}{|\mathcal{S}|^5 |\mathcal{A}|^5}$ and that the sample size per state-action pair exceeds}
\[
N \geq \frac{c_0 \log \left( \frac{|\mathcal{S}||\mathcal{A}|}{(1 - \gamma)\varepsilon \delta} \right)}{(1 - \gamma)^3 \varepsilon^2}. \tag{12}
\]
\textit{In addition, both the empirical QVI and PI algorithms w.r.t. $\widehat{\mathcal{M}}_{\textrm{p}}$ (cf.~\citep{azar2013minimax}, Algorithms 1-2) are able to recover $\hat{\pi}^*_{\textrm{p}}$ perfectly within $\mathcal{O}\left( \frac{1}{1 - \gamma} \log \left( \frac{|\mathcal{S}||\mathcal{A}|}{(1 - \gamma)\varepsilon \delta} \right) \right)$ iterations.}
\end{theorem}

Therefore, let $\mathcal{B} = \texttt{DataCollection}(\mathsf{Gen}, \mathcal{S} \times \mathcal{A}, N)$. Then, by instantiating \texttt{MDP-Solver}$(r + \lambda_k c, \mathcal{B}, \phi)$ with any model-based algorithm that returns an optimal policy with respect to the perturbed empirical MDP constructed from $\mathcal{B}$, Assumption~\ref{asmp_b1} can be satisfied with $f_{\mathrm{mdp}}(\mathcal{B}) = O(\epsilon)$. As a consequence, we recover the near-optimal sample complexity bounds for solving tabular constrained MDPs via our meta-theorem (Theorem~\ref{thm_main2}). Furthermore, the limited rage of $\epsilon$ (i.e. $(0,1/H^2]$) in Corollary \ref{thm_main2CT} will be improved to a full range (i.e. $(0,H]$).

\section{Supporting Lemmas}
\label{Appendix_suplem}

\subsection{Concentration Inequalities}

The following lemma is used throughout the paper. In the linear setting, we take $\mathcal{C}$ to be the core set and set $R = (1 + \lambda_k)H$. In the tabular setting, we let $\mathcal{C} = \mathcal{S} \times \mathcal{A}$ and set $R = H$.

\begin{lemma}
\label{lem_concentrate} 
Let $\hat{V}^i$ be an empirical value function with entries bounded in $[0, R]$, and let $\mathcal{C} \subseteq \mathcal{S} \times \mathcal{A}$. Then, for any $t \in {1, \ldots, T}$, the following holds:
\begin{align*}
\mathbb{P}\left( \exists (s,a) \in \mathcal{C} \ \text{ s.t. } \ \frac{1}{t}\sum_{i=0}^{t-1}[(\hat{P}_i \hat{V}^{i})(s,a) - ({P}\hat{V}^{i})(s,a)] \geq 2R \sqrt{{\log(2|\mathcal{C}|/\delta)}/{tM}}\right) \leq \delta
\end{align*}
\end{lemma}
\begin{proof}
Consider a fixed $t \in \{1,\cdots, T\}$ and $(s,a) \in \mathcal{C}$. Denote $y_{t,m,s,a}$ as the $m'$th next-state sample we collect for state-action pair $(s,a)$ at iteration $t$. Since
\begin{align*}
\frac{1}{t}\sum_{i=0}^{t-1} \left[ (\hat{P}_i \hat{V}^{i})(s,a)  -  ({P} \hat{V}^{i})(s,a) \right] &= \frac{1}{t}\sum_{i=0}^{t-1} \frac{1}{M}\sum_{m=1}^{M} \left[ \hat{V}^{i}(y_{t,m,s,a})  -  ({P} \hat{V}^{i})(s,a) \right]\\    
& = \frac{1}{t}\sum_{i=0}^{t-1} \frac{1}{M}\sum_{m=1}^{M}  \left[ \hat{V}^{i}(y_{t,m,s,a})  -  ({P} \hat{V}^{i})(s,a) \right]
\end{align*}
is a sum of bounded martingale differences with respect to the filtration $(\mathcal{F}_{i,m})_{i=0,m=1}^{t-1,M}$. Thus, using the Azuma-Hoeffding inequality (Lemma \ref{lem_Azuma_Hoeffding}),
\begin{align*}\
\mathbb{P}\left( \frac{1}{t}\sum_{i=0}^{t-1} \frac{1}{M}\sum_{m=1}^{M}  \left[ \hat{V}^{i}(y_{t,m,s,a})  -  ({P} \hat{V}^{i})(s,a) \right] \geq 2R \sqrt{\frac{\log(2|\mathcal{C}|/\delta)}{tM}} \right) \leq \frac{\delta}{|\mathcal{C}|}.
\end{align*}
Taking the union bound over $(s,a) \in \mathcal{C} $
\begin{align*}
&\mathbb{P}\left( \max_{(s,a)\in \mathcal{C}} \frac{1}{t}\sum_{i=0}^{t-1} \frac{1}{M}\sum_{m=1}^{M} \left[ \hat{V}^{i}(y_{t,m,s,a})  -  ({P} \hat{V}^{i})(s,a) \right]  \leq 2R \sqrt{\frac{\log(2|\mathcal{C}|/\delta)}{tM}}\right)\\
&\geq 1- \sum_{(s,a)\in \mathcal{C}} \mathbb{P}\left( \frac{1}{t}\sum_{i=0}^{t-1} \frac{1}{M}\sum_{m=1}^{M} \left[ \hat{V}^{i}(y_{t,m,s,a})  -  ({P} \hat{V}^{i})(s,a) \right] \geq 2R \sqrt{\frac{\log(2|\mathcal{C}|/\delta)}{tM}}\right) \\
&\geq 1-\delta,
\end{align*}
which implies the desired result.
\end{proof}

\begin{lemma}[Azuma-Hoeffding Inequality]
\label{lem_Azuma_Hoeffding}
Consider a real-valued stochastic process $(X_n)_{n=1}^N$ adapted to a
filtration $(\mathcal{F}_n)_{n=1}^N$. Assume that $X_n \in [l_n, u_n]$ and $\mathbb{E}_n[X_n]=0$ almost surely, for all $n$. Then,
\begin{align*}
\mathbb{P}\left( \sum_{n=1}^N X_n \geq \sqrt{\sum_{n=1}^N \frac{(u_n - l_n)^2}{2} \log \frac{1}{\delta}} \right) \leq \delta
\end{align*}
for any $\delta \in (0,1)$.
\end{lemma}

\begin{lemma}[Bernstein's Inequality]
\label{lem_Bernstein}
Consider a real-valued stochastic process $(X_n)_{n=1}^N$ adapted to a
filtration $(\mathcal{F}_n)_{n=1}^N$.Suppose that $X_{n} \leq U$ and $\mathbb{E}_{n}\left[X_{n}\right]=0$ almost surely, for all $n$. Then, letting $Z^{\prime}:=\sum_{n=1}^{N} \mathbb{E}_{n}\left[X_{n}^{2}\right]$,
$$
\mathbb{P}\left(\sum_{n=1}^{N} X_{n} \geq \frac{2 U}{3} \log \frac{1}{\delta}+\sqrt{2 Z \log \frac{1}{\delta}} \text { and } Z^{\prime} \leq Z\right) \leq \delta
$$
for any $Z \in[0, \infty)$ and $\delta \in(0,1)$.
\end{lemma}

\begin{lemma}[Conditional Bernstein's Inequality]
\label{lem_cond_Bernstein}
Consider the same notations and assumptions in Lemma \ref{lem_Bernstein}. Furthermore, let $\mathcal{E}$ be an event that implies $Z^{\prime} \leq Z$ for some $Z \in[0, \infty)$ with $\mathbb{P}(\mathcal{E}) \geq 1-\delta^{\prime}$ for some $\delta^{\prime} \in(0,1)$. Then,
$$
\mathbb{P}\left(\left.\sum_{n=1}^{N} X_{n} \geq \frac{2 U}{3} \log \frac{1}{\delta\left(1-\delta^{\prime}\right)}+\sqrt{2 Z \log \frac{1}{\delta\left(1-\delta^{\prime}\right)}} \right\rvert\, \mathcal{E}\right) \leq \delta
$$
for any $\delta \in(0,1)$.
\end{lemma}

\subsection{Lemmas for Variances}

\begin{lemma}[Popoviciu's Inequality for Variances]
\label{lem_Popoviciu}
The variance of any random variable bounded by $x$ is bounded by $x^{2}$.
\end{lemma}

\begin{lemma}[\cite{azar2013minimax}]
\label{lem_VX_VXY}
Suppose two real-valued random variables $X, Y$ whose variances, $\mathbb{V} X$ and $\mathbb{V} Y$, exist and are finite. Then, $\sqrt{\mathbb{V} X} \leq \sqrt{\mathbb{V}[X-Y]}+\sqrt{\mathbb{V} Y}$.
\end{lemma}

\begin{lemma}[Total variance lemma \cite{azar2013minimax}]
\label{lem_TVL}
For any policy $\pi$, $ \| (I - P_\pi)^{-1}  \sigma(V^{\pi}) \|_{\infty} \leq \sqrt{2H^3}$.
\end{lemma}

\subsection{Lemmas for Constrained MDPs}

\begin{lemma}[Constraint violation bound, Lemma B.2 in \cite{jain2022towards}]
\label{lemma_10}
For any $C \geq \lambda^*$ and any $\pi$ s.t. $V_r^{*}(\rho) - V_r^{\pi}(\rho) + C[b - V_c^{\pi}(\rho)]_+ \leq \beta$, we have $[b - V_c^{\pi}(\rho)]_+ \leq \frac{\beta}{C-\lambda^*}$.
\end{lemma}

\begin{lemma}[Bounding the dual variable, Lemma 4.1 in \cite{jain2022towards}]
\label{lem_bounddualv}
The objective \cref{eq:true-CMDP} satisfies strong duality, and the optimal dual variables are bounded as 
\[
\lambda^* \leq \frac{1}{(1 - \gamma)\zeta}, \quad \text{where } \zeta := \max_{\pi} V_c^{\pi}(\rho) - b > 0.
\]
\end{lemma}

\begin{lemma}[Bounding the sensitivity error, Lemma 13 in \cite{vaswani2022near}]
\label{lem_boundsense}
If we have
\[
\hat{\pi}^* \in \arg\max_{\pi} {V}^\pi_{r}(\rho) \text{ s.t. } {V}^\pi_{c}(\rho) \geq b + \Delta
\]
\[
\tilde{\pi}^* \in \arg\max_{\pi} {V}^\pi_{r}(\rho) \text{ s.t. } {V}^\pi_{c}(\rho) \geq b - \Delta,
\]
then the sensitivity error term can be bounded by:
\[
\left| {V}^{\hat{\pi}^*}_{r}(\rho) - {V}^{\tilde{\pi}^*}_{r}(\rho) \right| \leq 2 \Delta \lambda^*
\]
where $\lambda^*$ is the optimal Lagrange multiplier (i.e., the solution to \cref{eq:saddle}).
\end{lemma}

\end{document}